\RequirePackage[l2tabu,orthodox]{nag}
\documentclass
[letterpaper,11pt,]
{article}

\usepackage{etex}
\usepackage{verbatim}
\usepackage{xspace,enumerate}
\usepackage[dvipsnames]{xcolor}
\usepackage[T1]{fontenc}
\usepackage[full]{textcomp}
\usepackage[american]{babel}
\usepackage{mathtools}
\usepackage{amsthm}
\usepackage[
letterpaper,
top=1in,
bottom=1in,
left=1in,
right=1in]{geometry}
\usepackage{newpxtext} %
\usepackage{textcomp} %
\usepackage[varg,bigdelims]{newpxmath}
\usepackage[scr=rsfso]{mathalfa}%
\usepackage{bm} %
\let\mathbb\varmathbb
\usepackage{microtype}
\usepackage[pagebackref,colorlinks=true,urlcolor=blue,linkcolor=blue,citecolor=OliveGreen]{hyperref}
\usepackage[capitalise,nameinlink]{cleveref}
\crefname{lemma}{Lemma}{Lemmas}
\crefname{fact}{Fact}{Facts}
\crefname{theorem}{Theorem}{Theorems}
\crefname{corollary}{Corollary}{Corollaries}
\crefname{claim}{Claim}{Claims}
\crefname{example}{Example}{Examples}
\crefname{algorithm}{Algorithm}{Algorithms}
\crefname{problem}{Problem}{Problems}
\crefname{definition}{Definition}{Definitions}
\crefname{exercise}{Exercise}{Exercises}
\usepackage{amsthm}

\newtheorem{theorem}{Theorem}[section]
\newtheorem*{theorem*}{Theorem}
\newtheorem{lemma}[theorem]{Lemma}
\newtheorem*{lemma*}{Lemma}
\newtheorem{fact}[theorem]{Fact}
\newtheorem*{fact*}{Fact}

\newtheorem*{proposition*}{Proposition}

\newtheorem*{corollary*}{Corollary}

\newtheorem*{hypothesis*}{Hypothesis}

\newtheorem*{conjecture*}{Conjecture}
\theoremstyle{definition}
\newtheorem{definition}[theorem]{Definition}
\newtheorem*{definition*}{Definition}

\newtheorem*{construction*}{Construction}

\newtheorem*{example*}{Example}

\newtheorem*{question*}{Question}
\newtheorem{algorithm}[theorem]{Algorithm}
\newtheorem*{algorithm*}{Algorithm}

\newtheorem*{assumption*}{Assumption}

\newtheorem*{problem*}{Problem}

\newtheorem*{openquestion*}{Open Question}
\theoremstyle{remark}

\newtheorem*{claim*}{Claim}

\newtheorem*{remark*}{Remark}

\newtheorem*{observation*}{Observation}
\usepackage{paralist}
\let\originalleft\left
\let\originalright\right
\renewcommand{\left}{\mathopen{}\mathclose\bgroup\originalleft}
\renewcommand{\right}{\aftergroup\egroup\originalright}
\usepackage{turnstile}
\usepackage{mdframed}

\usepackage{xparse}
\usepackage{amsthm} %
\makeatletter
\let\latexparagraph\paragraph
\RenewDocumentCommand{\paragraph}{som}{%
  \IfBooleanTF{#1}
    {\latexparagraph*{#3}}
    {\IfNoValueTF{#2}
       {\latexparagraph{\maybe@addperiod{#3}}}
       {\latexparagraph[#2]{\maybe@addperiod{#3}}}%
  }%
}
\newcommand{\maybe@addperiod}[1]{%
  #1\@addpunct{.}%
}
\makeatother

\newcommand{\Authornote}[2]{}
\newcommand{\Authornotecolored}[3]{}
\newcommand{\Authorcomment}[2]{}
\newcommand{\Authorfnote}[2]{}

\newcommand{\Dnote}{\Authornote{D}}

\usepackage{boxedminipage}

\newcommand{\Paren}[1]{\left(#1\right)}

\newcommand{\abs}[1]{\lvert#1\rvert}

\newcommand{\set}[1]{\{#1\}}
\newcommand{\Set}[1]{\left\{#1\right\}}

\newcommand{\norm}[1]{\lVert#1\rVert}
\newcommand{\Norm}[1]{\left\lVert#1\right\rVert}

\newcommand{\iprod}[1]{\langle#1\rangle}

\newcommand{\Esymb}{\mathbb{E}}
\newcommand{\Psymb}{\mathbb{P}}

\DeclareMathOperator*{\E}{\Esymb}

\DeclareMathOperator*{\ProbOp}{\Psymb}
\renewcommand{\Pr}{\ProbOp}

\newcommand{\seteq}{\mathrel{\mathop:}=}

\newcommand\bdot\bullet

\newcommand{\N}{\mathbb N}
\newcommand{\R}{\mathbb R}

\newcommand{\cA}{\mathcal A}

\renewcommand{\leq}{\leqslant}
\renewcommand{\le}{\leqslant}
\renewcommand{\geq}{\geqslant}
\renewcommand{\ge}{\geqslant}
\let\epsilon=\varepsilon
\numberwithin{equation}{section}
\newcommand\MYcurrentlabel{xxx}
\newcommand{\MYstore}[2]{%
  \global\expandafter \def \csname MYMEMORY #1 \endcsname{#2}%
}
\newcommand{\MYload}[1]{%
  \csname MYMEMORY #1 \endcsname%
}
\newcommand{\MYnewlabel}[1]{%
  \renewcommand\MYcurrentlabel{#1}%
  \MYoldlabel{#1}%
}
\newcommand{\MYdummylabel}[1]{}
\newcommand{\torestate}[1]{%
  \let\MYoldlabel\label%
  \let\label\MYnewlabel%
  #1%
  \MYstore{\MYcurrentlabel}{#1}%
  \let\label\MYoldlabel%
}
\newcommand{\restatetheorem}[1]{%
  \let\MYoldlabel\label
  \let\label\MYdummylabel
  \begin{theorem*}[Restatement of \cref{#1}]
    \MYload{#1}
  \end{theorem*}
  \let\label\MYoldlabel
}
\newcommand{\restatelemma}[1]{%
  \let\MYoldlabel\label
  \let\label\MYdummylabel
  \begin{lemma*}[Restatement of \cref{#1}]
    \MYload{#1}
  \end{lemma*}
  \let\label\MYoldlabel
}
\newcommand{\restateprop}[1]{%
  \let\MYoldlabel\label
  \let\label\MYdummylabel
  \begin{proposition*}[Restatement of \cref{#1}]
    \MYload{#1}
  \end{proposition*}
  \let\label\MYoldlabel
}
\newcommand{\restatefact}[1]{%
  \let\MYoldlabel\label
  \let\label\MYdummylabel
  \begin{fact*}[Restatement of \cref{#1}]
    \MYload{#1}
  \end{fact*}
  \let\label\MYoldlabel
}
\newcommand{\restate}[1]{%
  \let\MYoldlabel\label
  \let\label\MYdummylabel
  \MYload{#1}
  \let\label\MYoldlabel
}

\newcommand{\e}{\epsilon}

\allowdisplaybreaks
\providecommand{\Dnote}[1]{}

\providecommand{\pmin}{p_{\mathrm{min}}}

\providecommand{\tightlist}{%
\setlength{\itemsep}{0pt}\setlength{\parskip}{0pt}}

\newcommand*{\dyad}[1]{#1#1{}^{\mkern-1.5mu\mathsf{T}}}

\title{
  Beyond Parallel Pancakes: Quasi-Polynomial Time Guarantees for Non-Spherical Gaussian Mixtures\thanks{This project has received funding from the European Research Council (ERC) under the European Union’s Horizon 2020 research and innovation programme (grant agreement No 815464).}
}

\author{
  Rares-Darius Buhai\thanks{ETH Z\"urich.
}
  \and
  David Steurer\footnotemark[2]
}

\begin{document}

\pagestyle{empty}

\maketitle
\thispagestyle{empty} %

\begin{abstract}

  We consider mixtures of \(k\ge 2\) Gaussian components with unknown means and unknown covariance (identical for all components) that are well-separated, i.e., distinct components have statistical overlap at most \(k^{-C}\) for a large enough constant \(C\ge 1\).

Previous statistical-query \cite{MR3734219-DiakonikolasKane17} and lattice-based \cite{bruna2021continuous, gupte2022continuous} lower bounds give formal evidence that, even for the special case of colinear means, distinguishing such mixtures from (pure) Gaussians may be exponentially hard (in \(k\)).

We show that, surprisingly, this kind of hardness can only appear if mixing weights are allowed to be exponentially small.
For polynomially lower bounded mixing weights, we show how to achieve non-trivial statistical guarantees in quasi-polynomial time.

Concretely, we develop an algorithm based on the sum-of-squares method with running time quasi-polynomial in the minimum mixing weight.
The algorithm can reliably distinguish between a mixture of \(k\ge 2\) well-separated Gaussian components and a (pure) Gaussian distribution.
As a certificate, the algorithm computes a bipartition of the input sample that separates some pairs of mixture components, i.e., both sides of the bipartition contain most of the sample points of at least one component.

For the special case of colinear means, our algorithm outputs a \(k\)-clustering of the input sample that is approximately consistent with all components of the underlying mixture.
We obtain similar clustering guarantees also for the case that the overlap between any two mixture components is lower bounded quasi-polynomially in \(k\) (in addition to being upper bounded polynomially in \(k\)).

A significant challenge for our results is that they appear to be inherently sensitive to small fractions of adversarial outliers unlike most previous algorithmic results for Gaussian mixtures.
The reason is that such outliers can simulate exponentially small mixing weights even for mixtures with polynomially lower bounded mixing weights.

A key technical ingredient of our algorithms is a characterization of separating directions for well-separated Gaussian components in terms of ratios of polynomials that correspond to moments of two carefully chosen orders logarithmic in the minimum mixing weight.

\end{abstract}

\clearpage

\microtypesetup{protrusion=false}
\tableofcontents{}
\microtypesetup{protrusion=true}

\clearpage

\pagestyle{plain}
\setcounter{page}{1}

\section{Introduction}
\label{sec:introduction}

Gaussian mixture models (GMMs) are among the most extensively studied statistical models in a wide range of scientific disciplines \cite{10.2307/90667,DBLP:conf/focs/Dasgupta99,DBLP:journals/jacm/AshtianiBHLMP20}.
Over the course of the last two decades, a major body of research explored what kinds of algorithmic guarantees are feasible for GMMs \cite{DBLP:conf/focs/Dasgupta99,DBLP:conf/focs/VempalaW02,DBLP:conf/stoc/KalaiMV10,DBLP:conf/focs/MoitraV10,MR3385380-Hsu13}.

Recent years have seen significant algorithmic advances along two dimensions.

The first kinds of advances concern mixtures of a large number of spherical Gaussians, i.e., Gaussians with identity \(I_d\) as covariance.\footnote{
  Many known algorithms for mixtures of Gaussians with covariance \(I_d\) also extend to somewhat more general settings, e.g., the case that the covariances are different multiples of \(I_d\) or diagonal matrices (axis-aligned case) or that case that the covariance is upper bounded in the Loewner order by \(I_d\).
  For our discussion, we focus on the simplest case (all covariances identity) because, to the best of our knowledge, these kinds of generalizations are orthogonal to the kind of generalization we aim for in this work.
}
Several works showed how to cluster such mixtures in time quasi-polynomial in the number \(k\) of components under a minimum mean-separation requirement of \(O(\sqrt{\log k})\), which up to a constant factor matches the minimum separation that guarantees clusterability of the mixture \cite{MR3826314-Hopkins18,MR38263160-Diakonikolas18,MR3826315-KothariSS18}.
In a recent breakthrough, the running time has been improved to polynomial assuming a slightly larger minimum separation of \(O(\log^{1/2+c} k)\) for any \(c>0\) \cite{MR4490076-LiuLi22}.
Even without any separation requirement, it is possible to compute quasi-polynomially sized covers of the set of means \cite{MR4232034-Diakonikolas20}.

The second kinds of advances concern mixtures of a small number of Gaussian components with unknown covariances.
These advances extended previous algorithmic guarantees to the robust setting, i.e., in the presence of a small constant fraction of adversarially chosen outliers.
Concretely, it is now possible to estimate the parameters of an arbitrary mixture of Gaussian components in the presence of such outliers \cite{MR4232031-DiakonikolasHopkinsKothari2020, DBLP:conf/stoc/LiuM21, DBLP:journals/corr/abs-2012-02119}.
The running time is polynomial in the ambient dimension but (at least) exponential in the number of components.

One of the most outstanding challenges remaining in this area is to clarify what kinds of algorithmic guarantees are possible when the number of components is large and their covariances are unknown.  So far, mixtures of a large number of Gaussian components with unknown covariances have defied comparable algorithmic progress.\footnote{
  A notable exception is a particular smoothed model for such mixtures when the ambient dimension is large enough \cite{DBLP:conf/stoc/GeHK15}.}
Indeed, there is formal evidence, in the form of statistical-query \cite{MR3734219-DiakonikolasKane17} and lattice-based \cite{bruna2021continuous, gupte2022continuous} lower bounds, to suggest that this setting is computationally inherently harder than the spherical setting.
Specifically, these results suggest that even for \(k\) components with tiny statistical overlaps, say at most \(2^{-k}\), approximately clustering the components may require time exponential in \(k\) despite the sample complexity being polynomial in \(d\) and \(k\).
Underlying this evidence is the well-known parallel pancakes construction:
Orthogonal to a randomly chosen direction \(u\), all components of this mixture distribution agree with a (pure) standard Gaussian distribution, and along direction \(u\), the components are well-separated but their mixture matches the first \(k\) moments of a (univariate) standard Gaussian distribution.\footnote{
  As a consequence of this construction, all means are colinear with \(u\).
  We also emphasize that these means are well-separated relative to the variance of each component in direction \(u\).
  However, since this variance is very small (about \(k^{-O(1)}\)), the standard Euclidean distance between the components is tiny.
  We remark that parallel pancakes constructions have been discussed in the literature already before \cite{MR3734219-DiakonikolasKane17}.
  The influential work \cite{DBLP:conf/focs/BrubakerV08} provided an efficient algorithm for the case \(k=2\).
}

In this work, we show that, surprisingly, this kind of hardness can appear only in the case that mixing weights are allowed to be exponentially small.
Indeed, we develop algorithms with substantial statistical guarantees that run in quasi-polynomial time whenever the mixing weights are bounded from below by a polynomial.
Before our work, the best known running times to achieve these kinds of guarantees were (at least) exponential in \(k\).
We hope that our work opens up a new direction of research on efficient algorithms for mixtures of well-separated Gaussians with polynomially lower bounded mixing weights.

\Dnote{MAYBE: skip this conjecture}
\Dnote{traditionally these kinds of natural open questions are placed in a conclusions setting.
  but it appears to be more natural to state it here.}
Within this new direction of research, we identify the following appealing open question:

\addtolength\leftmargini{0.3in}
\begin{quote}
  \itshape
  Consider a mixture of \(k\ge 2\) Gaussian components with unknown means \(\mu_1,\ldots,\mu_k\in\R^d\) and unknown covariance \(\Sigma\in\R^{d\times d}\) (identical for all components)
  \footnote{The hard instances in all lower bounds cited above are mixtures with identical covariances. A natural motivation to consider identical covariances is affine invariance. Algorithms for the spherical case assume that the input data is presented in a favorable affine transformation. However, a natural property desired for algorithms operating on geometric data is to be invariant under affine transformations \cite{DBLP:conf/focs/BrubakerV08}.}
  and with minimum mixing weight \(\pmin>0\).
  Suppose the components are well-separated in the sense that any two distinct components have statistical overlap at most \(\pmin^C\) for a large enough constant \(C\ge 1\).
  
  Given a sample of size \(n\ge d^{O(\log (1/\pmin))}\), can we compute in time polynomial in \(n\) a \(k\)-clustering of the sample that is consistent with the mixture components on all but at most a \(\pmin^{10}\) fraction of the sample?
\end{quote}

We conjecture that such an algorithm does exist.
Indeed, we confirm the conjecture for the special case that the means are colinear (\cref{thm:colinear-means-results}) and under a diameter bound (\cref{thm:small-radius-results}).
In the general case, our algorithm provides a somewhat weaker guarantee and computes only a bipartition of the sample that separates at least one pair of mixture components (\cref{thm:separating-polynomial-results}).\footnote{The algorithm for the general case also computes such a bipartition under the weaker assumption that \emph{there exists} a pair of mixture components that has small overlap (as opposed to all pairs having small overlap). Under this assumption full clustering is impossible and a partial clustering seems the appropriate guarantee to aim for.}

We identify an interesting challenge in the context of establishing the above conjecture that our techniques can partially overcome:
Any hypothetical algorithm establishing the above conjecture or our (non-hypothetical) algorithms inherently cannot be robust to even a tiny fraction of outliers (assuming the hardness of the parallel-pancakes constructions in \cite{MR3734219-DiakonikolasKane17, bruna2021continuous, gupte2022continuous}).
The reason is that a tiny \(1/k^{100}\) fraction of outliers are enough to simulate these hard instances by adding components with appropriately decaying mixing weights and spaced means.
At the same time, many recent algorithmic approaches in the context of GMMs are inherently tied to robustness.
For example, certain kinds of identifiability proofs used in the analysis of sum-of-squares based algorithms automatically imply robust algorithms.
Also many kinds of iteration schemes inherently require robustness for their subroutines in order to guarantee that the next iteration can successfully deal with the errors introduced by previous iterations.

\subsection{Results}
\label{sec:results}

\paragraph{Separating bipartition}

Suppose we are given a quasi-polynomial size sample of a mixture of \(k\) Gaussian components with unknown means \(\mu_1,\ldots,\mu_k\in\R^d\) and unknown covariance \(\Sigma\in\R^{d\times d}\) and with minimum mixing weight at least $1/k^{100}$ such that there exists a pair of mixture components $a \neq b$ with \(\norm{\Sigma^{-1/2} (\mu_a-\mu_b)}\gg \sqrt{\log k}\).
Then, as \cref{thm:separating-polynomial-results} shows, it is possible to compute in quasi-polynomial time a bipartition of the samples such that, for each side of the bipartition, there exists a component with $0.99$ of its samples assigned to it. 

\begin{theorem}
  \label{thm:separating-polynomial-results}
  Given a sample of size \(n\ge (kd)^{O(\log k)}\) from a mixture of \(k\) Gaussian components \(N(\mu_1,\Sigma),\ldots,N(\mu_k,\Sigma)\) with minimum mixing weight at least $1/k^{100}$ such that \(\max_{a\neq b} \norm{\Sigma^{-1/2} (\mu_a-\mu_b)}\gg \sqrt{\log k}\),
  there exists an algorithm that runs in time $n \cdot d^{O(\log k)}$ and returns with probability $0.99$ a partition of $[n]$ into two sets $C_1$ and $C_2$ such that, if the true clustering of the samples is $S_1, ..., S_k$, then
  \[\max_i \frac{|C_1 \cap S_i|}{|S_i|} \geq 0.99 \quad \text{and} \quad \max_i \frac{|C_2 \cap S_i|}{|S_i|} \geq 0.99.\]  
\end{theorem}

For general mixing weights, the same result holds with $k$ replaced by $1/\pmin$ in all guarantees. See \cref{thm:separating-bipartition-main} for the full result.

\paragraph{Colinear means}

Suppose, in addition, that the mixture of Guassians is well-separated, i.e., the minimum mean separation satisfies \(\min_{a\neq b} \norm{\Sigma^{-1/2} (\mu_a-\mu_b)}\gg \sqrt{\log k}\), and that the unknown means $\mu_1, ..., \mu_k$ are colinear.
Given a quasi-polynomial number of samples from the mixture, \cref{thm:colinear-means-results} shows that it is possible to compute in quasi-polynomial time a partition of the samples into $k$ clusters such that the fraction of samples assigned to incorrect clusters is polynomially small in $k$.

For simplicity, in the theorem statement below we assume that all eigenvalues of $\Sigma$ and all eigenvalues of the covariance matrix of the mixture are polynomially lower and upper bounded in $k$ and $d$.

\begin{theorem}
  \label{thm:colinear-means-results}
  Given a sample of size \(n\ge (kd)^{O(\log k)}\) from a mixture of \(k\) Gaussian components \(N(\mu_1,\Sigma),\ldots,N(\mu_k,\Sigma)\) with minimum mixing weight at least $1/k^{100}$ such that \(\min_{a\neq b} \norm{\Sigma^{-1/2} (\mu_a-\mu_b)}\gg \sqrt{\log k}\) and $\mu_1, ..., \mu_k$ colinear,
  there exists an algorithm that runs in time $n^{O(\log k)}$ and returns with high probability a partition of \([n]\) into $k$ sets $C_1, ..., C_k$ such that, if the true clustering of the samples is $S_1, ..., S_k$, then there exists a permutation $\pi$ of $[k]$ such that 
  \[1 - \frac{1}{n} \sum_{i=1}^k |C_i \cap S_{\pi(i)}| \leq k^{-O(1)}.\]
\end{theorem}

For general mixing weights, the same result holds with $k$ replaced by $1/\pmin$ in all guarantees. See \cref{thm:colinear-main} for the full result.

Given such a clustering, we can also recover the means and the covariance of the components using robust Gaussian estimation algorithms \cite{MR3945261-DiakonikolasKaneJournal19} or robust moment estimation algorithms \cite{MR3826315-KothariSS18}.
For example, via \cite{MR3826315-KothariSS18}, we obtain a multiplicative approximation to the covariance $(1-k^{-O(1)}) \Sigma \preceq \hat{\Sigma} \preceq (1+k^{-O(1)}) \Sigma$ and a "covariance-aware" approximation to the means $\|\Sigma^{-1/2}(\hat{\mu}_i - \mu_i)\| \leq k^{-O(1)}$.

\paragraph{Small radius}

If instead of colinear means we have bounded means $\norm{\Sigma^{-1/2} \mu_i} \leq R$ with ${R = \operatorname{polylog}(k)}$, \cref{thm:small-radius-results} shows that it is again possible to cluster the samples with a quasi-polynomial number of samples and quasi-polynomial time.

\begin{theorem}
  \label{thm:small-radius-results}
  Given a sample of size \(n\ge (kd)^{O(R^2 + \log k)}\) from a mixture of \(k\) Gaussian components \(N(\mu_1,\Sigma),\ldots,N(\mu_k,\Sigma)\) with minimum mixing weight at least $1/k^{100}$ such that \(\min_{a\neq b} \norm{\Sigma^{-1/2} (\mu_a-\mu_b)}\gg \sqrt{\log k}\) and $\norm{\Sigma^{-1/2} \mu_i} \leq R$,
  there exists an algorithm that runs in time $n^{O(R^2 + \log k)}$ and returns with high probability a partition of \([n]\) into $k$ sets $C_1, ..., C_k$ such that, if the true clustering of the samples is $S_1, ..., S_k$, then there exists a permutation $\pi$ of $[k]$ such that 
  \[1 - \frac{1}{n} \sum_{i=1}^k |C_i \cap S_{\pi(i)}| \leq k^{-O(1)}.\]
\end{theorem}

For general mixing weights, the same result holds with $k$ replaced by $1/\pmin$ in all guarantees. See \cref{thm:small-radius-main} for the full result.

As in the case of colinear means, given such a a clustering, we can recover the means and the covariance of the components. 

Unlike our results for separating bipartitions and colinear means, this result follows from a direct reduction to a previous algorithm for spherical components \cite{MR3826314-Hopkins18}.
Concretely, we observe that this algorithm requires only a rough multiplicative approximation (in the SOS sense) of a polynomial of the form \(q(v)=\norm{\Sigma^{1/2} v}^t\) for some \(t\) polylogarithmic in \(k\).
As we show, the empirical moment tensor of the mixture readily provides such an approximation.

\subsection{Related works}

\paragraph{Comparision to recent algorithms based on lattice basis reduction}

Two independent works (also independent and concurrent with a preprint of our work) obtain polynomial-time algorithms for learning parallel-pancakes mixtures for the case that the component variance is zero along the hidden direction (infinitesimally flat pancakes)\footnote{These works also crucially assume a mild bound on the bit complexity of the unknown means} \cite{zadik2022lattice,diakonikolas2022non}.
These algorithms are based on the LLL lattice basis reduction algorithm \cite{MR682664-Lenstra82} and have a completely different flavor than our algorithms and previous algorithms for Gaussian mixture models.
However, these lattice basis reduction techniques are expected to be brittle and limited to the case that the variance in the hidden direction is tiny.

\paragraph{Comparision to previous algorithms for mixtures with few components and unknown covariances}

Like our algorithms, many recent algorithms for learning GMMs make use of the sum-of-squares semidefinite programming hierarchy.
While these algorithms and analyses have not been designed for our setting, we find it still instructive to discuss the differences and similarities to our algorithms.

Many of these algorithms also have in common that they employ the proof-to-algorithm paradigm, which has become the predominant way to analyze algorithms based on sum-of-squares for statistical estimation problems.
(For expositions of this paradigm, see \cite{MR3727623-BarakSteurerICM14,MR3966537-RaghavendraSchrammSteurerICM18,DBLP:journals/fttcs/FlemingKP19}.)
This paradigm allows us to derive efficient estimation algorithms in a black-box way from identifiability proofs formalized in the sum-of-squares proof system.

As mentioned earlier, several recent algorithms consider mixtures of few well-separated Gaussian components with unknown covariances in the presence of adversarial outliers
\cite{MR4232031-DiakonikolasHopkinsKothari2020,DBLP:journals/corr/abs-2005-02970,DBLP:journals/corr/abs-2005-06417}.
While these algorithms have running times (at least) exponential in the number of components, their separation requirements are also exponentially stronger than ours.
Even in the case that all covariances are the same (\(\Sigma\)) and the well-separatedness stems purely from the means \(\mu_1,\ldots,\mu_k\), their identifiability proof requires separation \(\norm{\Sigma^{-1/2}(\mu_a-\mu_b)}^2\ge k^{O(1)}\) (e.g., \cite[Lemma 4.16]{DBLP:journals/corr/abs-2005-02970}).
In constrast, our separation condition is logarithmic in \(k\), which is the weakest separation condition, up to constant factors, that guarantees clusterability.

In order to deal with the kind of mild separation considered in this work, one could use one of the (robust) algorithms for parameter learning or density estimation of general \(k\)-component GMMs \cite{DBLP:conf/focs/MoitraV10,DBLP:conf/colt/BelkinS10,DBLP:journals/corr/abs-2012-02119,DBLP:conf/stoc/LiuM21}.
While some of these works use separation between components in order to compute a rough partial clustering of far-away components as a pre-processing step, there doesn't appear to be a way to further exploit milder kinds of separation.
For example, \cite{DBLP:conf/focs/MoitraV10} learns the means of the mixture up to small error after projecting along a randomly chosen direction.
This kind of projection cannot be expected to preserve any kind of separation of the high-dimensional mixture and even the sample complexity for recovering the means of this 1-dimensional mixture may be exponential in \(k\) (as shown in \cite{DBLP:conf/focs/MoitraV10}).
Both of the more recent works \cite{DBLP:journals/corr/abs-2012-02119,DBLP:conf/stoc/LiuM21} end up enumerating subspaces related to the unknown parameters of the mixture.
To the best of our knowledge, their approaches cannot avoid this step even for the kind of mildly-separated mixtures with lower bounded mixing weights considered in our work.

\section{Techniques}
\label{sec:techniques}

We consider uniform\footnote{In this section we restrict ourselves for ease of explanation to uniform mixtures. Our technical sections state all results for non-uniform mixtures.} mixtures of \(k\ge 2\) well-separated Gaussian components \(N(\mu_1,\Sigma),\ldots,N(\mu_k,\Sigma)\) with unknown means \(\mu_1,\ldots,\mu_k\in\R^d\) and unknown covariance \(\Sigma\in\R^{d\times d}\) (identical for all components).
Here, we say components are \emph{well-separated}\footnote{
  The term "clusterable mixture" is sometimes used in the literature to refer to mixtures with well-separated components.}
if the maximum affinity\footnote{The affinity of two probability measures is defined to be \(1\) minus their statistical distance \cite{pollard2002user}.} (also called overlap) between two distinct components is bounded by \(1/k^C\) for a large enough constant \(C\ge 1\).
For Gaussian components, this notion of well-separatedness means
\begin{equation}
  \label{eq:well-separated-techniques}
  \min_{a\neq b} \Norm{\Sigma^{-1/2}(\mu_a-\mu_b)}
  \gg \sqrt{\log k\,}
  \,.
\end{equation}

\paragraph{Distinguishing well-separated mixtures from (pure) Gaussians}

Our algorithms are informed by investigating the parallel pancakes construction underlying Statistical Query lower bounds for such mixtures.
This construction provides a mixture of \(k\) well-separated Gaussian components that appears to be exponentially hard to distinguish from the standard Gaussian distribution \(N(0,I_d)\).
In particular, this mixture matches the first \(\Omega(k)\) moments of \(N(0,1)\).

The starting point of our algorithms is the following observation:
In order for a mixture with \(k\ge 2\) well-separated Gaussian components to match the first \(t\) moments of \(N(0,I_d)\), the minimum mixing weight is necessarily smaller than \(2^{-\Omega(t)}\).
In particular, if the mixture has uniform mixing weights \(\frac 1k\), then always one of its first \(O(\log k)\) moments distinguishes it from a standard Gaussian.

Underlying this observation is the following simple fact:
\emph{A distribution uniform over \(k\) real values can match no more than \(O(\log k)\) moments of \(N(0,1)\).}
To verify this fact, let \(\bm A\) be a random variable uniformly distributed over \(k\) (not necessarily distinct) real values.
Then, for all even integers \(s\le t\), the ratio of the normalized order-\(s\) and order-\(t\) moments of \(\bm A\) is sandwiched in the following way,
\begin{equation}
  \label{eq:moment-ratio-techniques}
  k^{-1/s}
  \le \frac{\Paren{\E \bm A^s}^{1/s}}{\Paren{\E \bm A^t}^{1/t}}
  \le 1
  \,.
\end{equation}
(This ratio is maximized if \(\bm A\) is constant and minimized if \(\Pr\set{\bm A\neq 0}=1/k\).)
In particular for \(s=\log_2 k\), this ratio is lower bounded by \(1/2\).
On the other hand, for \(\bm B\sim N(0,1)\), the normalized moments satisfy \(\Paren{\E \bm B^r}^{1/r}=\Theta(r)^{1/2}\) and thus the ratio of normalized order-\(s\) and order-\(t\) moments is \(\Theta(s/t)^{1/2}\).
In particular, for some choice \(t=\Theta(s)\), this ratio is smaller than \(1/2\).
It follows that for this choice of \(s\) and \(t\), the ratios of normalized moments differ for \(\bm A\) and for \(\bm B\), which means that either their order-\(s\) or their order-\(t\) moments differ.\footnote{
  This proof shows that in order to distinguish a uniform distribution over \(k\) values from \(N(0,1)\), it is enough to compare two moments of order logarithmic in \(k\) where the choice of orders depends only on \(k\) but not on the particular distribution. 
}

\medskip

This observation about uniform mixtures of \(k\ge 2\) well-separated Gaussian components raises two questions:
(1) do the first \(O(\log k)\) moments also allow us to identify parameters of the mixture that are useful for clustering (in addition to allowing us to distinguish the mixture from \(N(0,1)\)),
and (2) can we make make these results computationally efficient?

At a high level, we address question (1) by investigating ratios akin to \cref{eq:moment-ratio-techniques} between multivariate polynomials of degree \(\Theta(\log k)\) derived from moments of the underlying mixture.
To address question (2), we employ the proofs-to-algorithms paradigm (cf. \cite{MR3727623-BarakSteurerICM14,MR3966537-RaghavendraSchrammSteurerICM18})
and translate our arguments to syntactic proofs captured by the sum-of-squares proof system.
These proofs then allow us to derive efficient algorithms (with running time \((kd)^{O(\log k)}\) or \((kd)^{(\log k)^{O(1)}}\)) in a black-box way.

\paragraph{From decision to search: separating directions and ratios of moments}

In order to address question (1), we consider the goal of finding a direction \(v\in\R^d\) that may be useful for clustering in the sense that along direction \(v\), two of the components are significantly further apart than their standard deviation in this direction.
More formally, we say that \(v\) is a \emph{separating direction} for a mixture of \(k\) Gaussian components with unknown means \(\mu_1,\ldots,\mu_k\in \R^d\) and unknown covariance \(\Sigma\in\R^{d\times d}\) if there exist two means \(\mu_a\) and \(\mu_b\) such that
\begin{equation}
  \label{eq:separating-direction-techniques}
  \abs{\iprod{\mu_a-\mu_b,v}}
  \gg \sqrt{\log k} \cdot \norm{\Sigma^{1/2}v}
  \,.
\end{equation}
We note that this direction \(v\) witnesses that the overlap of the components \(N(\mu_a,\Sigma)\) and \(N(\mu_b,\Sigma)\) is \(k^{-\omega(1)}\).
Conversely, whenever the overlap of two components is that small, there exists a vector \(v\) as above. 

We aim to identify separating directions as solutions to inequalities between the following kind of moment polynomials:
For \(r\in \N\), we denote the degree-\(2r\) \emph{moment polynomial} \(p_{2r}\in\R_{2r}[v]\) by
\begin{equation}
  \label{eq:moment-polynomial-techniques}
  p_{2r}(v) \seteq \E \iprod{\bm y-\bm y', v}^{2r} \,,
\end{equation}
where \(\bm y,\bm y'\) are two independent random vectors identically distributed according to a uniform mixture of \(k\) Gaussian components \(N(\mu_1,\Sigma),\ldots,N(\mu_k,\Sigma)\).

Using the fact that \(\bm y-\bm y'\) can be expressed as a sum of two independent random vectors, one distributed uniformly over \(\set{\mu_a-\mu_b}_{a,b\in [k]}\) and one distributed according to \(N(0,2\Sigma)\), these polynomials turn out to admit the following kind of approximation,
\begin{equation}
  \label{eq:moment-polynomial-approximation-techniques}
  p_{2r}(v) = \Paren{k^{-2/r}\cdot \Norm{M v}_{2r}^2 + \Theta(r)\cdot \Norm{A v}_2^2}^{r}\,.
\end{equation}
Here, \(A=\sqrt 2 \cdot \Sigma^{1/2}\), \(M\in\R^{k^2\times d}\) consists of the differences of means \((\mu_a-\mu_b)_{a,b\in [k]}\subseteq \R^d\) as rows and \(\Theta(r)\) hides a nonnegative function upper bounded \(O(r)\) and lower bounded by \(\Omega(r/k^{2/r})\).
(Since we will only consider \(r\ge \log k\), we have \(k^{-2/r} \ge \Omega(1)\).)
Note that the first term \(k^{-2}\cdot \Norm{M v}_{2r}^{2r}\) in (the binomial expansion of) \cref{eq:moment-polynomial-approximation-techniques} corresponds to the order-\(r\) moment of the uniform distribution over \(\set{\mu_a-\mu_b}_{a,b\in[k]}\) and the last term \(\Theta(r)^r\cdot \iprod{v, 2\Sigma v}^{r}\) to the order-\(r\) moment of \(N(0,2\Sigma)\).

We claim that for an appropriate choice \(s \le t\) with \(s=\Theta(t)=\Theta(\log k)\), a direction \(v\) is separating in the sense of \cref{eq:separating-direction-techniques} if and only \(p_{2s}(v)^{1/s} \gtrsim p_{2t}(v)^{1/t}\).
(Note that by convexity, \(p_{2s}(v)^{1/s} \le p_{2t}(v)^{1/t}\) holds for all directions \(v\).)
Underlying this claim is the familiar fact  that for all \(r\ge \log k\), the norm \(\Norm{M v}_{2r}\) equals up to constant factors the maximum entry of \(M v\), i.e., \(\max_{a\neq b}\abs{\iprod{\mu_a-\mu_b,v}}\). 

Indeed, suppose that \(v\) is a separating direction.
Then, \(p_{2s}(v)\) satisfies the lower bound,
\begin{equation}
  \label{eq:lower-bound-techniques}
  p_{2s}(v)^{1/s}
  \ge k^{-2/s}\cdot \max_{a\neq b} \iprod{\mu_a-\mu_b,v}^2
  \,.
\end{equation}
Since \(s=\Theta(\log k)\), we have \(p_{2s}(v)^{1/s}\gtrsim \max_{a\neq b}\iprod{\mu_a-\mu_b,v}^2\).
At the same time, \(p_{2t}(v)\) satisfies the upper bound,
\begin{equation}
  \label{eq:upper-bound-techniques}
  p_{2t}(v)^{1/t}
  \le \max_{a\neq b} \iprod{\mu_a-\mu_b,v}^2 + O(t)\cdot \norm{Av}_2^2
  \,.
\end{equation}
Since \(v\) is a separating direction and \(t=\Theta(\log k)\), the upper bound is dominated by the first term \(\max_{a\neq b} \iprod{\mu_a-\mu_b,v}^2\).
Taking together both bounds, it follows that \(p_{2s}(v)^{1/s} \gtrsim p_{2t}(v)^{1/t}\) for every separating direction \(v\).

Conversely, suppose that \(p_{2s}(v)^{1/s} \gtrsim  p_{2t}(v)^{1/t}\) and our goal is to show that \(v\) is a separating direction.
We lower bound \(p_{2t}(v)\) using the last term in the approximation \cref{eq:moment-polynomial-approximation-techniques} and apply the upper bound from \cref{eq:upper-bound-techniques} to \(p_{2s}(v)\).
In this way, we obtain the inequality
\begin{equation}
  \Omega(t) \cdot \norm{Av}_2^2
  \le \max_{a\neq b}\iprod{\mu_a-\mu_b,v}^2 + O(s)\cdot \norm{Av}_2^2
  \,.
\end{equation}
By choosing \(t\) to be a large enough constant multiplied by \(s\), we can ensure that the second term on the right-hand side is negligible.
In this case, \(v\) satisfies \(\max_{a\neq b}\iprod{\mu_a-\mu_b,v}^2 \ge \Omega(t) \cdot \norm{Av}_2^2\) for \(t=\Theta(\log k)\), which means that \(v\) is a separating direction.

\paragraph{Challenges toward efficient algorithms for clustering}

Disregarding computational efficiency, the above characterization of separating directions in terms of ratios of moment polynomials suggests the following simple strategy for clustering uniform mixtures of Gaussian components \(N(\mu_1,\Sigma),\dots,N(\mu_k,\Sigma)\):
we find an \(\e\)-cover of all separating directions by brute-force searching for an \(\e\)-cover of all solutions to an explicit polynomial system of the form \(\set{p_{2s}(v)=1,~p_{2t}(v)\le O(1)^{t}}\).
Each separating direction gives us some information about what pairs of sample points belong to different components.
For large enough mean separation \(\min_{a\neq b}\norm{\Sigma^{-1/2}(\mu_a-\mu_b)}\gg \sqrt{\log k\,}\), we can hope that by considering all such directions, we collect enough information to be able to extract a clustering of the sample that is approximately consistent with the components of the mixture.

This naive approach would require access only to moments of order \(O(\log k)\) (which could be accurately estimated from a sample of size \(d^{O(\log k)}\)) but the running time is exponentially large (due to brute-force searching for solutions to a polynomial system). 

A natural strategy to make this approach computationally efficient is the sum-of-squares hierarchy of semidefinite programming relaxations for systems of polynomial inequalities.
Indeed, we can show that the above characterization of separating directions is faithfully captured by the sum-of-squares proof system underlying the sum-of-squares hierarchy.
Unfortunately, it appears to be challenging to carry out the rounding step in full generality, i.e., extracting from the sum-of-squares hierarchy enough separating directions to separate all pairs of components and obtain a complete clustering of the sample.\footnote{
  In the context of estimation problems, rounding procedures for sum-of-squares hierarchies tend to work well if there is a unique target solution (e.g., a planted sparse vector in a random subspace) or if there is a small number of target solutions (e.g., the components of a low-rank tensor).
  One could try to simplify the structure of the set of separating directions, e.g., by focusing on "extreme" separating directions of the form \(v=\Sigma^{-1}(\mu_a-\mu_b)\).
  Unfortunately, we do not know the same kind of characterization in terms of polynomial inequalities for such a simplified set of separating directions.  
}

However, we can show that using the sum-of-squares hierarchy, it is possible to separate at least \emph{some} pairs of components of the mixture by what we call a separating polynomial.
Furthermore, for the special case of well-separated components with colinear means, we provide a more careful analysis and show that in this case the sum-of-squares hierarchy does offer enough information to extract a complete clustering.   

\paragraph{Efficiently computing a separating polynomial}

As discussed above, we consider the goal of separating some pairs of components of a mixture (as opposed to the stronger goal of separating all pairs of components as would be required for a complete clustering).
One way to achieve this goal is by finding a separating direction in the sense of \cref{eq:separating-direction-techniques}.
In light of our previous characterization of separating directions, a natural starting point is a sum-of-squares relaxation for a polynomial system \(\cA\) of the form \(\set{p_{2s}(v)=1,~p_{2t}(v)\le O(1)^{t}}\) for appropriate \(s\le t\) satisfying \(s=\Theta(t)=\Theta(\log k)\). 

Unfortunately, the structure of the set of separating directions does not appear to be amenable to the usual kind of rounding techniques for sum-of-squares relaxations, and it appears to be challenging to extract a single separating direction.
To overcome this obstacle, we allow our rounding procedure to output a more general object, called a \emph{separating polynomial}, that still allows us to separate some pairs of mixture components.

Recall that a solution to a sum-of-squares relaxation for a polynomial system \(\cA\) can be interpreted as \emph{pseudo-distribution} \(D\) that behaves in certain ways like a distribution supported on vectors satisfying \(\cA\).
More concretely, the pseudo-distribution \(D\) satsifies (in expectation) all polynomial inequalities that can be derived syntactically from \(\cA\) by a low-degree sum-of-squares proof (see \cref{sec:preliminaries}, especially \cref{def:sos-proof-preliminaries}).
The previously discussed characterization of separating directions in terms of the polynomial system \(\cA\) turns out to be captured by low-degree sum-of-squares proofs.
Concretely, we can derive from \(\cA\) via low-degree sum-of-squares proof the polynomial inequality\footnote{
  Here, we reuse the notation introduced in the context of \cref{eq:moment-polynomial-approximation-techniques}.}
\(\norm{M v}_{2s}^{2s} \ge (C \log k)^{s} \cdot \norm{A v}_2^{2s}\) (corresponding to \cref{eq:separating-direction-techniques}).
Here, \(C\ge 1\) is an absolute constant that we can choose as large as we like.
Consequently, the pseudo-distribution \(D\) satisfies this inequality in expectation, 
\begin{math}
  \tilde \E_{D(v)} \norm{M v}_{2s}^{2s}
  \ge (C \log k)^{s}\cdot\tilde \E_{D(v)} \norm{Av}^{2s}_2
  \,.   
\end{math}
By linearity of (pseudo-)expectation, there exist distinct components \(a\neq b\) such that
\begin{equation}
  \label{eq:pseudo-expectation-separating-direction-techniques}
  \tilde {\E_{D(v)}} \iprod{\mu_a-\mu_b,v}^{2s}
  \ge k^{-2} \tilde {\E_{D(v)}} \norm{Mv}_{2s}^{2s}
  \ge (k^{-2/s}\cdot C \log k)^{s}\cdot\tilde {\E_{D(v)}} \norm{Av}^{2s}_2
  \,.     
\end{equation}

We extract the following polynomial from this pseudo-distribution,
\begin{equation}
  \label{eq:separating-polynomial-techniques}
  q(u)
  \seteq \tilde {\E_{D(v)}} \iprod{u,v}^{2s}
  \,.
\end{equation}
By construction, \(q(\mu_a-\mu_b)\) equals the left-hand side of \cref{eq:pseudo-expectation-separating-direction-techniques}.
At the same time, letting \(\bm y,\bm y'\sim N(\mu_c,\Sigma)\) and \(\bm w\sim N(0,I_d)\), we have
\begin{align*}
  \E q(\bm y-\bm y')
  &= \E q(A\bm w)\\
  &= \tilde {\E_{D(v)}} \E \iprod{v,A \bm w}^{2s}\\
  &=  (2s-1)!! \cdot\tilde {\E_{D(v)}} \norm{A \bm w}^{2s}_2
\end{align*}
Consequently, since \(s=\Theta(\log k)\) and \((2s-1)!!\le O(s)^{s}\),
\begin{equation}
  \frac{q(\mu_a-\mu_b)}{\E q(A \bm w)}
  \ge \frac{(k^{-2/s}\cdot C \log k)^{s}}{(2s-1)!!}
  \ge \Omega(C)^s
  \,. 
\end{equation}
For an appropriate choice of \(C\ge 1\), the right-hand side above is at least \(10^s\).
Since by convexity \(\E q(\mu_a - \mu_b + A\bm w)\ge q(\mu_a-\mu_b)\), we obtain the following inequality,
\begin{equation}
  \label{eq:separation-guarantee-techniques}
  \frac{\E q(\mu_a-\mu_b+A\bm w)}{\E q(A \bm w)}
  \ge 10^s
  \,.
\end{equation}
This inequality shows that the polynomial \(q(u)\) separates the components \(N(\mu_a,\Sigma)\) and \(N(\mu_b,\Sigma)\) in the following sense:
The numerator of \cref{eq:separation-guarantee-techniques} is the typical value of \(q(\bm y-\bm y')\) for \(\bm y\sim N(\mu_a,\Sigma)\) and \(\bm y'\sim N(\mu_b,\Sigma)\).
The denominator of  \cref{eq:separation-guarantee-techniques} is the typical value of \(q(\bm y-\bm y')\) for \(\bm y,\bm y'\sim N(\mu_c,\Sigma)\) and all \(c\in [k]\).
\Cref{eq:separation-guarantee-techniques} asserts that the gap between these values is at least \(10^s\). 

The polynomial $q(u)$ can be used to compute a bipartition of the sample that separates at least one pair of components.
Note that $q(u)^{1/2s} = (\tilde{\E}_{D(v)} \iprod{u,v}^{2s})^{1/2s}$ satisfies the triangle inequality (see Lemma 4.5 in \cite{MR3727623-BarakSteurerICM14}).
Then we can define the distance function $d_q(x, y) = q(x - y)^{1/2s}$ and use it in a greedy algorithm in order to obtain the bipartition.

\paragraph{Efficiently computing a clustering for colinear means}

For the case that the means are colinear, we consider a strengthening of our previous approach.
Instead of trying to solve a polynomial system of the form \(\set{p_{2s}(v)=1,~p_{2t}(v)\le O(1)^{t}}\), we aim to solve the following related optimization problem: 
\begin{equation}
  \label{eq:optimization-techniques}
  \text{minimize} \quad
  \frac{p_{2t}(v)^{1/t}} {p_{2s}(v)^{1/s}}
  \quad \text{subject to}\quad
  v\in\R^d.
\end{equation}
Algorithmically, we again employ an appropriate sum-of-squares formulation.

To simplify some of our arguments, it is useful to preprocess the mixture and bring \(\bm y-\bm y'\) in isotropic position so that \(\tfrac 1{k^2}  \sum_{a,b=1}^k \dyad{(\mu_a-\mu_b)} + 2\Sigma = I_d\).
(Here, \(\bm y, \bm y'\) are two independent random vectors distributed according to the mixture.)
For every vector \(v\), we denote by \(v^{\parallel}\) its orthogonal projection into the span of \(\set{\mu_a-\mu_b}_{a,b\in [k]}\) and by \(v^{\perp} = v-v^{\parallel}\) its projection into the orthogonal complement.

Every optimizer \(v\) of \cref{eq:optimization-techniques} necessarily satisfies,
\begin{equation}
  \label{eq:optimality-techniques}
  \frac{p_{2t}(v)^{1/t}} {p_{2s}(v)^{1/s}}
  \le   \frac{p_{2t}(v^{\parallel})^{1/t}} {p_{2s}(v^{\parallel})^{1/s}}
  \le O(1)
  \,.
\end{equation}
Here, the upper bound \(O(1)\) hides an absolute constant whenever we have well-separated components and \(\log k \le s \le t\).
The argument for this upper bound is similar to our discussion for the characterization of separating directions.

We can also use the decomposition \(v=v^{\parallel}+v^{\perp}\) for our previous approximation \cref{eq:moment-polynomial-approximation-techniques} of moment polynomials,
\begin{equation}
  \label{eq:moment-polynomial-approximation-colinear-techniques}
  p_{2r}(v)
  = \Paren{k^{-2/r}\cdot \Norm{M v^{\parallel}}_{2r}^2 + \Theta(r)\cdot \Paren{\Norm{A v^{\parallel}}_2^2 + \Norm{v^{\perp}}_2^2}}^{r}
  \,.
\end{equation}
Here, we use that after bringing \(\bm y-\bm y'\) in isotropic position, the covariance \(\Sigma\) acts as identity orthogonal to the span of \(\set{\mu_a-\mu_b}_{a,b\in [k]}\).
In particular, \(A v = A v^{\parallel} + v^{\perp}\) and \(\norm{Av}^2_2 = \norm{A v^{\parallel}}_2^2 + \norm{v^{\perp}}_2^2\).

An immediate consequence of \cref{eq:moment-polynomial-approximation-colinear-techniques} is the following representation of the ratio we seek to minimize,
\begin{equation}
  \label{eq:ratio-approximation-colinear-techniques}
  \frac{p_{2t}(v)^{1/t}} {p_{2s}(v)^{1/s}}
  = \frac{
    p_{2t}(v^{\parallel})^{1/t} + \Theta(t)\cdot \norm{v^{\perp}}_2^2 \pm \Theta(t)\cdot \norm{A v^{\parallel}}_2^2
  }{
    p_{2s}(v^{\parallel})^{1/s} + \Theta(s)\cdot \norm{v^{\perp}}_2^2 \pm \Theta(s)\cdot \norm{A v^{\parallel}}_2^2
  }\,. 
\end{equation}
We claim that for an appropriate choice of \(s\) and \(t\) \cref{eq:ratio-approximation-colinear-techniques} and \cref{eq:optimality-techniques} together imply that \(\norm{v^{{\perp}}}\lesssim \norm{A v^{\parallel}}\).
Indeed, for the sake of a contradiction, suppose \(\norm{v^{{\perp}}}\gg \norm{A v^{\parallel}}\) so that the terms involving \(\norm{A v^{\parallel}}\) in \cref{eq:ratio-approximation-colinear-techniques} are negligible.
But then, if we choose \(t\) as \(s\) times a sufficiently larger constant factor, the remaining ratio \(\frac{p_{2t}(v^{\parallel})^{1/t} + \Theta(t)\cdot \norm{v^{\perp}}_2^2 } {p_{2s}(v^{\parallel})^{1/s} + \Theta(s)\cdot \norm{v^{\perp}}_2^2}\) is strictly bigger than \(\frac{p_{2t}(v^{\parallel})^{1/t} } {p_{2s}(v^{\parallel})^{1/s}}\), which contradicts our optimality condition \cref{eq:optimality-techniques}.
(For this argument, we are also using the previous upper bound \(\frac{p_{2t}(v^{\parallel})^{1/t} } {p_{2s}(v^{\parallel})^{1/s}}\le O(1)\) from \cref{eq:optimality-techniques}.)

It turns out that in order to compute a clustering for colinear means, it suffices to find a vector satisfying \(\norm{v^{{\perp}}}\lesssim \norm{A v^{\parallel}}\).

We note that the algorithm we present in \cref{sec:pp,sec:colinear} to find such a direction $v$ follows a somewhat different strategy and minimizes ratios of the form \({\norm{v}_2^2}/{p_{2s}(v)^{1/s}}\) or \({p_{2t}(v)^{1/t}}/{\norm{v}_2^2}\) via appropriate sum-of-squares formulations.

\section{Preliminaries}
\label{sec:preliminaries}

In this section we introduce sum-of-squares proofs and their duals, pseudo-distributions and pseudo-expectations.

\paragraph{Sum-of-squares proofs.}

\begin{definition}[Sum-of-squares proofs]
\label{def:sos-proof-preliminaries}
Let $p(x)$ and $q_1(x), ..., q_m(x)$ be polynomials over $x \in \mathbb{R}^n$ and let $\mathcal{A} = \{q_1(x) \geq 0, ..., q_m(x) \geq 0\}$ be a system of polynomial inequalities.
A \textit{sum-of-squares proof of degree $t$} that $p(x) \geq 0$ under $\mathcal{A}$ is an identity of the form
\begin{equation}
p(x) = \sum_{S \subseteq [m]} \left(\sum_{i=1}^{m_S} r_{S,i}(x)^2\right) \prod_{j\in S} q_j(x)
\end{equation}
for polynomials $r_{S, i}(x)$, such that $\max_{S, i} \operatorname{deg}(r_{S, i}(x)^2 \prod_{j \in S} q_j(x)) \leq t$.
\end{definition}

If there exists a sum-of-squares proof of degree $t$ that $p(x) \geq 0$ under $\mathcal{A}$, we write $\mathcal{A} \sststile{t}{x} p(x) \geq 0$.
We also use the notation $\mathcal{A} \sststile{t}{x} p(x) \geq q(x)$ if $\mathcal{A} \sststile{t}{x} p(x) - q(x) \geq 0$ and $\mathcal{A} \sststile{t}{x} p(x) \leq q(x)$ if $\mathcal{A} \sststile{t}{x} q(x) - p(x) \geq 0$.
If $\mathcal{A}=\emptyset$, we omit it altogether and write $\sststile{t}{x} p(x) \geq 0$. 
We also sometimes omit $\mathcal{A}$ if it is clear from context what axioms are assumed.
We note that if $\mathcal{A} \sststile{t}{x} p(x) \geq q(x)$ and $\mathcal{A} \sststile{t}{s} q(x) \geq r(x)$, then $\mathcal{A} \sststile{t}{x} p(x) \geq r(x)$, which allows writing chains of inequalities of the form $\mathcal{A} \sststile{t}{x} p(x) \geq s(x) \geq r(x)$.

\paragraph{Pseudo-distributions and pseudo-expectations.}

We begin by defining pseudo-distributions and pseudo-expectations.

\begin{definition}[Pseudo-distributions]
A \textit{pseudo-distribution} $D$ of degree $t$ is a function from $\mathbb{R}^n$ to $\mathbb{R}$ with finite support such that $\sum_{x \in \operatorname{supp}(D)} D(x) = 1$ and $\sum_{x \in \operatorname{supp}(D)} D(x) p(x)^2 \geq 0$ for all polynomials $p(x)$ with $\operatorname{deg}(p(x)^2) \leq t$.
\end{definition}

\begin{definition}[Pseudo-expectations]
Given a pseudo-distribution $D$ of degree $t$, the associated \textit{pseudo-expectation} $\tilde{\mathbb{E}}_{D(x)}$ is defined by $\tilde{\mathbb{E}}_{D(x)} f(x) = \sum_{x \in \operatorname{supp}(D)} D(x) f(x)$ for a function $f(x)$.
\end{definition}

We now define the notion of a pseudo-distribution that satisfies a set of polynomial inequalities.

\begin{definition}[Constrained pseudo-distributions]
A pseudo-distribution $D$ of degree $t$ \textit{satisfies} the set of polynomial inequalities $\mathcal{A} = \{q_1(x) \geq 0, ..., q_m(x)\geq 0\}$ if, for all $S \subseteq [m]$, $\tilde{\mathbb{E}}_{D(x)} r(x)^2 \prod_{j \in S} q_j(x) \geq 0$ for all polynomials $r(x)$ such that $\operatorname{deg}(r(x)^2 \prod_{j \in S}q_j(x)) \leq t$.

$D$ \textit{approximately satisfies} $\mathcal{A}$ up to error $\eta$ if, under the same conditions as in the previous case, $\tilde{\mathbb{E}}_{D(x)} r(x)^2 \prod_{j \in S} q_j(x) \geq -\eta \|r(x)^2\|_2 \prod_{j \in S} \|q_j(x)\|_2$, where $\|p(x)\|_2$ denotes the $2$-norm of the vector of coefficients of the polynomial $p(x)$.
\end{definition}

The connection between pseudo-distributions and sum-of-squares proofs is made in \cref{fact:pe-satisfies-sos}, which shows that if a pseudo-distribution satisfies a set of polynomial inequalities, then it also satisfies any other polynomial inequalities derived from this set through sum-of-squares proofs.

\begin{fact}
\label{fact:pe-satisfies-sos}
If $D$ is a pseudo-distribution of degree $t$ that satisfies $\mathcal{A}$ and if $\mathcal{A} \sststile{s}{x} p(x) \geq 0$, then ${\tilde{\mathbb{E}}_{D(x)} r(x)^2 p(x) \geq 0}$ for all polynomials $r(x)$ such that $\operatorname{deg}(r(x)^2 p(x)) \leq t$.
If $D$ approximately satisfies $\mathcal{A}$ up to error $\eta$, then, under the same conditions as in the previous case, ${\tilde{\mathbb{E}}_{D(x)} r(x)^2 p(x) \geq -\eta \|r(x)^2\|_2 \|p(x)\|_2}$.
\end{fact}

Finally, \cref{fact:sos-to-pe} shows that there exists an algorithm with time complexity $(n+m)^{O(t)}$ to compute a pseudo-distribution of degree $t$ that approximately satisfies $\mathcal{A}$ up to error $2^{-n^{\Theta(t)}}$.

\begin{fact}
\label{fact:sos-to-pe}
For $x \in \mathbb{R}^n$, if $\mathcal{A} = \{q_1(x) \geq 0, ..., q_m(x) \geq 0\}$ is feasible and explicitly bounded
\footnote{Explicit boundedness means that $\mathcal{A}$ contains a constraint of the form $x_1^2 + ... + x_n^2 \leq B$.
In our applications it is possible to add such a constraint with $B$ large enough such that the constraint is always satisfied by the intended solution.}
, then there exists an algorithm that runs in time $(n+m)^{O(t)}$ and computes a pseudo-distribution of degree $t$ that approximately satisfies $\mathcal{A}$ up to error $2^{-n^{\Theta(t)}}$
\footnote{In our applications this error is negligible.}.
\end{fact}

\section{Separating polynomial}
\label{sec:separatingpolynomial}

\paragraph{Setting.} We consider a mixture of \(k\) Gaussian distributions \(N(\mu_i, \Sigma)\) with mixing weights \(p_i\) for \(i=1,...,k\),
where \(\mu_i \in \mathbb{R}^d\), \(\Sigma \in \mathbb{R}^{d \times d}\) is positive definite, and \(p_i \geq 0\) and \(\sum_{i=1}^k p_i = 1\). 
Let \(\pmin = \min_i p_i\).

The distribution satisfies mean separation for at least one pair of means:
for some $C_{sep} > 0$, there exist $a, b \in [k]$ such that
\[\left\|\Sigma^{-1/2} (\mu_a - \mu_b)\right\|^2 \geq C_{sep} \log \pmin^{-1}.\] 

\begin{theorem}[Separating polynomial algorithm]
\label{thm:separating-polynomial-main}
Consider the Gaussian mixture model defined above, with $C_{sep}$ larger than some universal constant. Let $n_0 = (\pmin^{-1}d)^{O(\log \pmin^{-1})}$. Given a sample of size $n \geq n_0$ from the mixture, there exists an algorithm that computes in time $n \cdot d^{O(\log \pmin^{-1})}$ a $d$-variate degree-$O(\log \pmin^{-1})$ polynomial $q$ such that with high probability the following two properties hold. Let $s=\lceil \log \pmin^{-1}\rceil$. Then:
\begin{itemize}
\item There exist distinct \(a,b\in [k]\) such that the independent random vectors \(\bm y\sim N(\mu_a,\Sigma)\) and \(\bm y'\sim N(\mu_b,\Sigma)\) satisfy 
  \[
    \Pr\Set{q(\bm y - \bm y') \ge \frac{1}{20^{s}}}
    \ge 0.99999.
    \,
  \]
\item For all \(a\in [k]\), the independent random vectors \(\bm y,\bm y'\sim N(\mu_a,\Sigma)\) satisfy
  \[
    \Pr\Set{q(\bm y - \bm y') \le \frac{1}{200^{s}}}
    \ge 0.99999.
    \,
  \]
\end{itemize}
\end{theorem}

\begin{theorem}[Separating bipartition algorithm]
\label{thm:separating-bipartition-main}
Consider the Gaussian mixture model defined above, with $C_{sep}$ larger than some universal constant. Let $n_0 = (\pmin^{-1}d)^{O(\log \pmin^{-1})}$. Given a sample of size $n \geq n_0$ from the mixture, there exists an algorithm that runs in time $n \cdot d^{O(\log \pmin^{-1})}$ and returns with probability $0.99$ a partition of $[n]$ into two sets $C_1$ and $C_2$ such that, if true clustering of the samples is $S_1, ..., S_k$, then
\[\max_i \frac{|C_1 \cap S_i|}{|S_i|} \geq 0.99 \quad \text{and} \quad \max_i \frac{|C_2 \cap S_i|}{|S_i|} \geq 0.99.\]  
\end{theorem}

We introduce some further notation for this section. Let the random variable \(\bm{z} \in \mathbb{R}^d\) be distributed according to the difference of two independent samples from the mixture. Then \(\bm{z}\) is distributed according to a mixture of Gaussians $N(\mu_i - \mu_j, 2\Sigma)$ with mixing weights \(p_i p_j\) for all \(i, j \in [k]\). Let $\Sigma_z = 2\Sigma$, let $\bm{\mu}_z$ be $\mu_i - \mu_j$ with probability \(p_i p_j\), and let \(\bm{w}_z \sim N(0, \Sigma_z)\). Then we also have that $\bm{z} = \bm{\mu}_z + \bm{w}_z$, with $\bm{\mu}_z$ and $\bm{w}_z$ independent of each other.

\subsection{Exact moment results}

The main ingredient of the algorithm is \cref{lemma:separating-polynomial-completeness}, stated below.
This lemma shows that, given a pseudo-expectation that satisfies the moment lower bound $\mathbb{E} \langle \bm z, v\rangle^{2s} \geq c^s$ and the moment upper bound $\mathbb{E} \langle \bm z, v\rangle^{2t} \leq C^t$ for $s \ll t$, it is possible to construct a separating polynomial.
Note that the constraints that the pseudo-expectation satisfies are expressed in terms of exact moments of the distribution, to which we do not have access.
Finite sample considerations are discussed starting with \cref{subsec:separating-polynomial-finite-sample}.

\begin{lemma}[Separating polynomial from pseudo-expectation]
\label{lemma:separating-polynomial-completeness}
Let $c > 0$ and $C \geq 0$. Let $s \geq 1$ and $t \geq 50000 C s/c$ integers.
Given a pseudo-expectation \(\tilde{\mathbb{E}}\) of degree at least $2t$ over a variable \(v \in \mathbb{R}^d\) that satisfies $\{\mathbb{E} \langle \bm z, v\rangle^{2s} \geq c^s, \mathbb{E} \langle \bm z, v\rangle^{2t} \leq C^t\}$, let $q(u) = \langle \tilde{\mathbb{E}} v^{\otimes 2s}, u^{\otimes 2s}\rangle$.
Then:

\begin{itemize}
  \item There exist distinct \(a,b\in [k]\) such that the independent random vectors \(\bm y\sim N(\mu_a,\Sigma)\) and \(\bm y'\sim N(\mu_b,\Sigma)\) satisfy 
    \[
      \Pr\Set{q(\bm y - \bm y') \ge \frac{1}{2}\left(\frac{c}{16}\right)^s}
      \ge 0.99999
      \,.
    \]
  \item For all \(a\in [k]\), the independent random vectors \(\bm y,\bm y'\sim N(\mu_a,\Sigma)\) satisfy
    \[
      \Pr\Set{q(\bm y - \bm y') \le 320 \left(\frac{4Cs}{t}\right)^s}
      \ge 0.99999
      \,.
    \]
  \end{itemize}
\end{lemma}

In what follows, we prove a number of supporting lemmas, after which we prove \cref{lemma:separating-polynomial-completeness}.
Then, we state and prove \cref{lemma:separating-polynomial-soundness}, which shows that there exists a vector $v \in \mathbb{R}^d$ which satisfies the constraints required by \cref{lemma:separating-polynomial-completeness}.

We proceed with the supporting lemmas. \cref{lemma:moment-upper-bound-weak} and \cref{lemma:moment-lower-bound-weak} give sum-of-squares bounds on the moments of the mixture.

\begin{lemma}[Moment upper bound]
\label{lemma:moment-upper-bound-weak}
For $t \geq 1$ integer,
\[\sststile{2t}{v} \mathbb{E} \langle \bm{z}, v \rangle^{2t} \leq 2^{2t-1} \mathbb{E} \langle \bm{\mu}_z, v\rangle^{2t} + 2^{2t-1} (v^\top \Sigma_z v)^t t^t.\]
\end{lemma}

\begin{proof}
\begin{align*}
  \sststile{2t}{v} \mathbb{E} \langle \bm z, v\rangle^{2t}
&= \mathbb{E} (\langle \bm{\mu}_z, v\rangle + \langle \bm{w}_z, v\rangle)^{2t}
\stackrel{(1)}{\leq} 2^{2t-1} \mathbb{E} \langle \bm{\mu}_z, v\rangle^{2t} + 2^{2t-1} \mathbb{E} \langle \bm{w}_z, v\rangle^{2t}\\
&\stackrel{(2)}{\leq} 2^{2t-1} \mathbb{E} \langle \bm{\mu}_z, v\rangle^{2t} + 2^{2t-1} (v^\top \Sigma_z v)^t t^{t}
\end{align*}
where in (1) we used \cref{lemma:sos-triangle-two}
and in (2) we used that \(\mathbb{E}\langle \bm{w}_z, v\rangle^{2t} = (v^\top \Sigma_z v)^{t} (2t-1)!! \leq (v^\top \Sigma_z v)^t t^{t}\).
\end{proof}

\begin{lemma}[Moment lower bound]
\label{lemma:moment-lower-bound-weak}
For $t \geq 1$ integer,
\[\sststile{2t}{v} \mathbb{E} \langle \bm z, v\rangle^{2t} \geq \mathbb{E} \langle \bm{\mu}_z, v\rangle^{2t} + (v^\top \Sigma_z v)^t \frac{t^t}{2^t}.\]
\end{lemma}

\begin{proof}
\begin{align*}
\sststile{2t}{v} \mathbb{E} \langle \bm z, v\rangle^{2t}
&= \mathbb{E} (\langle \bm{\mu}_z, v\rangle + \langle \bm{w}_z, v\rangle)^{2t}
= \sum_{j=0}^{2t} \binom{2t}{j} \mathbb{E} \langle \bm{\mu}_z, v\rangle^{j} \langle \bm{w}_z, v\rangle^{2t-j}\\
&\stackrel{(1)}{=} \sum_{j=0}^{2t} \binom{2t}{j} \mathbb{E} \langle \bm{\mu}_z, v\rangle^{j} \mathbb{E} \langle \bm{w}_z, v\rangle^{2t-j}
\stackrel{(2)}{=} \sum_{s=0}^t \binom{2t}{2s} \mathbb{E} \langle \bm{\mu}_z, v\rangle^{2s} \mathbb{E} \langle \bm{w}_z, v\rangle^{2t-2s}\\
&\geq \mathbb{E} \langle \bm{\mu}_z, v\rangle^{2t} + \mathbb{E} \langle \bm{w}_z, v\rangle^{2t}
\stackrel{(3)}{\geq} \mathbb{E} \langle \bm{\mu}_z, v\rangle^{2t} + (v^\top \Sigma_z v)^t \frac{t^t}{2^t}
\end{align*}
where in (1) we used that \(\bm{\mu}_z\) and \(\bm{w}_z\) are independent, in (2) we used that \(\mathbb{E}\langle \bm{w}_z, v\rangle^{2t-j} = 0\) for \(2t-j\) odd, and in (3) we used that \(\mathbb{E}\langle \bm{w}_z, v\rangle^{2t} = (v^\top \Sigma_z v)^{t} (2t-1)!! \geq (v^\top \Sigma_z v)^t \frac{t^t}{2^t}\).
\end{proof}

Going forward, \cref{lemma:separating-polynomial-variance-upper-bound} proves that, if the moments of $\bm{z}$ are small in direction $v$, then the variance of the components of the mixture is also small in direction $v$.
Given in addition an upper bound on the moments of $\bm{z}$ in direction $v$ for a sufficiently large moment, \cref{lemma:separating-polynomial-mean-lower-bound} proves that the contribution of the means in direction $v$ is large.

\begin{lemma}
\label{lemma:separating-polynomial-variance-upper-bound}
Let $C \geq 0$. For \(t \geq 1\) integer,
\[\left\{\mathbb{E} \langle \bm z, v\rangle^{2t} \leq C^t\right\} \sststile{2t}{v} \left\{ v^\top \Sigma_z v \leq \frac{2C}{t} \right\}.\]
\end{lemma}

\begin{proof}
Substitute the lower bound of \cref{lemma:moment-lower-bound-weak} into the axiom:
\[\sststile{2t}{v} \mathbb{E} \langle \bm{\mu}_z, v\rangle^{2t} + (v^\top \Sigma_z v)^t t^t/2^t \leq C^t.\]
Use that $\sststile{2t}{v} \mathbb{E} \langle \bm{\mu}_z, v\rangle^{2t} \geq 0$ to drop the first term and then divide by $t^t/2^t$. This proves that $\sststile{2t}{v} (v^\top \Sigma_z v)^t \leq 2^tC^t/t^t.$
Finally, by \cref{lemma:sos-square-root-bound}, this implies that $\sststile{2t}{v} v^\top \Sigma_z v \leq 2C/t$.
\end{proof}

\begin{lemma}
\label{lemma:separating-polynomial-mean-lower-bound}
Let $c > 0$ and $C \geq 0$. For $s \geq 1$ and $t \geq 16 C s/c$ integers,
\[\left\{\mathbb{E} \langle \bm{z}, v\rangle^{2s} \geq c^s, \mathbb{E} \langle \bm z, v\rangle^{2t} \leq C^t\right\} \sststile{2t}{v} \left\{ \mathbb{E}\langle \bm{\mu}_z, v\rangle^{2s} \geq \left(\frac{c}{4}\right)^s \right\}.\]
\end{lemma}

\begin{proof}
Substitute the upper bound of \cref{lemma:moment-upper-bound-weak} into the first axiom:
\[\sststile{2s}{v} 2^{2s-1} \mathbb{E} \langle \bm{\mu}_z, v\rangle^{2s} + 2^{2s-1} (v^\top \Sigma_z v)^s s^s \geq c^s.\]
Use, by \cref{lemma:separating-polynomial-variance-upper-bound} and \cref{lemma:sos-square-power-upper-bound}, that $\sststile{2t}{v} (v^\top \Sigma_z v)^s \leq \left(2C/t\right)^s$:
\[\sststile{2s}{v} 2^{2s-1} \mathbb{E} \langle \bm{\mu}_z, v\rangle^{2s} + \left(8Cs/t\right)^s \geq c^s.\]
Then use that $t \geq 16Cs/c$ to obtain that $\sststile{2s}{v} 2^{2s-1} \mathbb{E} \langle \bm{\mu}_z, v\rangle^{2s} \geq c^s/2$. Finally, divide by $2^{2s-1}$ to obtain that $\sststile{2s}{v} \mathbb{E} \langle \bm{\mu}_z, v\rangle^{2s} \geq (c/4)^s$.
\end{proof}

Now we prove \cref{lemma:separating-polynomial-completeness}.

\begin{proof}[Proof of \cref{lemma:separating-polynomial-completeness}]
Let $\bm{w} \sim N(0, I_d)$.
Note that, for \(\bm{y} \sim N(\mu_{a}, \Sigma)\) and \(\bm{y}' \sim N(\mu_{b}, \Sigma)\) we have that \(\bm{y} - \bm{y}' \sim \mu_{a} - \mu_{b} + \Sigma_z^{1/2}\bm{w}\), so $q(\bm{y} -\bm{y}')=q(\mu_a - \mu_b + \Sigma_z^{1/2} \bm{w})$.
Similarly, for \(\bm{y}, \bm{y}' \sim N(\mu_{a}, \Sigma)\) we have that \(\bm{y} - \bm{y}' \sim \Sigma_z^{1/2} \bm{w}\), so $q(\bm{y} -\bm{y}')=q(\Sigma_z^{1/2} \bm{w})$.
Therefore, we want to show (1) that there exist distinct $a, b \in [k]$ such that $q(\mu_a - \mu_b + \Sigma_z^{1/2} \bm{w})$ is large and (2) that $q(\Sigma_z^{1/2} \bm{w})$ is small.

By \cref{lemma:separating-polynomial-mean-lower-bound}, we have $\tilde{\mathbb{E}} \mathbb{E} \langle \bm{\mu}_z, v\rangle^{2s} \geq \left(c/4\right)^s$.
By linearity, we also have $\mathbb{E} \tilde{\mathbb{E}} \langle \bm{\mu}_z, v\rangle^{2s} \geq \left(c/4\right)^s$, so $\mathbb{E} q(\bm{\mu}_z) \geq (c/4)^s$. 
Therefore there exists some $\mu_z$ in the support of $\bm{\mu}_z$ such that $q(\mu_z) \geq (c/4)^s$.
Therefore, there exist $a, b \in [k]$ such that $q(\mu_a - \mu_b) \geq (c/4)^s$. 
Furthermore, $a$ and $b$ are distinct, because otherwise $q(\mu_a - \mu_b) = 0$.

We attempt to lower bound $q(\mu_a - \mu_b + \Sigma_z^{1/2} \bm{w})$:
\begin{align*}
q(\mu_{a} - \mu_{b} + \Sigma_z^{1/2} \bm{w})
&= \langle \tilde{\mathbb{E}} v^{\otimes 2s}, (\mu_{a} - \mu_{b} + \Sigma_z^{1/2} \bm{w})^{\otimes 2s}\rangle\\
&= \tilde{\mathbb{E}} \langle v, \mu_{a} - \mu_{b} + \Sigma_z^{1/2} \bm{w} \rangle^{2s}\\
&\geq \frac{1}{2^{2s-1}} \tilde{\mathbb{E}} \langle v, \mu_{a} - \mu_{b} \rangle^{2s} - \tilde{\mathbb{E}} \langle v, \Sigma_z^{1/2} \bm{w} \rangle^{2s}\\
&= \frac{1}{2^{2s-1}} q(\mu_{a} - \mu_{b}) - q(\Sigma_z^{1/2} \bm{w})\\
&\geq \left(c/16\right)^s - q(\Sigma_z^{1/2} \bm{w}).
\end{align*}

We want to show that $q(\Sigma_z^{1/2} \bm{w})$ is small with high probability. We start by analyzing the mean and second moment of $q(\Sigma_z^{1/2} \bm{w})$. For $\ell \in \{1, 2\}$, we have
\begin{align*}
\mathbb{E} q(\Sigma_z^{1/2} \bm{w})^{\ell}
&= \mathbb{E} \left(\langle \tilde{\mathbb{E}} v^{\otimes 2s}, (\Sigma_z^{1/2} \bm{w})^{\otimes 2s} \rangle\right)^{\ell}
= \mathbb{E} \left( \tilde{\mathbb{E}} \langle v, \Sigma_z^{1/2} \bm{w}\rangle^{2s} \right)^{\ell}\\
&\stackrel{(1)}{\leq} \mathbb{E} \tilde{\mathbb{E}} \langle v, \Sigma_z^{1/2} \bm{w}\rangle^{2s\ell}
= \tilde{\mathbb{E}} \mathbb{E} \langle v, \Sigma_z^{1/2} \bm{w}\rangle^{2s\ell}\\
&= \tilde{\mathbb{E}} \mathbb{E} \langle \Sigma_z^{1/2} v, \bm{w}\rangle^{2s\ell}
\leq (s\ell)^{s\ell} \tilde{\mathbb{E}} (v^\top \Sigma_z v)^{s\ell}\\
&\stackrel{(2)}{\leq} \left(2CS\ell/t\right)^{s\ell},
\end{align*}
where in (1) for $\ell=2$ we used \cref{lemma:sos-pe-jensen} and in (2) we used that, by \cref{lemma:separating-polynomial-variance-upper-bound} and \cref{lemma:sos-square-power-upper-bound}, \(\tilde{\mathbb{E}} (v^\top \Sigma_z v)^{s\ell} \leq O\left(2C/t\right)^{s\ell}\).
Then \(\mathbb{E} q(\Sigma_z{1/2} \bm{w}) \leq \left(2Cs/t\right)^{s}\) and \(\mathbb{E} q(\Sigma_z^{1/2} \bm{w})^2 \leq \left(4Cs/t\right)^{2s}\).
Therefore, by Chebyshev's inequality, with probability $0.99999$,
\[q(\Sigma_z^{1/2} \bm{w}) \leq \left(2Cs/t\right)^{s} + \sqrt{100000} \left(4Cs/t\right)^{s} \leq 320 \left(4Cs/t\right)^{s}.\]
In this case, we also have
\[q(\mu_{a} - \mu_{b} + \Sigma_z^{1/2} \bm{w}) \geq \left(c/16\right)^s - 320 \left(4Cs/t\right)^{s} \geq \left(c/16\right)^s/2,\]
where in the last inequality we used that $t \geq 50000 C s/c$. This concludes the proof.
\end{proof}

\begin{lemma}[Existence of vector that satisfies moment contraints]
\label{lemma:separating-polynomial-soundness}
Let \(s, t \geq \lceil \log \pmin^{-1} \rceil\) integers.
If \(t \leq \max_{i, j} \|\Sigma_z^{-1/2}(\mu_{i} - \mu_{j})\|^2\), there exists some \(v \in \mathbb{R}^d\) that satisfies \(\mathbb{E} \langle \bm z, v\rangle^{2s} = 1\) and \(\mathbb{E} \langle \bm z, v\rangle^{2t} \leq 30^t\).
Furthermore, $\|\operatorname{cov}(\bm{z})^{1/2} v\|^2 \leq 8$.
\end{lemma}

\begin{proof}
Let \((a, b) = \arg\max_{(i, j)} \|\Sigma_z^{-1/2}(\mu_i - \mu_j)\|\) and let \(v = \Sigma_z^{-1} (\mu_{a} - \mu_{b})\).
The vector for which we will guarantee the stated properties is $v^* = \frac{v}{(\mathbb{E} \langle \bm z, v\rangle^{2s})^{1/2s}}$.

We begin by proving that $\max_{i,j} \langle \mu_i - \mu_j, v\rangle^2 = \langle \mu_a - \mu_b, v\rangle^2$, which will be used later.
We have
\begin{align*}
\max_{i,j} \langle \mu_i - \mu_j, v\rangle^2
&= \max_{i,j} \langle \mu_i - \mu_j, \Sigma_z^{-1} (\mu_{a} - \mu_{b})\rangle^2\\
&= \max_{i,j} \langle \Sigma_z^{-1/2} (\mu_i - \mu_j), \Sigma_z^{-1/2} (\mu_{a} - \mu_{b})\rangle^2\\
&\leq \|\Sigma_z^{-1/2} (\mu_i - \mu_j)\|^2 \cdot \|\Sigma_z^{-1/2} (\mu_a - \mu_b)\|^2\\
&\leq \|\Sigma_z^{-1/2} (\mu_a - \mu_b)\|^4\\
&= \langle \mu_a - \mu_b, v\rangle^2.
\end{align*}

We also have for the variance in direction $v$ that 
\[(v^\top \Sigma_z v)^t = ((\mu_{a} - \mu_{b})^\top \Sigma_z^{-1} (\mu_{a} - \mu_{b}))^t = \|\Sigma_z^{-1/2}(\mu_{a} - \mu_{b})\|^{2t}.\]

We now derive upper bounds for the $2t$ moments of $\langle \bm z, v\rangle$ in direction $v$ and lower bounds for the $2s$ moments in direction $v$. Recall that we assume $t \leq \max_{i, j} \|\Sigma_z^{-1/2}(\mu_{i} - \mu_{j})\|^2$. For the upper bound, using \cref{lemma:moment-upper-bound-weak} we have
\begin{align*}
\mathbb{E} \langle \bm z, v\rangle^{2t}
&\leq 2^{2t-1} \mathbb{E} \langle \bm{\mu}_z, v\rangle^{2t} + 2^{2t-1} (v^\top \Sigma_z v)^t t^t\\
&\leq 2^{2t-1} \max_{i, j} \langle \mu_i - \mu_j, v\rangle^{2t} + 2^{2t-1} \|\Sigma_z^{-1/2}(\mu_{a} - \mu_{b})\|^{2t} t^t\\
&= 2^{2t-1} \|\Sigma_z^{-1/2}(\mu_{a} - \mu_{b})\|^{4t} + 2^{2t-1} \|\Sigma_z^{-1/2}(\mu_{a} - \mu_{b})\|^{2t} t^t\\
&\leq 2^{2t-1} \|\Sigma_z^{-1/2}(\mu_{a} - \mu_{b})\|^{4t} + 2^{2t-1} \|\Sigma_z^{-1/2}(\mu_{a} - \mu_{b})\|^{4t}\\
&\leq 2^{2t} \|\Sigma_z^{-1/2}(\mu_{a} - \mu_{b})\|^{4t}.
\end{align*}

For the lower bound, using \cref{lemma:moment-lower-bound-weak} we have
\begin{align*}
\mathbb{E} \langle \bm z, v\rangle^{2s}
&\geq \mathbb{E} \langle \bm{\mu}_z, v\rangle^{2s} + (v^\top \Sigma_z v)^s \frac{s^s}{2^s}\\
&\geq \pmin^{2} \max_{i, j} \langle \mu_i - \mu_j, v\rangle^{2s} + \|\Sigma_z^{-1/2}(\mu_{a} - \mu_{b})\|^{2s} \frac{s^s}{2^s}\\
&= \pmin^{2} \|\Sigma_z^{-1/2}(\mu_{a} - \mu_{b})\|^{4s} + \|\Sigma_z^{-1/2}(\mu_{a} - \mu_{b})\|^{2s} \frac{s^s}{2^s}\\
&\geq \pmin^{2} \|\Sigma_z^{-1/2}(\mu_{a} - \mu_{b})\|^{4s}.
\end{align*}

Recall that $v^* = \frac{v}{(\mathbb{E} \langle \bm z, v\rangle^{2s})^{1/2s}}$. Clearly, \(\mathbb{E} \langle \bm z, v^* \rangle^{2s} = 1\). Furthermore,
\begin{align*}
\mathbb{E} \langle \bm z, v^*\rangle^{2t}
&= \frac{\mathbb{E} \langle \bm z, v\rangle^{2t}}{(\mathbb{E} \langle \bm z, v\rangle^{2s})^{t/s}}
\leq \frac{2^{2t} \|\Sigma_z^{-1/2}(\mu_{a} - \mu_{b})\|^{4t}}{\pmin^{2t/s} \|\Sigma_z^{-1/2}(\mu_{a} - \mu_{b})\|^{4t}}
= \left(\frac{4}{\pmin^{2/s}}\right)^t
\leq \left( 4e^2 \right)^t \leq 30^t
\end{align*}
where in the last inequality we used that $\pmin^{1/s} \geq e^{-1}$. Therefore $v^*$ satisfies the desired moment constraints.

Finally, we prove that $\|\operatorname{cov}(\bm z)^{1/2} v^*\|^2 \leq 8$.
Note that $\operatorname{cov}(\bm z) = \operatorname{cov}(\bm{\mu}_z) + \Sigma_z$.
We have then
\begin{align*}
\|\operatorname{cov}(\bm z)^{1/2} v^*\|^2
&= (v^*)^\top \operatorname{cov}(\bm z) v^*\\
&= (v^*)^\top \operatorname{cov}(\bm{\mu}_z) v^* + (v^*)^\top \Sigma_z v^*\\
&\stackrel{(1)}{\leq} \max_{i,j} \langle \mu_i - \mu_j, v^*\rangle^2 + \|\Sigma_z^{1/2} v^*\|^2\\
&= \frac{\|\Sigma_z^{-1/2}(\mu_a - \mu_b)\|^4}{(\mathbb{E} \langle \bm z, v\rangle^{2s})^{1/s}} + \frac{\|\Sigma_z^{1/2} \Sigma_z^{-1} (\mu_a - \mu_b)\|^2}{(\mathbb{E} \langle \bm z, v\rangle^{2s})^{1/s}}\\
&\stackrel{(2)}{\leq} \frac{\|\Sigma_z^{-1/2}(\mu_a - \mu_b)\|^4}{\pmin^{2/s} \|\Sigma_z^{-1/2}(\mu_a - \mu_b)\|^4} + \frac{\|\Sigma_z^{1/2} \Sigma_z^{-1} (\mu_a - \mu_b)\|^2}{\pmin^{2/s} \|\Sigma_z^{-1/2}(\mu_a - \mu_b)\|^4}\\
&\stackrel{(3)}{\leq} e^2 + e^2 \frac{1}{\|\Sigma_z^{-1/2}(\mu_a - \mu_b)\|^2}\\
&\stackrel{(4)}{=} e^2 + o(1) \leq 8,
\end{align*}
where in (1) we used that $(v^*)^\top \operatorname{cov}(\bm{\mu}_z) v^* = (v^*)^\top \mathbb{E} \bm{\mu}_z\bm{\mu}_z^\top v^* \leq \max_{\mu_z} (v^*)^\top \mu_z \mu_z^\top v^*$ for $\mu_z$ in the support of $\bm{\mu}_z$,
in (2) we used the lower bound that we derived above on $\mathbb{E} \langle \bm z, v\rangle^{2s}$,
in (3) we used that $\pmin^{1/s} \geq e^{-1}$,
and in (4) we used mean separation.

\end{proof}

\subsection{Finite sample bounds}
\label{subsec:separating-polynomial-finite-sample}

Recall that, to apply \cref{lemma:separating-polynomial-completeness}, we need to find a pseudo-expectation that satisfies the moment lower bound $\mathbb{E} \langle \bm z, v\rangle^{2s} \geq c^s$ and the moment upper bound $\mathbb{E} \langle \bm z, v\rangle^{2t} \leq C^t$ for $s \ll t$.
\cref{lemma:separating-polynomial-finite-sample-main} shows that it suffices to find a pseudo-expectation that satisfies $\|\widehat{\operatorname{cov}}(\bm z)^{1/2} v\|^2 \lesssim 8$, $\hat{\mathbb{E}}\langle \bm{z}, v\rangle^{2s} \succsim c^s$, and $\hat{\mathbb{E}}\langle \bm{z}, v\rangle^{2t} \lesssim C^t$.
Without the bound on the norm of $\widehat{\operatorname{cov}}(\bm z)^{1/2} v$ the errors may be arbitrarily large.

The result is supported by \cref{lemma:separating-polynomial-finite-sample-norm}, which shows that quadratics in the empirical covariance matrix are close to quadratics in the population covariance matrix of the components, and by \cref{lemma:separating-polynomial-finite-sample-closeness}, which shows that the empirical moments are close to the population moments.
The proofs of these two lemmas are deferred to the appendix.

\begin{lemma}[Moment constraints from empirical moment constraints]
\label{lemma:separating-polynomial-finite-sample-main}
Let $\eta < 0.001$. Let $c, C \geq 0$ and let $t \geq 1$ integer. For \(n \geq (\pmin^{-1} d)^{O(t)} \eta^{-2} \epsilon^{-1}\), with probability $1-\epsilon$,
\begin{align*}
&\left\{\|\widehat{\operatorname{cov}}(\bm z)^{1/2} v\|^2 \leq (1+\eta) \cdot 8, \hat{\mathbb{E}}\langle \bm{z}, v\rangle^{2s} \geq c^s + \eta, \hat{\mathbb{E}}\langle \bm{z}, v\rangle^{2t} \leq C^t - \eta \right\}\\
&\qquad \sststile{2t}{v} \left\{ \mathbb{E} \langle \bm{z}, v\rangle^{2s} \geq c^s, \mathbb{E} \langle \bm{z}, v\rangle^{2t} \leq C^t \right\}.
\end{align*}
Furthermore, with probability $1-\epsilon$, the axiom is satisfied with $s$, $t$, and $v$ as in \cref{lemma:separating-polynomial-soundness} and with $c=(1-\eta)^{1/s}$ and $C=(30^t+\eta)^{1/t}$.
\end{lemma}

\begin{proof}
By \cref{lemma:separating-polynomial-finite-sample-norm}, $\sststile{2}{v} \|\operatorname{cov}(\bm z)^{1/2} v\|^2 \leq (1+\eta)^2 \cdot 8 \leq 9$.
Then, by \cref{lemma:separating-polynomial-finite-sample-closeness}, $\sststile{2s}{v} \mathbb{E} \langle \bm{z}, v\rangle^{2s} \geq c^s$ and $\sststile{2t}{v} \mathbb{E} \langle \bm{z}, v\rangle^{2t} \leq C^t$.
With $n$ as given, this holds with probability at least $1-\epsilon$.

For the second claim of the lemma, we have by \cref{lemma:separating-polynomial-soundness} that there exists some $v$ with \(\mathbb{E} \langle \bm z, v\rangle^{2s} = 1\), \(\mathbb{E} \langle \bm z, v\rangle^{2t} \leq 30^t\), and $\|\operatorname{cov}(\bm{z})^{1/2} v\|^2 \leq 8$.
By \cref{lemma:separating-polynomial-finite-sample-norm}, we also have that $\|\widehat{\operatorname{cov}}(\bm z)^{1/2} v\|^2 \leq (1+\eta)^2 \cdot 8 \leq 9$, and then by \cref{lemma:separating-polynomial-finite-sample-closeness}, we also have that $\hat{\mathbb{E}}\langle \bm{z}, v\rangle^{2s} \geq 1 - \eta$ and $\hat{\mathbb{E}}\langle \bm{z}, v\rangle^{2t} \leq 30^t + \eta$.
Again, with $n$ as given, this holds with probability at least $1-\epsilon$.
\end{proof}

\begin{lemma}
\label{lemma:separating-polynomial-finite-sample-norm}
Let $C \geq 0$. For \(n \geq kd^2 \log^2(d/\epsilon) O(\eta^{-2})\), with probability $1-\epsilon$,
\[\{\|\widehat{\operatorname{cov}}(\bm z)^{1/2} v\|^2 \leq C\} \sststile{2}{v} \{\|\operatorname{cov}(\bm z)^{1/2} v\|^2 \leq (1+\eta) C\},\]
\[\{\|\operatorname{cov}(\bm z)^{1/2} v\|^2 \leq C\} \sststile{2}{v} \{\|\widehat{\operatorname{cov}}(\bm z)^{1/2} v\|^2 \leq (1+\eta) C\}.\]
\end{lemma}
\begin{proof}
See \cref{proofs-separatingpolynomial}.
\end{proof}

\begin{lemma}
\label{lemma:separating-polynomial-finite-sample-closeness}
Let $C \geq 0$ and let $t \geq 1$ integer. For \(n \geq (C \pmin^{-1} d)^{O(t)} \eta^{-2} \epsilon^{-1}\), with probability $1-\epsilon$,
\[\{\|\operatorname{cov}(\bm{z})^{1/2} v\|^2 \leq C\} \sststile{O(t)}{v} \{\hat{\mathbb{E}}\langle \bm z, v\rangle^{2t} \leq \mathbb{E}\langle \bm z, v\rangle^{2t} + \eta\},\]
\[\{\|\operatorname{cov}(\bm{z})^{1/2} v\|^2 \leq C\} \sststile{O(t)}{v} \{\hat{\mathbb{E}}\langle \bm z, v\rangle^{2t} \geq \mathbb{E}\langle \bm z, v\rangle^{2t} - \eta\}.\]
\end{lemma}
\begin{proof}
See \cref{proofs-separatingpolynomial}.
\end{proof}

\subsection{Proof of Theorem~\ref{thm:separating-polynomial-main}}

\begin{proof}
Let $s = \lceil \log \pmin^{-1} \rceil$ and $t = 10000000 s$. The algorithm is:
\begin{enumerate}
  \item Compute a pseudo-expectation $\tilde{\mathbb{E}}$ of degree $2t$ over $v \in \mathbb{R}^d$ such that $\|\widehat{\operatorname{cov}}(\bm z)^{1/2} v\|^2 \leq 1.01 \cdot 8$, $\hat{\mathbb{E}}\langle \bm{y}, v\rangle^{2s} \geq 1 - 0.005$, and $\hat{\mathbb{E}}\langle \bm{y}, v\rangle^{2t} \leq 30^t + 0.005\}$.
  \item Construct a separating polynomial $q$ based on $\tilde{\mathbb{E}}$ as in \cref{lemma:separating-polynomial-completeness}.
  \item Return $q$.
\end{enumerate}

We now analyze the algorithm.
First, we argue that there exists a pseudo-expectation $\tilde{\mathbb{E}}$ that satisfies the given constraints.
Note that $t \leq \max_{i,j} \|\Sigma_z^{-1/2}(\mu_i - \mu_j)\|^2$ if $C_{sep}$ is a large enough constant.
Therefore, the conditions of \cref{lemma:separating-polynomial-soundness} for $s$ and $t$ are satisfied.
Then, by \cref{lemma:separating-polynomial-finite-sample-main}, for $n \geq n_0$, there exists a vector $v \in \mathbb{R}^d$ that satisfies the given constraints. 
Then, there also exists a pseudo-expectation that satisfies the constraints.

Second, we argue that $q$ has the desired properties.
By \cref{lemma:separating-polynomial-finite-sample-main}, $\tilde{\mathbb{E}}$ also sastisfies with high probability that $\hat{\mathbb{E}}\langle \bm{y}, v\rangle^{2s} \geq 1 - 0.01 \geq 0.99^s$ and $\hat{\mathbb{E}}\langle \bm{y}, v\rangle^{2t} \leq 30^t + 0.01 \leq 31^t$.
Then the conditions of \cref{lemma:separating-polynomial-completeness} are satisfied with $c=0.99$ and $C=31$.
Then, we are guaranteed to return a separating polynomial $q$ with the following properties:
\begin{itemize}
  \item For independent random vectors $\bm{y}$ and $\bm{y}'$ sampled from different components, we have with probability at least $0.99999$ that
  \[q(\bm{y} - \bm{y}') \geq \frac{1}{2}\left(\frac{c}{16}\right)^s \geq \frac{1}{2} \left(\frac{0.99}{16}\right)^s \geq \frac{1}{20^s}.\] 
  \item For independent random vectors $\bm{y}$ and $\bm{y}'$ sampled from the same component, we have with probability at least $0.99999$ that
  \[q(\bm{y} - \bm{y}') \leq 320 \left(\frac{4Cs}{t}\right)^s \leq 320 \left(\frac{4\cdot 31 \cdot s}{t}\right)^s \leq \frac{1}{200^s}.\]
\end{itemize}

The time complexity of the algorithm is dominated by the time to compute the pseudo-expectation.
The pseudo-expectation is of degree $O(\log \pmin^{-1})$ over $d$ variables, and each constraint requires summing over the $n$ samples.
Therefore, the time to compute the pseudo-expectation is $n \cdot d^{O(\log \pmin^{-1})}$.
\end{proof}

\subsection{Proof of Theorem~\ref{thm:separating-bipartition-main}}

We start by stating and proving \cref{lemma:bipartition-from-distance}, which shows how to obtain a bipartition of the samples given a suitable distance function. We then prove \cref{thm:separating-bipartition-main}.

\begin{lemma}[Bipartition from distance function]
\label{lemma:bipartition-from-distance}
Assume access to a distance function $d_q: \mathbb{R}^d \times \mathbb{R}^d \to \mathbb{R}$ such that:
\begin{itemize}
    \item There exist distinct \(a,b\in [k]\) such that the independent random vectors \(\bm y\sim N(\mu_a,\Sigma)\) and \(\bm y'\sim N(\mu_b,\Sigma)\) satisfy 
      \[
        \Pr\Set{d_q(\bm y, \bm y') \ge \frac{1}{\sqrt{20}}}
        \ge 0.99999.
        \,
      \]
    \item For all \(a\in [k]\), the independent random vectors \(\bm y,\bm y'\sim N(\mu_a,\Sigma)\) satisfy
      \[
        \Pr\Set{d_q(\bm y, \bm y') \le \frac{1}{\sqrt{200}}}
        \ge 0.99999.
        \,
      \]
    \end{itemize}
Given a sample of size \(n = \Omega(\pmin^{-1})\) from the mixture, there exists a polynomial-time algorithm that returns with probability $0.99$ a partition of $[n]$ into two sets $C_1$ and $C_2$ such that, if true clustering of the samples is $S_1, ..., S_k$, then
\[\max_i \frac{|C_1 \cap S_i|}{|S_i|} \geq 0.99 \quad \text{and} \quad \max_i \frac{|C_2 \cap S_i|}{|S_i|} \geq 0.99.\]
\end{lemma}

\begin{proof}
The algorithm is: 
\begin{enumerate}
    \item Choose $i \in [n]$ uniformly at random.
    \item Let $S = \{j \in [n]: d_q(y_i, y_j) \leq \frac{1}{\sqrt{200}}\}$.
    \item Return $S$ and $[n] \setminus S$.
\end{enumerate}

We now analyze the algorithm.

First, we prove that, with probability at least $0.995$, a $0.99$-fraction of the samples from the same component as $i$ are included in $S$.
With high probability, a $0.99998$-fraction of the pairs of samples $(y, y')$ with $y$ and $y'$ from the same component as $i$ satisfy $d_q(y, y') \leq \frac{1}{\sqrt{200}}$.
Then, the fraction of samples from this component that are farther than $\frac{1}{\sqrt{200}}$ from more than a $0.01$-fraction of the other samples in the component is at most $\frac{1-0.99998}{0.01} = 0.002$.
Then, overall, with probability at least $0.995$, $i$ is closer than $\frac{1}{\sqrt{200}}$ to at least a $0.99$-fraction of the other samples in the component.
In this case, $S$ includes a $0.99$-fraction of the samples from the same component as $i$.

Second, we prove that, with probability at least $0.995$, at least a $0.99$-fraction of the samples from one of the components are not included in $S$.
Let $a, b \in [k]$ be the two components for which the large-distance guarantee holds.
With high probability, a $0.99998$-fraction of the pairs of samples $(y, y')$ with $y$ from $a$ and $y'$ from $b$ satisfy $d_q(y, y') \geq \frac{1}{\sqrt{20}}$.
We have $d_q(y,y') \leq d_q(y, y_i) + d_q(y', y_i)$, so if $d_q(y,y') \geq \frac{1}{\sqrt{20}}$, then at least one of $d_q(y, y_i)$ or $d_q(y', y_i)$ is at least $\frac{1}{2\sqrt{20}} > \frac{1}{\sqrt{200}}$.
Then, for such pairs, it is impossible for both $y$ and $y'$ to be in $S$.
Suppose that a $p_a$-fraction of the samples from $a$ are in $S$ and that a $p_b$-fraction of the samples from $b$ are in $S$.
We need then that $1-p_a p_b \geq 0.99998$, so $\min(p_a, p_b) \leq \sqrt{1-0.9998} \leq 0.01$.
Therefore, $[n]\setminus S$ contains at least a $0.99$-fraction of the samples from one of $a$ or $b$.

Therefore, with probability at least $0.99$, the conclusion of the lemma holds.
\end{proof}

\begin{proof}[Proof of \cref{thm:separating-bipartition-main}]
The algorithm is:
\begin{enumerate}
  \item Run the algorithm from \cref{thm:separating-polynomial-main} to obtain a polynomial $q$.
  \item Run the algorithm from \cref{lemma:bipartition-from-distance} with the distance function $d_q(x,y) = q(x - y)^{1/2s}$.
  \item Return the resulting bipartition.
\end{enumerate}
We now analyze the algorithm.
Recall that $q(u) = \tilde{\E} \iprod{u,v}^{2s}$, so $q(u)^{1/2s} = (\tilde{\E} \iprod{u,v}^{2s})^{1/2s}$.
Then, we have by \cref{lemma:sos-pe-triangle} that $q^{1/2s}$ satisfies the triangle inequality.
It follows that $d_q(x,y) = q(x - y)^{1/2s}$ is a distance function.
Then, by the guarantees of \cref{thm:separating-polynomial-main}, $d_q$ satisfies the requirements of \cref{lemma:bipartition-from-distance}, so the stated guarantees follow.

The time complexity is dominated by the time complexity of the algorithm from \cref{thm:separating-polynomial-main}.

\end{proof}

\section{Parallel pancakes}
\label{sec:pp}

The model studied in this section is a well-separated mixture of Gaussians with colinear means that is in isotropic position.
As shown in \cref{sec:isotropic-position-properties}, the isotropic position property makes this model similar to the parallel pancakes construction, in the sense that the only direction in which the components of the mixture have variance different from $1$ is the direction of the means.
In \cref{sec:colinear} we study the same model without the isotropic position assumption.

\paragraph{Setting.} We consider a mixture of \(k\) Gaussian distributions \(N(\mu_i, \Sigma)\) with mixing weights \(p_i\) for \(i=1,...,k\),
where \(\mu_i \in \mathbb{R}^d\), \(\Sigma \in \mathbb{R}^{d \times d}\) is positive definite, and \(p_i \geq 0\) and \(\sum_{i=1}^k p_i = 1\). 
Let \(\pmin = \min_i p_i\).

The distribution is in isotropic position: for $\bm{y}$ distributed according to the mixture, we have $\mathbb{E} \bm{y} = 0$ and $\operatorname{cov}(\bm{y}) = I_d$.

The distribution also satisfies mean separation and mean colinearity:
\begin{itemize}
\tightlist
\item
  Mean separation:
  for some $C_{sep} > 0$ and for all \(i \neq j\),
  \[\left\|\Sigma^{-1/2} (\mu_i - \mu_j)\right\|^2 \geq C_{sep} \log \pmin^{-1}.\]
\item
  Mean colinearity:
  for some unit vector \(u \in \mathbb{R}^d\) and for all \(i\),
  \[\mu_i = \langle \mu_i, u\rangle u.\]
\end{itemize}

Also define $\sigma^2 = u^\top \Sigma u$, which is the variance of the components in the direction of the means.

\begin{theorem}[Parallel pancakes algorithm]
\label{thm:parallel-pancakes-main}
Consider the Gaussian mixture model defined above, with $C_{sep}$ larger than some universal constant. Let 
\[n_0 = \left(\frac{1}{\sigma^2}\right)^{O(1)} \cdot (\pmin^{-1} d)^{O(\log \pmin^{-1})}.\]
Given a sample of size $n \geq n_0$ from the mixture, there exists an algorithm that runs in time $n^{O(\log \pmin^{-1})}$ and returns with high probability a partition of \([n]\) into $k$ sets $C_1, ..., C_k$ such that, if the true clustering of the samples is $S_1, ..., S_k$, then there exists a permutation $\pi$ of $[k]$ such that 
\[1 - \frac{1}{n} \sum_{i=1}^k |C_i \cap S_{\pi(i)}| \leq \left(\frac{\pmin}{k}\right)^{O(1)}.\]
\end{theorem}

We introduce some further notation for this section.
Let $\bm{y}$ be distributed according to the mixture.
We specify the model as \(\bm{y} = \bm{\mu} + \bm{w}\), where \(\bm{\mu}\) takes value \(\mu_i\) with probability \(p_i\) and \(\bm{w} \sim N(0, \Sigma)\), with \(\bm{\mu}\) and \(\bm{w}\) independent of each other.

\subsection{Isotropic position properties}
\label{sec:isotropic-position-properties}

In this section we prove some consequences of the fact that \(\bm{y}\) is in isotropic position. 
\cref{lemma:isotropic-position-covariance} shows that \(\Sigma = I_d - (1-\sigma^2) uu^\top\) with \(0 < \sigma^2 \leq 1\).
This means that \(\Sigma\) can have at most one eigenvalue less than \(1\) and that the eigenvectors corresponding to this eigenvalue are parallel to the direction of the means $u$.
Then \cref{lemma:isotropic-position-separation} uses this form of \(\Sigma\) to quantify the separation of the means along direction \(u\) in terms of \(\sigma^2\).

\begin{lemma}[Isotropic position component covariance matrix]
\label{lemma:isotropic-position-covariance}
We have that (1) $\sigma^2 = 1 - \sum_{i=1}^k p_i \langle \mu_i, u\rangle^2$, (2) $\Sigma = I_d - (1-\sigma^2) uu^\top$, and (3) $0 < \sigma^2 \leq 1$.
\end{lemma}

\begin{proof}
We have that $\mathbb{E} \bm y = \mathbb{E} \bm \mu + \mathbb{E} \bm w = \mathbb{E} \bm \mu$. Because the distribution is in isotropoic position, we also have that \(\mathbb{E} \bm y = 0\), so the equation above implies that \(\mathbb{E} \bm \mu = 0\). Then \(\operatorname{cov}(\bm \mu) = \sum_{i=1}^k p_i \mu_i \mu_i^\top\).

Furthermore, since $\bm \mu$ and $\bm w$ are independent, we have that $\operatorname{cov}(\bm y) = \operatorname{cov}(\bm \mu) + \operatorname{cov}(\bm w) = \sum_{i=1}^k p_i \mu_i \mu_i^\top + \Sigma$.
Because the distribution is in isotropic position, we also have that \(\operatorname{cov}(\bm y) = I_d\), so the equation above implies $\Sigma = I_d - \sum_{i=1}^k p_i \mu_i \mu_i^\top$. Plugging in \(\mu_i = \langle \mu_i, u\rangle u\), we have $\Sigma = I_d - \left(\sum_{i=1}^k p_i \langle \mu_i, u\rangle^2\right) uu^\top$.

Then, it follows that $\sigma^2 = u^\top \Sigma u = 1 - \sum_{i=1}^k p_i \langle \mu_i, u\rangle^2$. This proves (1). The fact that $1-\sigma^2 = \sum_{i=1}^k p_i \langle \mu_i, u\rangle^2$ also proves (2). For (3), \(\sigma^2 > 0\) follows by the definition using that $\Sigma$ is positive definite and \(\sigma^2 \leq 1\) follows by (1).
\end{proof}

\begin{lemma}[Isotropic position mean separation]
\label{lemma:isotropic-position-separation}
For all \(i \neq j\),
\[\langle \mu_i - \mu_j, u\rangle^2 \geq C_{sep} \sigma^2 \log \pmin^{-1}.\]
\end{lemma}

\begin{proof}
By \cref{lemma:isotropic-position-covariance}, \(\Sigma = I_d - (1-\sigma^2) uu^\top\). This implies that $\Sigma^{-1/2} = I_d + \left(1/\sqrt{\sigma^2} - 1\right) uu^\top$. Then, using that \(\mu_i = \langle \mu_i, u\rangle u\), we have that
\begin{align*}
\Sigma^{-1/2} (\mu_i - \mu_j)
&= \left(I_d + \left(\frac{1}{\sqrt{\sigma^2}} - 1\right) uu^\top\right) \left(\langle \mu_i - \mu_j, u\rangle u\right)
= \frac{1}{\sqrt{\sigma^2}} \langle \mu_i - \mu_j, u\rangle u.
\end{align*}
Therefore, the separation condition $\left\|\Sigma^{-1/2} (\mu_i - \mu_j)\right\|^2 \geq C_{sep} \log \pmin^{-1}$ is equivalent to $\frac{1}{\sigma^2} \langle \mu_i - \mu_j, u\rangle^2 \geq C_{sep} \log \pmin^{-1}$. The conclusion follows by multiplying both sides by $\sigma^2$.
\end{proof}

\subsection{Exact moment direction recovery}
\label{sec:sosident}

In this section we discuss how to recover a direction close to the direction of the means $u$, assuming oracle access to moments \(\mathbb{E} \bm{y}^{\otimes t}\) for any positive integer \(t\).
Access to these moments allows us to calculate exactly directional moments of the form $\mathbb{E} \langle \bm{y}, v\rangle^t$, which simplifies the analysis.
Finite sample considerations are discussed starting with \cref{sec:finitesample_pp}.

\cref{thm:infinite-algorithm-high-correlation} shows that there exists an algorithm that computes a unit vector $\hat{u}$ with correlation $1-O(\min(\sigma^2, \frac{1}{k}))$ with the direction of the means $u$.
We remark that it is necessary for $\hat{u}$ to have a correlation of at least $1-O(\sigma^2)$ with $u$ in order for the components of the mixture to be separated along direction $\hat{u}$.

\begin{theorem}[Direction recovery with exact moments]
\label{thm:infinite-algorithm-high-correlation}
Assume oracle access to \(\mathbb{E} \bm{y}^{\otimes t}\) for any positive integer \(t\).
Then there exists an algorithm with time complexity \(\left(\log \frac{1}{\sigma^2}\right) \cdot d^{O(\log \pmin^{-1})}\) that outputs a unit vector \(\hat{u} \in \mathbb{R}^d\) such that \(\langle u, \hat{u}\rangle^2 \geq 1 - 320\min(\sigma^2, \frac{1}{k})\).
\end{theorem}

The two main ingredients for \cref{thm:infinite-algorithm-high-correlation} are \cref{thm:find-direction-main}, which gives an algorithm to compute pseudo-expectations over unit vectors correlated with $u$, and \cref{thm:sample-direction-main}, which gives an algorithm to sample from such pseudo-expectations.
We note that \cref{thm:find-direction-main} can be interpreted as a collection of sum-of-squares identifiability proofs for the direction of the means $u$.

We state these two supporting theorems and then prove \cref{thm:infinite-algorithm-high-correlation}.
After that, we work toward proving the supporting theorems.

\begin{theorem}[Direction sum-of-squares identifiability]
\label{thm:find-direction-main}
Assume oracle access to \(\mathbb{E} \bm{y}^{\otimes t}\) for any positive integer \(t\).
Then there exists an algorithm with time complexity \(\left(\log \frac{1}{\sigma^2}\right) \cdot d^{O(\log \pmin^{-1})}\) that computes two peseudo-expectations \(\tilde{\mathbb{E}}_U\) and \(\tilde{\mathbb{E}}_L\) of degree \(O(\log \pmin^{-1})\) over a variable \(v \in \mathbb{R}^d\) such that the following holds.
Let \(s = \lceil \log \pmin^{-1} \rceil\), let \(t = 5000s\), and let $\tau = \frac{800e}{C_{sep}k^2}$. 
Then \(\tilde{\mathbb{E}}_U \|v\|^2 = 1\), \(\tilde{\mathbb{E}}_L \|v\|^2 = 1\), and:

\begin{itemize}
\tightlist
\item If $\sigma^2 \geq \tau$, then \(\tilde{\mathbb{E}}_U \langle u, v\rangle^{2s} \geq (1-\tau)^{s}\).
\item If $\sigma^2 < \tau$ and $\mathbb{E} \langle \bm{\mu}, u\rangle^{2s} \geq (4es)^s$, then
  \(\tilde{\mathbb{E}}_U \langle u, v\rangle^{2s} \geq (1-\sigma^2)^{s}\).
\item If $\sigma^2 < 0.001$ and $\mathbb{E} \langle \bm{\mu}, u\rangle^{2s} \leq (100s)^s$, then
  \(\tilde{\mathbb{E}}_L \langle u, v\rangle^{2t} \geq (1-20\sigma^2)^{t}\).
\end{itemize}
\end{theorem}

\begin{theorem}[Direction sum-of-squares sampling]
\label{thm:sample-direction-main}
Let \(t \in \mathbb{N}\) and \(\epsilon \in \mathbb{R}\) such that \(t \geq 1\) and \(0 \leq \epsilon \leq 1/(3t^2)\).
Let \(u \in \mathbb{R}^d\) be a unit vector.
Given a pseudo-expectation $\tilde{\mathbb{E}}$ of degree \(2t\) over a variable \(v \in \mathbb{R}^d\) that satisfies \(\tilde{\mathbb{E}} \|v\|^2 = 1\) and \(\tilde{\mathbb{E}} \langle u, v\rangle^{2t} \geq (1-\epsilon)^t\),
there exists an algorithm with time complexity \(d^{O(t)}\) that returns a unit vector \(\hat{u} \in \mathbb{R}^d\) such that \(\langle u, \hat{u}\rangle^2 \geq 1 - 16 \epsilon\).
\end{theorem}

\begin{proof}[Proof of \cref{thm:infinite-algorithm-high-correlation}]
Let $\tau = \frac{800e}{C_{sep}k^2}$. The algorithm is:
\begin{enumerate}
\tightlist
\item
Run the algorithm from \cref{thm:find-direction-main} to obtain pseudo-expectations \(\tilde{\mathbb{E}}_U\) and \(\tilde{\mathbb{E}}_L\).
\item
Run the algorithm from \cref{thm:sample-direction-main} for pseudo-expectations \(\tilde{\mathbb{E}}_U\) and \(\tilde{\mathbb{E}}_L\) to obtain unit vectors $\hat{u}_U \in \mathbb{R}^d$ and $\hat{u}_L \in \mathbb{R}^d$, respectively.
\item
If $\sigma^2 \geq \tau$, return $\hat{u}_U$. Else, for $s=\lceil \log \pmin^{-1} \rceil$, if $\mathbb{E} \langle \bm{y}, \hat{u}_U\rangle^{2s} \geq (50s)^s$, return $\hat{u}_U$. Else, return $\hat{u}_L$.
\end{enumerate}

We now analyze the algorithm.  We consider the three possible cases in step (3) of the algorithm:

\begin{itemize}
\item Suppose $\sigma^2 \geq \tau$. Then \cref{thm:find-direction-main} guarantees that \(\tilde{\mathbb{E}}_U \langle u, v\rangle^{2s} \geq (1-\tau)^{s}\), so by \cref{thm:sample-direction-main} we have $\langle u, \hat{u}_U\rangle^2 \geq 1-16\tau \geq 1-16\min(\sigma^2, \tau)$.
\item Suppose $\sigma^2 < \tau$ and $(\mathbb{E} \langle \bm{y}, \hat{u}_U\rangle^{2s})^{1/s} \geq 50s$. We have by \cref{lemma:moment-upper-bound-strong} that $(\mathbb{E} \langle \bm{y}, \hat{u}_U\rangle^{2s})^{1/s} \leq (\mathbb{E} \langle \bm{\mu}, u\rangle^{2s})^{1/s} + es,$
so it must be the case that $(\mathbb{E} \langle \bm{\mu}, u\rangle^{2s})^{1/s} \geq (50-e)s \geq 4es$. Then \cref{thm:find-direction-main} guarantees that \(\tilde{\mathbb{E}}_U \langle u, v\rangle^{2s} \geq (1-\sigma^2)^{s}\), so by \cref{thm:sample-direction-main} we have $\langle u, \hat{u}_U\rangle^2 \geq 1 - 16\sigma^2 \geq 1-16\min(\sigma^2, \tau)$. 
\item Suppose $\sigma^2 < \tau$ and $\mathbb{E} \langle \bm{y}, \hat{u}_U\rangle^{2s} < (50s)^s$. We have by \cref{lemma:moment-lower-bound-strong} that $(\mathbb{E} \langle \bm{y}, \hat{u}_U\rangle^{2s})^{1/s} \geq (\mathbb{E} \langle \bm{\mu}, u\rangle^{2s})^{1/s},$ so it must be the case that $(\mathbb{E} \langle \bm{\mu}, u\rangle^{2s})^{1/s} < 50s \leq 100s$. Then \cref{thm:find-direction-main} guarantees that \(\tilde{\mathbb{E}}_L \langle u, v\rangle^{2t} \geq (1-20\sigma^2)^{t}\), so by \cref{thm:sample-direction-main} we have $\langle u, \hat{u}_L\rangle^2 \geq 1 - 16\cdot 20\sigma^2 \geq 1-320\min(\sigma^2, \tau)$.
\end{itemize}

Let $\hat{u}$ be the unit vector returned by step (3) of the algorithm. Then we are guaranteed that in all cases $\langle u, \hat{u}\rangle^2 \geq 1 - 320\min(\sigma^2, \tau) \geq 1 - 320\min(\sigma^2, \frac{1}{k})$, where we used the loose upper bound $\tau \leq \frac{1}{k}$.

The time complexity of the algorithm is dominated by the time to run the algorithm from \cref{thm:find-direction-main}.
\end{proof}

\subsubsection{Sum-of-squares identifiability (proof of Theorem~\ref{thm:find-direction-main})}

We prove a number of supporting lemmas and then prove \cref{thm:find-direction-main}. 
The most important components are \cref{lemma:find-direction-moment-maximization} and \cref{lemma:find-direction-moment-minimization},
which give sum-of-squares proofs that, for suitably chosen $s, t= O(\log \pmin^{-1})$, either the maximizer of $\mathbb{E}\langle \bm{y}, v\rangle^{2s}$ or the minimizer of $\mathbb{E}\langle \bm{y}, v\rangle^{2t}$ over unit vectors $v$ must be close to $u$.

We start with \cref{lemma:moment-equality-strong}, \cref{lemma:moment-upper-bound-strong} and \cref{lemma:moment-lower-bound-strong}, which give sum-of-squares bounds on the moments of the mixture.
Informally, for \(t = \Omega(\log \pmin^{-1})\), these bounds correspond to the following decomposition of the directional \(2t\) moments:
\begin{equation}
(\mathbb{E} \langle \bm{y}, v\rangle^{2t})^{1/t} = \langle u, v\rangle^2 \left(\mathbb{E} \langle \bm{\mu}, u\rangle^{2t}\right)^{1/t} + \Theta(1) \cdot t (v^\top \Sigma v).
\end{equation}

\begin{lemma}[Moment equality]
\label{lemma:moment-equality-strong}
For $t \geq 1$ integer,
\[\sststile{2t}{v} \mathbb{E} \langle \bm{y}, v \rangle^{2t} = \sum_{s=0}^t \binom{2t}{2s} \langle u, v\rangle^{2s} \mathbb{E} \langle \bm{\mu}, u\rangle^{2s} (v^\top\Sigma v)^{t-s} (2t-2s-1)!!.\]
\end{lemma}

\begin{proof}
\begin{align*}
\sststile{2t}{v} \mathbb{E} \langle \bm y, v\rangle^{2t}
&= \mathbb{E} (\langle \bm{\mu}, v\rangle + \langle \bm w, v\rangle)^{2t}
= \sum_{j=0}^{2t} \binom{2t}{j} \mathbb{E} \langle \bm{\mu}, v\rangle^{j} \langle \bm{w}, v\rangle^{2t-j}\\
&\stackrel{(1)}{=} \sum_{j=0}^{2t} \binom{2t}{j} \mathbb{E} \langle \bm{\mu}, v\rangle^{j} \mathbb{E} \langle \bm{w}, v\rangle^{2t-j}
\stackrel{(2)}{=} \sum_{s=0}^t \binom{2t}{2s} \mathbb{E} \langle \bm{\mu}, v\rangle^{2s} \mathbb{E} \langle \bm{w}, v\rangle^{2t-2s}\\
&\stackrel{(3)}{=} \sum_{s=0}^t \binom{2t}{2s} \langle u, v\rangle^{2s} \mathbb{E} \langle \bm{\mu}, u\rangle^{2s} (v^\top \Sigma v)^{t-s} (2t-2s-1)!!
\end{align*}
where in (1) we used that \(\bm{\mu}\) and \(\bm{w}\) are independent, in (2) we used that \(\mathbb{E}\langle \bm{w}, v\rangle^{2t-j} = 0\) for \(2t-j\) odd, and in (3) we used that \(\bm{\mu} = \langle \bm{\mu}, u\rangle u\) and that \(\mathbb{E}\langle \bm{w}, v\rangle^{2t-2s} = (v^\top \Sigma v)^{t-s} (2t-2s-1)!!\).
\end{proof}

\begin{lemma}[Moment upper bound]
\label{lemma:moment-upper-bound-strong}
For $t \geq 1$ integer,
\[\sststile{2t}{v} \mathbb{E} \langle \bm{y}, v \rangle^{2t} \leq \left( \langle u, v\rangle^2 \left(\mathbb{E} \langle \bm{\mu}, u\rangle^{2t}\right)^{1/t} + et (v^\top \Sigma v) \right)^t.\]
\end{lemma}

\begin{proof}
Starting with the result in \cref{lemma:moment-equality-strong},
\begin{align*}
\sststile{2t}{v} \mathbb{E} \langle \bm y, v\rangle^{2t}
&= \sum_{s=0}^t \binom{2t}{2s} \langle u, v\rangle^{2s} \mathbb{E} \langle \bm{\mu}, u\rangle^{2s} (v^\top \Sigma v)^{t-s} (2t-2s-1)!!\\
&\stackrel{(1)}{\leq} \sum_{s=0}^t \binom{2t}{2s} \langle u, v\rangle^{2s} \left(\mathbb{E} \langle \bm{\mu}, u\rangle^{2t}\right)^{s/t} (v^\top \Sigma v)^{t-s} (2t-2s-1)!!\\
&\stackrel{(2)}{\leq} \sum_{s=0}^t \binom{t}{s} \langle u, v\rangle^{2s} \left(\mathbb{E} \langle \bm{\mu}, u\rangle^{2t}\right)^{s/t} (v^\top \Sigma v)^{t-s} (et)^{t-s}\\
&= \left( \langle u, v\rangle^2 \left(\mathbb{E} \langle \bm{\mu}, u\rangle^{2t}\right)^{1/t} + e t (v^\top \Sigma v) \right)^t.
\end{align*}
In (1) we used that \(s \leq t\) and Jensen's inequality as follows:
\begin{align*}
\mathbb{E} \langle \bm{\mu}, u\rangle^{2s}
&= \mathbb{E} \langle \bm{\mu}, u\rangle^{2t (s/t)}
\leq \left(\mathbb{E} \langle \bm{\mu}, u\rangle^{2t}\right)^{s/t}.
\end{align*}

In (2) we used that \(\binom{2t}{2s} (2t-2s-1)!! \leq \binom{t}{s} (et)^{t-s}\) for \(0 \leq s \leq t\) integers, which is proved in \cref{lemma:find-direction-binom}.
\end{proof}

\begin{lemma}[Moment lower bound]
\label{lemma:moment-lower-bound-strong}
For $t \geq 1$ integer,
\[\sststile{2t}{v} \mathbb{E} \langle \bm y, v\rangle^{2t} \geq \left( \langle u, v\rangle^2 \left(\mathbb{E} \langle \bm{\mu}, u\rangle^{2t}\right)^{1/t} + \pmin^{1/t}t/2 (v^\top \Sigma v) \right)^t.\]
\end{lemma}

\begin{proof}
Starting with the result in \cref{lemma:moment-equality-strong},
\begin{align*}
\sststile{2t}{v} \mathbb{E} \langle \bm y, v\rangle^{2t}
&= \sum_{s=0}^t \binom{2t}{2s} \langle u, v\rangle^{2s} \mathbb{E} \langle \bm{\mu}, u\rangle^{2s} (v^\top \Sigma v)^{t-s} (2t-2s-1)!!\\
&\stackrel{(1)}{\geq} \sum_{s=0}^t \binom{2t}{2s} \langle u, v\rangle^{2s} \left(\mathbb{E} \langle \bm{\mu}, u\rangle^{2t}\right)^{s/t} \left(\pmin^{1/t}\right)^{t-s} (v^\top \Sigma v)^{t-s} (2t-2s-1)!!\\
&\stackrel{(2)}{\geq} \sum_{s=0}^t \binom{t}{s} \langle u, v\rangle^{2s} \left(\mathbb{E} \langle \bm{\mu}, u\rangle^{2t}\right)^{s/t} \left(\pmin^{1/t}\right)^{t-s} (v^\top \Sigma v)^{t-s} (t/2)^{t-s}\\
&= \left( \langle u, v\rangle^2 \left(\mathbb{E} \langle \bm{\mu}, u\rangle^{2t}\right)^{1/t} + \pmin^{1/t}t/2 (v^\top \Sigma v) \right)^t.
\end{align*}

In (1) we used that \(s \leq t\) and the fact that the \(s\)-norm is greater than or equal to the \(t\)-norm as follows:
\begin{align*}
\mathbb{E} \langle \bm{\mu}, u\rangle^{2s}
&= \sum_{i=1}^k p_i \langle \mu_i, u\rangle^{2s}
= \sum_{i=1}^k (p_i^{1/s}\langle \mu_i, u\rangle^2)^s\\
&\geq \left(\sum_{i=1}^k (p_i^{1/s}\langle \mu_i, u\rangle^2)^t\right)^{s/t}
= \left(\sum_{i=1}^k p_i^{t/s} \langle \mu_i, u\rangle^{2t}\right)^{s/t}\\
&\geq \left(\pmin^{t/s-1}\sum_{i=1}^k p_i \langle \mu_i, u\rangle^{2t}\right)^{s/t}
= \pmin^{1-s/t} \left(\mathbb{E} \langle \bm{\mu}, u\rangle^{2t}\right)^{s/t}.
\end{align*}

In (2) we used that \(\binom{2t}{2s} (2t-2s-1)!! \geq \binom{t}{s} (t/2)^{t-s}\) for \(0 \leq s \leq t\) integers, which is proved in \cref{lemma:find-direction-binom}.

\end{proof}

\cref{lemma:find-direction-sigma} shows that the contribution of the means to the $\Omega(\log \pmin^{-1})$ moments in direction $u$ is lower bounded by $\Omega(k^2 \sigma^2 \log \pmin^{-1})$.
This result is used in some of the later proofs to argue that if the mean contribution is small, then $\sigma^2$ is small, and conversely, that if $\sigma^2$ is large, then the mean contribution is large.

\begin{lemma}
\label{lemma:find-direction-sigma}
For $2s \geq \lceil \log \pmin^{-1} \rceil$ integer,
\[(\mathbb{E} \langle \bm{\mu}, u\rangle^{2s})^{1/s} \geq \frac{C_{sep}}{100} k^2 \sigma^2 \log \pmin^{-1}.\]
\end{lemma}
\begin{proof}
By \cref{lemma:isotropic-position-separation}, for all $i \neq j$, $|\langle \mu_i - \mu_j, u\rangle| \geq \sqrt{C_{sep} \sigma^2 \log \pmin^{-1}}$. Then there exist $a, b\in[k]$ such that 
$|\langle \mu_a - \mu_b, u \rangle| \geq (k-1) \sqrt{C_{sep} \sigma^2 \log \pmin^{-1}}$. Hence, there exists $a \in [k]$ such that $|\langle \mu_a, u \rangle| \geq \frac{k-1}{2} \sqrt{C_{sep} \sigma^2 \log \pmin^{-1}}$. Then
\[(\mathbb{E} \langle \bm{\mu}, u\rangle^{2s})^{1/s} \geq \pmin^{1/s} \max_i \langle \mu_i, u \rangle^2 \geq \pmin^{1/s} \left(\frac{k-1}{2}\right)^2 C_{sep} \sigma^2 \log \pmin^{-1} \geq \frac{C_{sep}}{100} k^2 \sigma^2 \log \pmin^{-1},\]
where we used that $\pmin^{1/s} \geq e^{-2}$.
\end{proof}

We now state and prove the sum-of-squares identifiability proofs of \cref{lemma:find-direction-moment-maximization} and \cref{lemma:find-direction-moment-minimization}. Let $s, t = O(\log \pmin^{-1})$ with $s \ll t$. \cref{lemma:find-direction-moment-maximization} proves that, in the case \((\mathbb{E} \langle \bm{\mu}, u \rangle^{2s})^{1/s} \geq \Theta(s)\), if \(\mathbb{E} \langle \bm{y}, v\rangle^{2s}\) is close to its maximum value over unit vectors \(v\), then \(\langle u, v\rangle^{2s}\) is close to \(1\).
\cref{lemma:find-direction-moment-minimization} proves that, in the opposite case \((\mathbb{E} \langle \bm{\mu}, u \rangle^{2s})^{1/s} \leq \Theta(s)\), if \(\mathbb{E} \langle \bm{y}, v\rangle^{2t}\) is close to its minimum value over unit vectors \(v\), then \(\langle u, v\rangle^{2t}\) is close to \(1\).

\begin{lemma}[Direction sum-of-squares identifiability from moment maximization]
\label{lemma:find-direction-moment-maximization}
Let $M \geq 2$. Let $s$ be an integer such that \(2s \geq \lceil \log \pmin^{-1} \rceil\). Suppose that \(\mathbb{E} \langle \bm{\mu}, u \rangle^{2s} \geq (M es)^{s}\). Then, for $\epsilon \le \sigma^2/M$,
\[\left\{\|v\|^2=1, \mathbb{E} \langle \bm{y}, v\rangle^{2s} \geq (1-\epsilon)\left(\mathbb{E} \langle \bm{\mu}, u\rangle^{2s} - \epsilon\right)\right\} \sststile{2s}{v} \left\{\langle u, v \rangle^{2s} \geq (1-4\sigma^2/M)^{s}\right\}.\]
Furthermore, \(v=u\) satisfies the axiom with \(\epsilon=0\).
\end{lemma}

\begin{proof}
Substitute the upper bound of Lemma 3 into the axiom:
\[\sststile{2s}{v} \left( \langle u, v\rangle^2 \left(\mathbb{E} \langle \bm{\mu}, u\rangle^{2s}\right)^{1/s} + e s(v^\top \Sigma v) \right)^s \geq (1-\epsilon)\left(\mathbb{E} \langle \bm{\mu}, u\rangle^{2s} - \epsilon\right).\]
Divide by \(\mathbb{E} \langle \bm{\mu}, u\rangle^{2s}\):
\[\sststile{2s}{v} \left( \langle u, v\rangle^2 + (v^\top \Sigma v) \frac{e s}{\left(\mathbb{E} \langle \bm{\mu}, u\rangle^{2s}\right)^{1/s}} \right)^s \geq (1-\epsilon)\left(1 - \frac{\epsilon}{\mathbb{E} \langle \bm{\mu}, u\rangle^{2s}}\right).\]

Recall that \(\left(\mathbb{E} \langle \bm{\mu}, u\rangle^{2s}\right)^{1/s} \geq Mes\), and substitute the lower bound on both sides:
\begin{align*}
\sststile{2s}{v} \left( \langle u, v\rangle^2 + \frac{1}{M}(v^\top \Sigma v) \right)^s
&\geq (1-\epsilon)\left(1 - \frac{\epsilon}{(Mes)^s}\right)
\end{align*}

Use that \(v^\top \Sigma v = 1 - (1-\sigma^2)\langle u, v\rangle^2\):
\[\sststile{2t}{v} \left( \langle u, v\rangle^2 + \frac{1}{M}(1 - (1-\sigma^2)\langle u, v\rangle^2) \right)^s \geq (1-\epsilon)\left(1 - \frac{\epsilon}{(Mes)^s}\right).\]

We simplify now the right-hand side. Use the loose bound $1 - \epsilon/(Mes)^s \geq 1 - \epsilon \geq (1-\epsilon)^{s-1}$ to obtain
\[\sststile{2s}{v} \left( \langle u, v\rangle^2 + \frac{1}{M}(1 - (1-\sigma^2)\langle u, v\rangle^2) \right)^s \geq (1 - \epsilon)^s.\]

Finally, apply \cref{lemma:find-direction-sos-power-main-lower-bound} with \(x=\langle u, v\rangle\) and \(\gamma=\frac{1}{1-\epsilon}\) to obtain that
\begin{align*}
\sststile{2s}{v} \langle u, v \rangle^{2s}
&\geq \left(\frac{M-\frac{1}{1-\epsilon}}{\frac{1}{1-\epsilon}}\frac{1}{M-1+\sigma^2}\right)^s
= \left(\frac{M-1-M\epsilon}{M-1+\sigma^2}\right)^s\\
&\geq \left(\frac{M-1-\sigma^2}{M-1+\sigma^2}\right)^{s}
\geq \left(1-4\sigma^2/M\right)^s.
\end{align*}

To show that \(v=u\) satisfies the axiom, simply note that \cref{lemma:find-direction-moment-minimization} implies that $\mathbb{E} \langle \bm{y}, u\rangle^{2s} \geq \mathbb{E} \langle \bm{\mu}, u\rangle^{2s}$.

\end{proof}

\begin{lemma}[Direction sum-of-squares identifiability from moment minimization]
\label{lemma:find-direction-moment-minimization}
Suppose $\sigma^2 < 0.001$.
Let $s$ be an integer such that \(2s \geq \lceil \log \pmin^{-1} \rceil\).
Suppose that \(\mathbb{E} \langle \bm{\mu}, u \rangle^{2s} \leq (100s)^s\).
Let $t$ be an integer such that $t \geq 5000s$. Then, for \(\epsilon \leq \sigma^2/100\),
\[\left\{\|v\|^2 = 1, \mathbb{E} \langle \bm{y}, v\rangle^{2t} \leq (1+\epsilon)\left(\left(\left(\mathbb{E}\langle \bm{\mu}, u\rangle^{2t}\right)^{1/t} + et \sigma^2 \right)^{t} + \epsilon\right)\right\} \sststile{2t}{v} \left\{\langle u, v\rangle^{2t} \geq (1-20\sigma^2)^t\right\}.\]
Furthemore, \(v=u\) satisfies the axiom with \(\epsilon=0\).
\end{lemma}

\begin{proof}
We start by proving that, for \(t \geq s\), \(\mathbb{E} \langle \bm{\mu}, u \rangle^{2t} \leq (100e^{2}s)^{t}\).
We have that 
\[\pmin \cdot \max_{i} \langle \mu_i, u\rangle^{2s} \leq \mathbb{E} \langle \bm{\mu}, u \rangle^{2s} \leq (100s)^{s}.\]
Taking the \(s\)-th root and using that \(\pmin^{-1/s} \leq e^{2}\), we obtain that $\max_i \langle \mu_i, u\rangle^{2} \leq 100e^{2}s$.
Therefore, $\mathbb{E} \langle \bm{\mu}, u \rangle^{2t} \leq \max_i \langle \mu_i, u\rangle^{2t} = (100e^{2}s)^t$.

We now proceed with the main claim.
Substitute the lower bound of Lemma 4 into the axiom:
\[\sststile{2t}{v} \left( \langle u, v\rangle^2 \left(\mathbb{E} \langle \bm{\mu}, u\rangle^{2t}\right)^{1/t} + \pmin^{1/t}t/2 (v^\top \Sigma v) \right)^t \leq (1+\epsilon)\left(\left(\left(\mathbb{E}\langle \bm{\mu}, u\rangle^{2t}\right)^{1/t} + \sigma^2 et\right)^{t} + \epsilon\right).\]

Divide by \(\mathbb{E}\langle \bm{\mu}, u\rangle^{2t}\):
\[\sststile{2t}{v} \left( \langle u, v\rangle^2 + (v^\top \Sigma v) \frac{\pmin^{1/t}t/2}{\left(\mathbb{E} \langle \bm{\mu}, u\rangle^{2t}\right)^{1/t}} \right)^t \leq (1+\epsilon)\left(\left(1 + \sigma^2\frac{et}{\left(\mathbb{E} \langle \bm{\mu}, u\rangle^{2t}\right)^{1/t}}\right)^{t} + \frac{\epsilon}{\mathbb{E}\langle \bm{\mu}, u\rangle^{2t}}\right).\]

Let \(\Delta = \frac{\pmin^{1/t}t/2}{\left(\mathbb{E} \langle \bm{\mu}, u\rangle^{2t}\right)^{1/t}}\). Then
\[\sststile{2t}{v} \left( \langle u, v\rangle^2 + \Delta (v^\top \Sigma v) \right)^t \leq (1+\epsilon)\left(\left(1 + \sigma^2 \frac{2e}{\pmin^{1/t}} \Delta \right)^{t} + \frac{\epsilon}{\mathbb{E}\langle \bm{\mu}, u\rangle^{2t}}\right).\]

Note that \(\left(\mathbb{E} \langle \bm{\mu}, u\rangle^{2t}\right)^{1/t} \leq 100e^{2}s\) and \(\pmin^{1/t} = e^{-s/t} \geq e^{-1}\).
Then \(\Delta \geq \frac{e^{-1} t / 2}{100e^2s}\).
For $t \geq 5000s$ we have then \(\Delta \geq 10\) and \(\pmin^{-1/t} \leq 1.4\).
Then \(\frac{2e}{\pmin^{1/t}} \Delta \leq 8 \Delta\).
Then:
\[\sststile{2t}{v} \left( \langle u, v\rangle^2 + \Delta (v^\top \Sigma v)\right)^t \leq (1+\epsilon)\left(\left(1 + 8\Delta \sigma^2\right)^{t} + \frac{\epsilon}{\mathbb{E}\langle \bm{\mu}, u\rangle^{2t}}\right).\]
Divide by \(\left(1 + 8\Delta \sigma^2\right)^{t}\):
\[\sststile{2t}{v} \left( \frac{\langle u, v\rangle^2 + \Delta (v^\top \Sigma v)}{1 + 8\Delta \sigma^2}\right)^t \leq (1+\epsilon)\left(1 + \frac{\epsilon}{\mathbb{E}\langle \bm{\mu}, u\rangle^{2t}\left(1 + 8\Delta \sigma^2\right)^{t}}\right).\]

Use that \(v^\top \Sigma v = 1 - (1-\sigma^2)\langle u, v\rangle^2\):
\[\sststile{2t}{v} \left( \frac{\langle u, v\rangle^2 + \Delta (1 - (1-\sigma^2)\langle u, v\rangle^2)}{1 + 8\Delta \sigma^2}\right)^t \leq (1+\epsilon)\left(1 + \frac{\epsilon}{\mathbb{E}\langle \bm{\mu}, u\rangle^{2t}\left(1 + 8\Delta \sigma^2\right)^{t}}\right).\]

We simplify now the term involving \(\epsilon\).
Note that, by Jensen's inequality, \(\mathbb{E}\langle \bm{\mu}, u\rangle^{2t} \geq (\mathbb{E} \langle \bm{\mu}, u\rangle^2)^t = (1-\sigma^2)^t\). Also note that $(1-\sigma^2)(1+8\Delta\sigma^2) \geq 1$ for $\sigma^2 \leq 1/2$ and $\Delta \geq 10$. Then use the loose bound
\[1 + \frac{\epsilon}{\mathbb{E}\langle \bm{\mu}, u\rangle^{2t}\left(1 + 8\Delta \sigma^2\right)^{t}} \leq 1 + \frac{\epsilon}{((1-\sigma^2)(1+8\Delta\sigma^2))^t} \leq 1 + \epsilon \leq (1+\epsilon)^{t-1}\]
to obtain
\[\sststile{2t}{v} \left( \frac{\langle u, v\rangle^2 + \Delta (1 - (1-\sigma^2)\langle u, v\rangle^2)}{1 + 8\Delta \sigma^2}\right)^t \leq (1 + \epsilon)^t.\]

Finally, apply \cref{lemma:find-direction-sos-power-main-upper-bound} with \(x=\langle u, v\rangle\) and \(\gamma=\frac{1}{1+\epsilon}\) to obtain that
\begin{align*}
\sststile{2t}{v} \langle u, v\rangle^{2t}
&\geq \left( \frac{\frac{1}{1+\epsilon}\Delta - 1}{\frac{1}{1+\epsilon}(\Delta-1)} (1-10\sigma^2)\right)^t = \left( \left( 1 - \frac{\epsilon}{\Delta-1} \right) (1-10\sigma^2)\right)^t\\
&= \left((1-\sigma^2/100) (1-10\sigma^2)\right)^t \geq \left(1-20\sigma^2\right)^t.
\end{align*}

To show that \(v=u\) satisfies the axiom, simply note that \cref{lemma:find-direction-moment-maximization} implies that ${\mathbb{E} \langle \bm{y}, u\rangle^{2t} \leq \left(\left(\mathbb{E} \langle \bm{\mu}, u\rangle^{2t}\right)^{1/t} + et \sigma^2 \right)^t}$.

\end{proof}

We now prove \cref{thm:find-direction-main}.

\begin{proof}[Proof of \cref{thm:find-direction-main}.]
Let $s = \lceil \log \pmin^{-1} \rceil$, $t = 5000s$, and $\tau=800e/(C_{sep}k^2)$. The algorithm is:
\begin{enumerate}
  \item If $\sigma^2 \geq \tau$, then let $M= C_{sep}k^2\sigma^2/(200e)$. Else, let $M=4$.
  \item Binary search up to resolution \(\sigma^2/(100M)\) the largest \(T_U\) in the interval \([0, (\pmin^{-1})^s]\) such that there exists a degree-\(2s\) pseudo-expectation that satisfies \(\{\|v\|^2=1, \mathbb{E}\langle \bm{y}, v\rangle^{2s} \geq T_U\}\). Let \(\tilde{\mathbb{E}}_U\) be the resulting pseudo-expectation for this \(T_U\).
  \item Binary search up to resolution \(\sigma^2/10000\) the smallest \(T_L\) in the interval \([0, (\pmin^{-1}+et)^t]\) such that there exists a degree-\(2t\) pseudo-expectation that satisfies \(\{\|v\|^2=1, \mathbb{E}\langle \bm{y}, v\rangle^{2t} \leq T_L\}\).
  Let \(\tilde{\mathbb{E}}_L\) be the resulting pseudo-expectation for this \(T_L\). 
  \item Return $\tilde{\mathbb{E}}_U$ and $\tilde{\mathbb{E}}_L$.
\end{enumerate}

We now analyze the algorithm.
To begin with, suppose that the \(T_U\) found is at least the maximum value of $\mathbb{E} \langle \bm{y}, v\rangle^{2s}$ and that the \(T_L\) found is at most the minimum value of $\mathbb{E} \langle \bm{y}, v\rangle^{2t}$.
In this case $\tilde{\mathbb{E}}_U$ and $\tilde{\mathbb{E}}_L$ satisfy the axioms of \cref{lemma:find-direction-moment-maximization} and \cref{lemma:find-direction-moment-minimization}, respectively. Then our algorithm achieves the stated guarantees:
\begin{itemize}
  \item Suppose $\sigma^2 \geq \tau$. Note that, in this case, $M= C_{sep}k^2\sigma^2/(200e) \geq 2$. By \cref{lemma:find-direction-sigma}, we have that 
  \[(\mathbb{E} \langle \bm{\mu}, u\rangle^{2s})^{1/s} \geq \frac{C_{sep}}{100} k^2 \sigma^2 \log \pmin^{-1} = 2e M \log \pmin^{-1} \geq Mes.\]
  Then the conditions of \cref{lemma:find-direction-moment-maximization} are satisfied, and $\tilde{\mathbb{E}}_U$ satisfies $\langle u, v\rangle^{2s} \geq (1-4\sigma^2/M)^s = (1-\tau)^s$.
  \item Suppose $\sigma^2 < \tau$ and $\mathbb{E} \langle \bm{\mu}, u\rangle^{2s} \geq (4es)^s$. Note that, in this case, $M=4$. Then the conditions of \cref{lemma:find-direction-moment-maximization} are satisfied, and $\tilde{\mathbb{E}}_U$ satisfies $\langle u, v\rangle^{2s} \geq (1-4\sigma^2/M)^s = (1-\sigma^2)^s$.
  \item Suppose $\sigma^2 < 0.001$ and $\mathbb{E} \langle \bm{\mu}, u\rangle^{2s} \leq (100s)^s$. Then the condition of \cref{lemma:find-direction-moment-minimization} are satisfied, and $\tilde{\mathbb{E}}_L$ satisfies $\langle u, v\rangle^{2s} \geq (1-20\sigma^2)^s$.
\end{itemize}

We argue now that $T_U$ is large enough and that $T_L$ is small enough in order for the pseudo-expectations to satisfy the axioms of the lemmas.
For that, we need 
\[T_U \geq (1-\sigma^2/M) \left(\mathbb{E}\langle\bm{\mu}, u\rangle^{2s} - \sigma^2/M\right),\]
\[T_L \leq (1+\sigma^2/100) \left(\left(\left(\mathbb{E}\langle \bm{\mu}, u\rangle^{2t}\right)^{1/t} + \sigma^2 et\right)^t + \sigma^2/100\right).\]

We prove that the intervals in which we binary search $T_U$ and $T_L$ contain $\mathbb{E} \langle \bm{\mu}, u\rangle^{2s}$ and $((\mathbb{E} \langle \bm{\mu}, u\rangle^{2t})^{1/t}+\sigma^2 et)^t$, respectively.
Then, binary search with the proposed resolutions is guaranteed to find $T_U$ and $T_L$ that satisfy the bounds stated above.

Using that \(\mathbb{E}\langle \bm{\mu}, u\rangle^{2}=1-\sigma^2\), we have that
\begin{align*}
\mathbb{E}\langle \bm{\mu}, u\rangle^{2t} 
&= \sum_{i=1}^k p_i \langle \mu_i, u\rangle^{2t}
= \sum_{i=1}^k (p_i^{1/t} \langle \mu_i, u\rangle^{2})^{t}
\leq \left(\sum_{i=1}^k p_i^{1/t} \langle \mu_i, u\rangle^{2}\right)^t\\
&\leq \left(\pmin^{1/t-1} \sum_{i=1}^k p_i \langle \mu_i, u\rangle^{2}\right)^t
= \pmin^{-(t-1)} \mathbb{E}\langle \bm{\mu}, u\rangle^{2}
= \pmin^{-(t-1)} (1-\sigma^2)\\
&\leq \pmin^{-t}
\end{align*}
and 
\begin{align*}
\mathbb{E}\langle \bm{\mu}, u\rangle^{2t} 
&\geq (\mathbb{E}\langle \bm{\mu}, u\rangle^{2})^t = (1-\sigma^2)^t. 
\end{align*}

Therefore,
\[\mathbb{E}\langle\bm{\mu}, u\rangle^{2s} \in [(1-\sigma^2)^s, (\pmin^{-1})^s],\]
\[\left(\left(\mathbb{E}\langle \bm{\mu}, u\rangle^{2t}\right)^{1/t} + \sigma^2 et\right)^t \in [(1-\sigma^2)^t, (\pmin^{-1} + e t)^t].\]
Then the intervals in which we binary search are wide enough and binary search is guaranteed to succeed.

The time complexity of the algorithm is given by the number of steps in the binary search multiplied by the time to compute each of the pseudo-expectations. The number of steps in the binary search is
\[O\left(\max\left\{\log ((\pmin^{-1})^s k^2), \log \frac{(\pmin^{-1})^s}{\sigma^2}, \log \frac{(\pmin^{-1}+et)^t}{\sigma^2}\right\}\right) = O\left(\log \frac{1}{\sigma^2} + \log^2 \pmin^{-1}\right).\]
For each step, we compute a pseudo-expectation of degree \(O(\log \pmin^{-1})\) over \(d\) variables, which requires time \(d^{O(\log \pmin^{-1})}\).
Therefore the time complexity is
\[O\left(\log \frac{1}{\sigma^2} + \log^2 \pmin^{-1}\right) \cdot d^{O(\log \pmin^{-1})} = \left(\log \frac{1}{\sigma^2}\right) \cdot d^{O(\log \pmin^{-1})}.\]

\end{proof}

\subsubsection{Sum-of-squares sampling (proof of Theorem~\ref{thm:sample-direction-main})}
\label{sec:sossample}

We state and prove \cref{lemma:sample-direction-matrix-to-vector}, which is used in the proof of \cref{thm:sample-direction-main}.
This lemma shows that, given a symmetric postivie definite matrix $M$ correlated with a rank-$1$ matrix $uu^\top$ for a unit vector $u$, there exists an algorithm to recover a unit vector correlated with $u$.
After that, we proceed to prove the theorem.

\begin{lemma}[Matrix rank-$1$ approximation]
\label{lemma:sample-direction-matrix-to-vector}
Let \(0 \leq \epsilon < \frac{1}{8}\).
Let \(u \in \mathbb{R}^d\) be a unit vector.
Given a symmetric positive semi-definite matrix \(M \in \mathbb{R}^{d \times d}\) with \(\|M\|_F \leq 1\) such that \(\langle uu^\top, M\rangle_F \geq 1-\epsilon\), 
there exists a polynomial-time algorithm that finds a unit vector \(\hat{u} \in \mathbb{R}^d\) such that \(\langle u, \hat{u}\rangle^2 \geq 1-8\epsilon\).
\end{lemma}

\begin{proof}
The algorithm is to compute $vv^\top$ as the best rank-$1$ approximation of $M$ and return $\frac{v}{\|v\|}$, which is uniquely defined up to a sign flip.

We now analyze the accuracy of the algorithm. We have that $\langle uu^\top, M \rangle_F \geq 1-\epsilon$, so $\|uu^\top - M\|_F^2 \leq 2 - 2\langle uu^\top, M\rangle_F \leq 2\epsilon$.
For \(vv^\top\) the best rank-\(1\) approximation of \(M\), we have then that
\[\|uu^\top - vv^\top\|_F \leq \left\|uu^\top - M\right\|_F + \left\|M - vv^\top\right\|_F \leq 2\sqrt{2\epsilon},\]
so $\|uu^\top- vv^\top\|_F^2 \leq 8\epsilon$. Let $\hat{u} = \frac{v}{\|v\|}$. To analyze the error of $\hat{u}$, note that 
\begin{align*}
\|uu^\top - vv^\top\|_F^2
&= 1 + \|v\|^4 - 2\|v\|^2\left\langle uu^\top, \hat{u}\hat{u}^\top\right\rangle_F\\
&\geq 1 - \left\langle uu^\top, \hat{u}\hat{u}^\top\right\rangle_F\\
&= \frac{1}{2} \left\|uu^\top - \hat{u}\hat{u}^\top \right\|_F^2,
\end{align*}
where in the inequality we used that \(1+x^4-2x^2 y \geq 1-y\) for \(x \in \mathbb{R}\) and \(0 \leq y \leq 1\), with $x=\|v\|$ and $y=\langle uu^\top, \hat{u}\hat{u}^\top\rangle$. Then ${\|uu^\top - \hat{u}\hat{u}^\top\|_F^2 \leq 16\epsilon}$.
Therefore,
\[\langle u, \hat{u}\rangle^2 = \langle uu^\top, \hat{u}\hat{u}^\top\rangle_F = 1 - \frac{1}{2} \|uu^\top-\hat{u}\hat{u}^\top\|_F^2 \geq 1 - 8\epsilon.\]
\end{proof}

\begin{proof}[Proof of \cref{thm:sample-direction-main}]
The algorithm is to compute $M = \tilde{\mathbb{E}} vv^\top$, apply the algorithm from \cref{lemma:sample-direction-matrix-to-vector} to $M$ in order to obtain a unit vector $\hat{u}$, and return $\hat{u}$.

We now analyze the algorithm. We start by analyzing the properties of $\tilde{\mathbb{E}}$ in more detail. Our first goal is to obtain the lower bound \(\tilde{\mathbb{E}} \langle u, v\rangle^2 \geq 1-2\epsilon\).
We start by proving the much weaker lower bound \(\tilde{\mathbb{E}} \langle u, v\rangle^2 \geq 1-t\epsilon\). Then, we use this lower bound to prove an upper bound \(\tilde{\mathbb{E}} \langle u,v\rangle^{2t} \leq 1 - t (1-\tilde{\mathbb{E}} \langle u, v\rangle^2)/2\). 
Comparing this result with the given lower bound $\tilde{\mathbb{E}} \langle u, v\rangle^{2t} \geq (1-\epsilon)^t \geq 1-t\epsilon$ leads to the conclusion that \(\tilde{\mathbb{E}} \langle u, v\rangle^2 \geq 1-2\epsilon\).

We proceed with the detailed proof of this fact. Recall that $\tilde{\mathbb{E}}$ satisfies $\|v\|^2 = 1$ and $\langle u, v\rangle^{2t} \geq (1-\epsilon)^t$. We have that \(\{\|v\|^2 = 1\} \sststile{2}{v} \{0 \leq \langle u, v\rangle^2 \leq 1\}\), where the lower bound is trivial and the upper bound is by \cref{lemma:sos-cs}. Therefore, $\tilde{\mathbb{E}}$ also satisfies $0 \leq \langle u, v\rangle^2 \leq 1$.

By \cref{lemma:sample-direction-lower-bound-root} and using that $(1-\epsilon)^t \geq 1-t\epsilon$, we have that
\[\{0 \leq \langle u, v\rangle^2 \leq 1, \langle u, v\rangle^{2t} \geq (1-\epsilon)^t\} \sststile{2t}{v} \{\langle u, v\rangle^2 \geq 1-t\epsilon\}.\]
By \cref{lemma:sample-direction-boost} applied to \(1-\langle u,v\rangle^2\) with \(C=\frac{1}{t^2 \epsilon}\), we also have that
\[\{1 - t\epsilon \leq \langle u, v\rangle^2 \leq 1\} \sststile{2t}{v} \left\{\langle u, v\rangle^{2t} \leq 1-t(1-\langle u,v\rangle^2)/2\right\}.\]
Then
\[1-t\epsilon \leq \tilde{\mathbb{E}} \langle u, v\rangle^{2t} \leq 1- t (1-\tilde{\mathbb{E}} \langle u,v\rangle^2)/2,\]
so by rearranging, $\tilde{\mathbb{E}} \langle u,v\rangle^2 \geq 1-2\epsilon$.

Then $M = \tilde{\mathbb{E}} vv^\top$ satisfies 
\[\langle uu^\top, M\rangle_F = u^\top M u = \tilde{\mathbb{E}} \langle u, v\rangle^2 \geq 1-2\epsilon.\]
In addition, $M$ is symmetric positive-definite and 
\[\|M\|_F \leq \operatorname{Tr}(M) = \tilde{\mathbb{E}} \operatorname{Tr}(vv^\top) = \tilde{\mathbb{E}} \|v\|^2 = 1.\]
Therefore, $M$ satisfies the conditions of \cref{lemma:sample-direction-matrix-to-vector}, and we are guaranteed that $\hat{u}$ satisfies \(\langle u, \hat{u}\rangle^2 \geq 1-16\epsilon\).

The given pseudo-expectation is of degree $O(t)$ over $d$ variables, so representing it requires $d^{O(t)}$ space. Then we simply bound the time complexity by $d^{O(t)}$, which dominates the other steps of the algorithm.
\end{proof}

\subsection{Finite sample bounds}
\label{sec:finitesample_pp}

In \cref{sec:sosident} we assumed oracle access to $\mathbb{E} \bm{y}^{\otimes t}$. However, our algorithm only has access to empirical moments. \cref{lemma:finite-moments-change-empirical} gives a sum-of-squares proof that that the empirical moments are in fact close to the population moments. We defer the proof of the lemma to the appendix.

\begin{lemma}[Closeness of empirical moments and population moments]
\label{lemma:finite-moments-change-empirical}
For \(n \geq (\pmin^{-1} d)^{O(t)} \eta^{-2} \epsilon^{-1}\), with probability $1-\epsilon$,
\[\left\{\|v\|^2 = 1\right\} \sststile{O(t)}{v} \left\{\hat{\mathbb{E}}\langle \bm y, v\rangle^{2t} \leq \mathbb{E}\langle \bm y, v\rangle^{2t} + \eta\right\},\]
\[\left\{\|v\|^2 = 1\right\} \sststile{O(t)}{v} \left\{\hat{\mathbb{E}}\langle \bm y, v\rangle^{2t} \geq \mathbb{E}\langle \bm y, v\rangle^{2t} - \eta\right\}.\]
\end{lemma}
\begin{proof}
See \cref{proofs-pp}.
\end{proof}

\subsection{Proof of Theorem~\ref{thm:parallel-pancakes-main}}
\label{subsec:proof-pp}

In the setting of \cref{thm:parallel-pancakes-main} we only have access to empirical moments.
\cref{thm:pp-finite-sample-find-direction} and \cref{thm:pp-finite-sample-high-correlation} adapt \cref{thm:find-direction-main} and \cref{thm:infinite-algorithm-high-correlation} to this setting, respectively.
Also recall that the goal of \cref{thm:parallel-pancakes-main} is to return a clustering, not only a unit vector close to $u$.
Toward that goal, \cref{thm:clustering-main} shows that there exists an algorithm that, given a unit vector close to $u$, computes such a clustering.
We state and prove all of these theorems and then combine them to prove \cref{thm:parallel-pancakes-main}.

\begin{theorem}[Finite sample equivalent of \cref{thm:find-direction-main}]
\label{thm:pp-finite-sample-find-direction}
Let 
\[n_0 = \left(\frac{1}{\sigma^2}\right)^{O(1)} \cdot (\pmin^{-1} d)^{O(\log \pmin^{-1})}.\]
Given a sample of size $n \geq n_0$ from the mixture, there exists an algorithm that runs in time \(\left(\log \frac{1}{\sigma^2}\right) \cdot n \cdot d^{O(\log \pmin^{-1})}\) that computes with high probability two peseudo-expectations \(\tilde{\mathbb{E}}_U\) and \(\tilde{\mathbb{E}}_L\) of degree \(O(\log \pmin^{-1})\) over a variable \(v \in \mathbb{R}^d\) such that the following holds.
Let \(s = \lceil \log \pmin^{-1} \rceil\), let \(t = 5000s\), and let $\tau = \frac{800e}{C_{sep}k^2}$. 
Then \(\tilde{\mathbb{E}}_U \|v\|^2 = 1\), \(\tilde{\mathbb{E}}_L \|v\|^2 = 1\), and:

\begin{itemize}
\tightlist
\item If $\sigma^2 \geq \tau$, then \(\tilde{\mathbb{E}}_U \langle u, v\rangle^{2s} \geq (1-\tau)^{s}\).
\item If $\sigma^2 < \tau$ and $\mathbb{E} \langle \bm{\mu}, u\rangle^{2s} \geq (4es)^s$, then
  \(\tilde{\mathbb{E}}_U \langle u, v\rangle^{2s} \geq (1-\sigma^2)^{s}\).
\item If $\sigma^2 < 0.001$ and $\mathbb{E} \langle \bm{\mu}, u\rangle^{2s} \leq (100s)^s$, then
  \(\tilde{\mathbb{E}}_L \langle u, v\rangle^{2t} \geq (1-20\sigma^2)^{t}\).
\end{itemize}
\end{theorem}

\begin{proof}
The algorithm is the same as that in the proof of \cref{thm:find-direction-main}, except that in step (2) and step (3) of the algorithm the constraints that the pseudo-expectations are required to satisfy are $\{\|v\|^2=1, \hat{\mathbb{E}}\langle \bm{y}, v\rangle^{2s} \geq T_U\}$ and $\{\|v\|^2=1, \hat{\mathbb{E}}\langle \bm{y}, v\rangle^{2t} \leq T_L\}$, respectively.

By \cref{lemma:finite-moments-change-empirical}, for $n \geq n_0$, we have that with high probability 
\[\{\|v\|^2=1, \hat{\mathbb{E}}\langle \bm{y}, v\rangle^{2s} \geq T_U\} \sststile{2s}{v} \{\hat{\mathbb{E}}\langle \bm{y}, v\rangle^{2s} \geq T_U - \sigma^2/(100M)\},\]
\[\{\|v\|^2=1, \hat{\mathbb{E}}\langle \bm{y}, v\rangle^{2t} \leq T_L\} \sststile{2s}{v} \{\hat{\mathbb{E}}\langle \bm{y}, v\rangle^{2t} \leq T_L + \sigma^2/10000\}.\]

These errors, combined with the errors from the binary search resolution, are still within the amount tolerated by \cref{lemma:find-direction-moment-maximization} and \cref{lemma:find-direction-moment-minimization}, respectively, so the same guarantees hold.

The number of steps required by the binary search is the same as in \cref{thm:find-direction-main}.
For each step of the binary search, we compute a pseudo-expectation of degree \(O(\log \pmin^{-1})\) over \(d\) variables, and each constraint requires summing over the $n$ samples, so the time required is \(n \cdot d^{O(\log \pmin^{-1})}\). 
Therefore the time complexity is
\[O\left(\log \frac{1}{\sigma^2} + \log^2 \pmin^{-1}\right) \cdot n \cdot d^{O(\log \pmin^{-1})} = \left(\log \frac{1}{\sigma^2}\right) \cdot n \cdot d^{O(\log \pmin^{-1})}.\]
\end{proof}

\begin{theorem}[Finite sample equivalent of \cref{thm:infinite-algorithm-high-correlation}]
\label{thm:pp-finite-sample-high-correlation}
Let 
\[n_0 = \left(\frac{1}{\sigma^2}\right)^{O(1)} \cdot (\pmin^{-1} d)^{O(\log \pmin^{-1})}.\]
Given a sample of size $n \geq n_0$ from the mixture, there exists an algorithm with time complexity \(\left(\log \frac{1}{\sigma^2}\right) \cdot n \cdot d^{O(\log \pmin^{-1})}\) that outputs with high probability a unit vector \(\hat{u} \in \mathbb{R}^d\) such that \(\langle u, \hat{u}\rangle^2 = 1 - 320 \min(\sigma^2, \frac{1}{k})\).
\end{theorem}

\begin{proof}
The algorithm is the same as that in the proof of \cref{thm:infinite-algorithm-high-correlation}, with two exceptions:

\begin{itemize}
  \item In step (1) of the algorithm, we run the algorithm from \cref{thm:pp-finite-sample-find-direction} instead of the algorithm from \cref{thm:find-direction-main}.
  The pseudo-expectations $\tilde{\mathbb{E}}_L$ and $\tilde{\mathbb{E}}_U$ satisfy the same guarantees.
  \item In step (3) of the algorithm, we check if $\hat{\mathbb{E}} \langle \bm{y}, \hat{u}_U\rangle^{2s} \geq (50s)^s$ instead of $\mathbb{E} \langle \bm{y}, \hat{u}_U\rangle^{2s} \geq (50s)^s$.
  By \cref{lemma:finite-moments-change-empirical}, for $n \geq n_0$, with high probability the difference between the two moments is less than $1$.
  Then, if $\hat{\mathbb{E}} \langle \bm{y}, \hat{u}_U\rangle^{2s} \geq (50s)^s$, we also have $\mathbb{E} \langle \bm{y}, \hat{u}_U\rangle^{2s} \geq (50s)^s - 1$, and if $\hat{\mathbb{E}} \langle \bm{y}, \hat{u}_U\rangle^{2s} < (50s)^s$, we also have $\mathbb{E} \langle \bm{y}, \hat{u}_U\rangle^{2s} < (50s)^s + 1$.
  It is easy to verify that the analysis in \cref{thm:infinite-algorithm-high-correlation} is still correct with these slightly weaker bounds.
\end{itemize}

Therefore, the same guarantees hold as in \cref{thm:infinite-algorithm-high-correlation}.
The time complexity of the algorithm is dominated by the time to run the algorithm from \cref{thm:pp-finite-sample-find-direction}.

\end{proof}

\begin{theorem}[Clustering algorithm]
\label{thm:clustering-main}
For some $C > 0$, suppose that a unit vector $\hat{u} \in \mathbb{R}^d$ is known such that \(\langle u, \hat{u}\rangle^2 \geq 1 - C \min(\sigma^2, \frac{1}{k})\).
Suppose that $C_{sep}/C$ is larger than some universal constant.
Then, given a sample of size $n \geq (\pmin^{-1})^{O(1)}$ from the mixture, there exists an algorithm that runs in time $n^{O(\log \pmin^{-1})}$ and returns with high probability a partition of \([n]\) into $k$ sets $C_1, ..., C_k$ such that, if the true clustering of the samples is $S_1, ..., S_k$, then there exists a permutation $\pi$ of $[k]$ such that 
\[1 - \frac{1}{n} \sum_{i=1}^k |C_i \cap S_{\pi(i)}| \leq \left(\frac{\pmin}{k}\right)^{O(1)}.\]
\end{theorem}

\begin{proof}
Our algorithm runs the algorithm from Theorem 5.1 of \cite{MR3826314-Hopkins18} with some $t = O(\log \pmin^{-1})$ large enough on input samples \(\langle y_1, \hat{u}\rangle/(\sqrt{2(C+1)\sigma^2})\), ..., \(\langle y_n, \hat{u}\rangle/(\sqrt{2(C+1)\sigma^2})\), and returns the clustering that this algorithm computes as an intermediate step.

We now analyze the algorithm. Note that $\langle \bm{y}, \hat{u}\rangle/\sqrt{2(C+1)\sigma^2}$ is distributed according to a one-dimensional mixture of Gaussians in which all the components have the same variance $\hat{u}^\top \Sigma \hat{u}/(2(C+1)\sigma^2)$.
We have that 
\[\hat{u}^\top \Sigma \hat{u} = 1 - (1 - \sigma^2)\langle u, \hat{u}\rangle^2 = 1 - \langle u, \hat{u}\rangle^2 + \sigma^2 \langle u, \hat{u}\rangle^2 \leq C \sigma^2 + \sigma^2 = (C+1)\sigma^2.\]
Therefore, the variance $\hat{u}^\top \Sigma \hat{u}/(2(C+1)\sigma^2)$ is upper bounded by $1/2$.
For the guarantees of the algorithm from \cite{MR3826314-Hopkins18} to hold, we further need to show that the mixture has large separation between the means of the components.
Note that the mean corresponding to $\mu_i$ in the original mixture becomes $\langle \mu_i, \hat{u}\rangle/\sqrt{2(C+1)\sigma^2}$ in the new mixture.
For \(i \neq j\), we have that
\begin{align*}
(\langle \mu_i - \mu_j, \hat{u}\rangle)^2
&= \langle \langle \mu_i, u\rangle u - \langle \mu_j, u\rangle u, \hat{u}\rangle^2
= \langle u, \hat{u}\rangle^2 \langle \mu_i - \mu_j, u \rangle^2\\
&\geq \langle u, \hat{u}\rangle^2 C_{sep} (u^\top\Sigma u) \log \pmin^{-1}
= \langle u, \hat{u}\rangle^2 C_{sep} \sigma^2 \log \pmin^{-1}\\
&\geq \frac{C_{sep}}{2} \sigma^2 \log \pmin^{-1}
\end{align*}
where in the last inequality we used that \(\langle u, \hat{u}\rangle^2 \geq 1 - C/k \geq 1/2\). Then 
\begin{align*}
\left(\frac{\langle \mu_i - \mu_j, \hat{u}\rangle}{\sqrt{2(C+1)\sigma^2}}\right)^2
&\geq \frac{C_{sep}}{4(C+1)} \log \pmin^{-1}.
\end{align*}

For $C_{sep}/C$ larger than some universal constant, the separation coefficient $C_{sep}/(4(C+1))$ is large enough for the guarantees of Theorem 5.1 of \cite{MR3826314-Hopkins18} to hold meaningfully with $t=O(\log \pmin^{-1})$.
Then this algorithm computes a clustering with the stated guarantees.
The algorithm requries $n \geq (\pmin^{-1})^{O(1)}$ and runs in time $n^{O(\log \pmin^{-1})}$.
\end{proof}

\begin{proof}[Proof of \cref{thm:parallel-pancakes-main}]
Run the algorithm from \cref{thm:pp-finite-sample-high-correlation} to obtain a unit vector $\hat{u}$ that satisfies $\langle u, \hat{u}\rangle^2 \geq 1-320\min(\sigma^2, \frac{1}{k})$.
Then run the clustering algorithm from \cref{thm:clustering-main} using this unit vector $\hat{u}$.
For $C_{sep}/320$ larger than some universal constant, this algorithm is guaranteed to return a clustering with the stated guarantees.

The time complexity from \cref{thm:pp-finite-sample-high-correlation} is \(\left(\log \frac{1}{\sigma^2}\right) \cdot n \cdot d^{O(\log \pmin^{-1})}\) and the time complexity from \cref{thm:colinear-main} is $n^{O(\log \pmin^{-1})}$.
We assume $n \geq \left(\frac{1}{\sigma^2}\right)^{O(1)} \cdot (\pmin^{-1} d)^{O(\log \pmin^{-1})}$.
Therefore, the time complexity is dominated by the time to run the clustering algorithm from \cref{thm:colinear-main}.
\end{proof}

\section{Colinear means}
\label{sec:colinear}

In this section we remove the isotropic position assumption from the model in \cref{sec:pp}. Our strategy is straightfoward: we first put the mixture in isotropic position and then run the algorithm from \cref{thm:parallel-pancakes-main}.
The technical challenge is that we can only put the mixture in approximate isotropic position.
Then, we show that the guarantees of \cref{thm:parallel-pancakes-main} continue to hold with approxiamte isotropic position, albeit with a sample complexity that depends on the condition number of the covariance matrix of the mixture.

\paragraph{Setting.} We consider a mixture of \(k\) Gaussian distributions \(N(\mu_i^0, \Sigma^0)\) with mixing weights \(p_i\) for \(i=1,...,k\),
where \(\mu_i^0 \in \mathbb{R}^d\), \(\Sigma^0 \in \mathbb{R}^{d \times d}\) is positive definite, and \(p_i \geq 0\) and \(\sum_{i=1}^k p_i = 1\). 
Let \(\pmin = \min_i p_i\).

The distribution also satisfies mean separation and mean colinearity:
\begin{itemize}
\tightlist
\item
  Mean separation:
  for some $C_{sep} > 0$ and for all \(i \neq j\),
  \[\left\|\left(\Sigma^{0}\right)^{-1/2} (\mu_i^0 - \mu_j^0)\right\|^2 \geq C_{sep} \log \pmin^{-1}.\]
\item
  Mean colinearity:
  for some vector \(\mu_{base}^0 \in \mathbb{R}^d\) and some unit vector \(u^0 \in \mathbb{R}^d\) and for all \(i\),
  \[\mu_i^0 = \mu_{base}^0 + \langle \mu_i^0 - \mu_{base}^0, u^0\rangle u^0.\]
\end{itemize}

Also define, for $\bm{y}^0$ distributed according to the mixture, 
\[\sigma^2 = \frac{(u^0)^\top \operatorname{cov}(\bm{y}^0)^{-1} u^0}{(u^0)^\top (\Sigma^0)^{-1} u^0}.\]
As shown in \cref{lemma:isotropic-position-sigma-sq}, $\sigma^2$ has the same interpretation as in \cref{sec:pp}: it is equal to the variance of the components in the direction of the means after an isotropic position transformation.

\begin{theorem}[Colinear means algorithm]
\label{thm:colinear-main}
Consider the Gaussian mixture model defined above, with $C_{sep}$ larger than some universal constant. For $\bm{y}^0$ distributed according to the mixture, let 
\[n_0 = \left(\frac{1}{\sigma^2} \cdot \|\operatorname{cov}(\bm{y}^0)\| \cdot \|\operatorname{cov}(\bm{y}^0)^{-1}\| \right)^{O(1)} \cdot (\pmin^{-1} d)^{O(\log \pmin^{-1})}.\]
Given a sample of size $n \geq n_0$ from the mixture, there exists an algorithm that runs in time $n^{O(\log \pmin^{-1})}$ and returns with high probability a partition of \([n]\)  into $k$ sets $C_1, ..., C_k$ such that, if the true clustering of the samples is $S_1, ..., S_k$, then there exists a permutation $\pi$ of $[k]$ such that 
\[1 - \frac{1}{n} \sum_{i=1}^k |C_i \cap S_{\pi(i)}| \leq \left(\frac{\pmin}{k}\right)^{O(1)}.\]
\end{theorem}

We introduce some further notation for this section.
Let $\bm{y}^0$ be distributed according to the mixture.
We specify the model as \(\bm{y}^0 = \bm{\mu}^0 + \bm{w}^0\),
where \(\bm{\mu}^0\) takes value \(\mu_i\) with probability \(p_i\) and \(\bm{w}^0 \sim N(0, \Sigma^0)\), with \(\bm{\mu}^0\) and \(\bm{w}^0\) independent of each other.

\subsection{Isotropic position transformation}
\label{sec:isotropic}

In this section we argue that, if we put the mixture in exact isotropic position, the conditions of \cref{thm:parallel-pancakes-main} are satisfied and we can simply run that algorithm.

Assume acces to \(\mathbb{E} \bm{y}^0\) and to an invertible matrix \(W \in \mathbb{R}^{d\times d}\) such that \(W \operatorname{cov}(\bm{y}^0) W^\top = I_d\).
Then, define the random variable $\bm{y}$ by the affine transformation $\bm{y} = W(\bm{y}^0 - \mathbb{E} \bm{y}^0)$.
The distribution of $\bm{y}$ is in isotropic position: it has mean \(0\) and covariance matrix \(I_d\).
Furthermore, \cref{lemma:isotropic-position-model} shows that $\bm{y}$ is a mixture of Gaussians in which the components are affine transformed versions of the original components, and that the mixture continues to satisfy mean separation and mean colinearity.
Then, the conditions of \cref{thm:parallel-pancakes-main} are satisfied. Therefore, if the original input samples are $y^0_1, ..., y^0_n$, we can simply run that algorithm on input samples $W(y^0_1 - \mathbb{E} y^0)$, ..., $W(y^0_n - \mathbb{E} y^0)$\footnote{It is straightforward that if sample $y^0_i$ comes from the $i$-th component in the original mixture then $W(y^0_i - \mathbb{E} y^0)$ continues to come from the $i$-th component in the affine transformed mixture.}.

We define now some variables used to state \cref{lemma:isotropic-position-model}. Recall that $\bm{y} = W(\bm{y}^0 - \mathbb{E} \bm y^0)$. Define $\mu_i = W(\mu_i^0 - \mathbb{E} \bm y^0)$, $\Sigma = W \Sigma^0 W^\top$, and $u = \frac{W u^0}{\|W u^0\|v}$. Also define random variables $\bm{\mu} = W (\bm{\mu}^0 - \mathbb{E} \bm y^0)$ and $\bm{w} = W \bm{w}^0$.

\begin{lemma}[Model after isotropic position transformation]
\label{lemma:isotropic-position-model}
The random variable \(\bm{y}\) is distributed according to a mixture of \(k\) Gaussian distributions \(N(\mu_i, \Sigma)\) with mixing weights \(p_i\) for \(i=1,...,k\), where \(\Sigma\) is positive definite.
Alternatively, we specify the model as \(\bm{y} = \bm{\mu} + \bm{w}\), with \(\bm{\mu}\) and \(\bm{w}\) independent of each other.
Furthermore, for all \(i \neq j\) we have mean separation
\[\left\|\Sigma^{-1/2} (\mu_i - \mu_j)\right\|^2 \geq C_{sep} \log \pmin^{-1}\]
and for all \(i\) we have mean colinearity
\[\mu_i = \langle \mu_i, u\rangle u.\]
\end{lemma}

\begin{proof}
Recall that \(\bm{y}^0 = \bm{\mu}^0 + \bm{w}^0\). Then $W(\bm{y}^0 - \mathbb{E} \bm y^0) = W(\bm{\mu}^0 - \mathbb{E} \bm y^0) + W\bm{w}^0$, so $\bm{y} = \bm{\mu} + \bm{w}$.
Also note that \(\bm{\mu}\) takes value \(W(\mu_i^0-\mathbb{E} \bm y^0)=\mu_i\) with probability \(p_i\) and \(\bm{w} \sim N(0, W \Sigma^0 W^\top) = N(0, \Sigma)\).
Therefore, \(\bm{y}\) is distributed according to a mixture of \(k\) Gaussian distributions \(N(\mu_i, \Sigma)\) with mixing weights \(p_i\).

To show that \(\Sigma\) is positive definite, we note that for any vector \(v \in\mathbb{R}^d\) with \(v \neq 0\) we have that
\[v^\top \Sigma v = v^\top W \Sigma^0 W^\top v = (W^\top v)^\top \Sigma^0 W^\top v > 0,\]
where we used that \(W^\top v \neq 0\) because \(W\) is invertible, after which we used that \(\Sigma^0\) is positive definite.

We prove now mean colinearity and mean separation. We start with mean colinearity. Recall that $\mu_i^0 = \mu_{base}^0 + \langle \mu_i^0 - \mu_{base}^0, u^0\rangle u^0$. Then, using that $\mathbb{E} \bm y^0 = \mathbb{E} \bm \mu^0$,
\[\mathbb{E} \bm y^0 = \mu_{base}^0 + \langle \mathbb{E} \bm y^0 - \mu_{base}^0, u^0\rangle u^0,\]
so
\[\mu_i = W(\mu_i^0 - \mathbb{E} \bm y^0) = W(\langle \mu_i^0 - \mathbb{E} \bm{y}^0, u^0 \rangle u^0) = \langle \mu_i^0 - \mathbb{E} \bm{y}^0, u^0\rangle W u^0.\]
Then, using that \(u = \frac{W u^0}{\|W u^0\|}\),
\begin{align*}
\langle \mu_i, u\rangle u
&= \left\langle \langle \mu_i^0 - \mathbb{E} \bm{y}^0, u^0\rangle W u^0, \frac{W u^0}{\|W u^0\|}\right\rangle \frac{W u^0}{\|W u^0\|}\\
&= \langle \mu_i^0 - \mathbb{E} \bm{y}^0, u^0\rangle \left\langle \frac{W u^0}{\|W u^0\|}, \frac{W u^0}{\|W u^0\|}\right\rangle W u^0\\
&= \langle \mu_i^0 - \mathbb{E} \bm{y}^0, u^0\rangle W u^0\\
&= \mu_i,
\end{align*}
which proves mean colinearity. For mean separation, we have that
\begin{align*}
\left\|\Sigma^{-1/2} (\mu_i - \mu_j)\right\|^2
&= \left\| (W \Sigma^0 W^\top)^{-1/2} W (\mu_i^0 - \mu_j^0) \right\|^2\\
&= (\mu_i^0 - \mu_j^0)^\top W^\top (W \Sigma^0 W^\top)^{-1} W (\mu_i^0 - \mu_j^0)\\
&= (\mu_i^0 - \mu_j^0)^\top W^\top (W^\top)^{-1} (\Sigma^0)^{-1} W^{-1} W (\mu_i^0 - \mu_j^0))\\
&= (\mu_i^0 - \mu_j^0)^\top (\Sigma^0)^{-1} (\mu_i^0 - \mu_j^0)\\
&= \left\|(\Sigma^0)^{-1/2} (\mu_i^0 - \mu_j^0)\right\|^2\\
&\geq C_{sep} \log \pmin^{-1}.
\end{align*}
\end{proof}

\begin{lemma}
\label{lemma:isotropic-position-sigma-sq}
\[\sigma^2 = \frac{(u^0)^\top \operatorname{cov}(\bm{y}^0)^{-1} u^0}{(u^0)^\top (\Sigma^0)^{-1} u^0} = u^\top \Sigma u.\]
\end{lemma}

\begin{proof}
For the purposes of this proof, we define $\sigma^2 = u^\top \Sigma u$ as in \cref{sec:isotropic-position-properties} and prove that it also matches the definition in this section.

We have, as in the proof of \cref{lemma:isotropic-position-model}, that
\[\|(\Sigma^0)^{-1/2}(\mu_i^0 - \mu_j^0)\|^2 = \|\Sigma^{-1/2}(\mu_i - \mu_j)\|^2.\]
For the left-hand side, we have that  
\begin{align*}
\|(\Sigma^0)^{-1/2}(\mu_i^0 - \mu_j^0)\|^2
&= \|(\Sigma^0)^{-1/2} \langle \mu_i^0 - \mu_j^0, u^0\rangle u^0\|^2\\
&= \|(\Sigma^0)^{-1/2} u^0\|^2 \cdot \langle \mu_i^0 - \mu_j^0, u^0\rangle^2.
\end{align*}
For the right-hand side, using from \cref{lemma:isotropic-position-covariance} that $\Sigma = I_d - (1-\sigma^2) uu^\top$, we have that 
\begin{align*}
\|\Sigma^{-1/2}(\mu_i - \mu_j)\|^2
&= \left\| \left(I_d - (1-\sigma^2) uu^\top\right)^{-1/2} (\mu_i - \mu_j)) \right\|^2\\
&= \left\| \left(I_d - (1-\sigma^2) \frac{(Wu^0)(Wu^0)^\top}{\|Wu^0\|^2}\right)^{-1/2} \langle \mu_i^0 - \mu_j^0, u^0\rangle W u^0 \right\|^2\\
&= \left\|\left(I_d + \left(\frac{1}{\sigma} - 1\right) \frac{(Wu^0)(Wu^0)^\top}{\|Wu^0\|^2}\right)W u^0\right\|^2 \cdot \langle \mu_i^0 - \mu_j^0, u^0\rangle^2\\
&= \frac{1}{\sigma^2} \cdot \|W u^0\|^2 \cdot \langle \mu_i^0 - \mu_j^0, u^0\rangle^2.
\end{align*}
Therefore
\[\|(\Sigma^0)^{-1/2} u^0\|^2 \cdot \langle \mu_i^0 - \mu_j^0, u^0\rangle^2 = \frac{1}{\sigma^2} \cdot \|W u^0\|^2 \cdot \langle \mu_i^0 - \mu_j^0, u^0\rangle^2,\]
so
\[\sigma^2 = \frac{\|W u^0\|^2}{\|(\Sigma^0)^{-1/2} u^0\|^2} = \frac{(u^0)^\top \operatorname{cov}(\bm{y}^0)^{-1} u^0}{(u^0)^\top (\Sigma^0)^{-1} u^0},\]
where we used that, by \cref{lemma:finite-sample-isotropic-matrix-orthogonal}, $W = Q \operatorname{cov}(\bm{y}^0)^{-1/2}$ for an orthogonal matrix $Q$, so $\|W u^0\| = \|\operatorname{cov}(\bm{y}^0)^{-1/2} u^0\|$.
\end{proof}

\subsection{Finite sample isotropic position transformation}
\label{sec:finitesample-isotropic-transformation}

Without access to \(\mathbb{E} \bm{y}^0\) and to $W$, we apply the isotropic position transformation with \(\hat{\mathbb{E}} \bm{y}^0\) and some matrix \(\hat{W} \in \mathbb{R}^{d \times d}\) defined as follows.
Let the singular value decomposition of \(\widehat{\operatorname{cov}}(\bm y^0)\) be \(\hat{U} \hat{\Lambda} \hat{U}^\top\).
Then define $\hat{W}$ and $W$ as
\begin{equation}
\hat{W} = (\hat{U}^\top\widehat{\operatorname{cov}}(\bm y^0)\hat{U})^{-1/2} \hat{U}^\top, \quad W = (\hat{W}\operatorname{cov}(\bm y^0)\hat{W}^\top)^{-1/2} \hat{W}.
\end{equation}
This choice is analogous to that in Appendix C in \cite{MR3385380-Hsu13}.
By \cref{lemma:finite-sample-isotropic-matrix-facts}, we have that $\hat{W} \widehat{\operatorname{cov}}(\bm{y}^0)\hat{W}^\top = I_d$ and $W \operatorname{cov}(\bm{y}^0) W^\top = I_d$.
Hence, $\hat{W}$ corresponds to an isotropic position transformation for the empirical covariance matrix, and $W$ to one for the population covariance matrix.
In our algorithm, we will apply the approximate isotropic position transformation to input samples $y_1^0, ..., y_n^0$ as $\hat{W}(y_1^0 - \hat{\mathbb{E}} y^0), ..., \hat{W}(y_n^0 - \hat{\mathbb{E}} y^0)$.

\subsection{Finite sample bounds}
\label{sec:finitesample_col}

\cref{lemma:finite-moments-change} gives a sum-of-squares proof that the empirical moments of the mixture with approximate isotropic position transformation are close to the population moments of the mixture with exact isotropic position transformation.
This lemma is supported by \cref{lemma:finite-moments-change-mean} and \cref{lemma:finite-moments-change-covariance}, which prove that the moments do not change much due to the use of $\hat{\mathbb{E}} \bm{y}^0$ and $\hat{W}$, respectively.

Additionally, for arbitrary unit vectors $v$, \cref{lemma:finite-sample-direction-change} proves that $\langle \hat{W}(\mu_i^0 - \mu_j^0), v\rangle$ is close to $\langle W(\mu_i^0 - \mu_j^0), v\rangle$ and \cref{lemma:finite-sample-variance-change} proves that $v^\top \hat{W} (\Sigma^0)^{1/2}$ is close to $v^\top W (\Sigma^0)^{1/2}$. These facts are used in the proof of \cref{thm:colinear-main} to argue that the clustering algorithm is correct.

\begin{lemma}[Closeness of empirical approximate isotropic position moments and population exact isotropic position moments]
\label{lemma:finite-moments-change}
Let $\eta < 0.001$. For
\[n \geq \left(\|\operatorname{cov}(\bm{y}^0)\| \cdot \|\operatorname{cov}(\bm{y}^0)^{-1}\|\right)^{O(1)} \cdot (\pmin^{-1} d)^{O(t)} \eta^{-2} \epsilon^{-1},\]
with probability $1-\epsilon$, 
\[\left\{\|v\|^2=1\right\} \sststile{2t}{v} \left\{\hat{\mathbb{E}}\langle \hat{W} (\bm y^0 - \hat{\mathbb{E}} \bm y^0), v\rangle^{2t} \leq (1+\eta) \cdot \mathbb{E}\langle W (\bm y^0 - \mathbb{E} \bm y^0), v\rangle^{2t} + \eta\right\},\]
\[\left\{\|v\|^2=1\right\} \sststile{2t}{v} \left\{\hat{\mathbb{E}}\langle \hat{W} (\bm y^0 - \hat{\mathbb{E}} \bm y^0), v\rangle^{2t} \geq (1-\eta) \cdot \mathbb{E}\langle W (\bm y^0 - \mathbb{E} \bm y^0), v\rangle^{2t} - \eta\right\}.\]
\end{lemma}
\begin{proof}
Select $n$ such that the results of \cref{lemma:finite-moments-change-mean}, \cref{lemma:finite-moments-change-covariance}, and \cref{lemma:finite-moments-change-empirical} hold each with probability $1-\epsilon/3$. 
Then, overall, all three results hold with probability $1-\epsilon$.
Then we have with probability $1-\epsilon$ that 
\begin{align*}
&\hat{\mathbb{E}}\langle \hat{W} (\bm y^0 - \hat{\mathbb{E}} \bm y^0), v\rangle^{2t}\\
&\quad \leq (1+\eta/10) \hat{\mathbb{E}}\langle \hat{W} (\bm y^0 - \mathbb{E} \bm y^0), v\rangle^{2t} + \eta/10\\
&\quad \leq (1+\eta/10) \left((1+\eta/10) \hat{\mathbb{E}}\langle W (\bm y^0 - \mathbb{E} \bm y^0), v\rangle^{2t} + \eta/10 \right) + \eta/10\\
&\quad \leq (1+\eta/10) \left((1+\eta/10) \left( (1+\eta/10) \mathbb{E}\langle W (\bm y^0 - \mathbb{E} \bm y^0), v\rangle^{2t} + \eta/10 \right) + \eta/10 \right) + \eta/10\\
&\quad \leq (1+\eta) \mathbb{E}\langle W (\bm y^0 - \mathbb{E} \bm y^0), v\rangle^{2t} + \eta.
\end{align*}
and 
\begin{align*}
&\hat{\mathbb{E}}\langle \hat{W} (\bm y^0 - \hat{\mathbb{E}} \bm y^0), v\rangle^{2t}\\
&\quad \geq (1-\eta/10) \hat{\mathbb{E}}\langle \hat{W} (\bm y^0 - \mathbb{E} \bm y^0), v\rangle^{2t} - \eta/10\\
&\quad \geq (1-\eta/10) \left((1-\eta/10) \hat{\mathbb{E}}\langle W (\bm y^0 - \mathbb{E} \bm y^0), v\rangle^{2t} - \eta/10 \right) - \eta/10\\
&\quad \geq (1-\eta/10) \left((1-\eta/10) \left( (1-\eta/10) \mathbb{E}\langle W (\bm y^0 - \mathbb{E} \bm y^0), v\rangle^{2t} - \eta/10 \right) - \eta/10 \right) - \eta/10\\
&\quad \geq (1-\eta) \mathbb{E}\langle W (\bm y^0 - \mathbb{E} \bm y^0), v\rangle^{2t} - \eta.
\end{align*}

\end{proof}

\begin{lemma}%
\label{lemma:finite-moments-change-mean}
Let $\eta < 0.001$. For
\[n \geq kd \log^2(d/\epsilon) \cdot \left(\frac{t \cdot \|\operatorname{cov}(\bm{y}^0)\| \cdot \|\operatorname{cov}(\bm{y}^0)^{-1}\|}{\eta}\right)^{O(1)},\]
with probability $1-\epsilon$,
\[\left\{\|v\|^2=1\right\} \sststile{2t}{v} \left\{\hat{\mathbb{E}}\langle \hat{W} (\bm y^0 - \hat{\mathbb{E}} \bm y^0), v\rangle^{2t} \leq (1+\eta) \cdot \hat{\mathbb{E}}\langle \hat{W} (\bm y^0 - \mathbb{E} \bm y^0), v\rangle^{2t} + \eta\right\},\]
\[\left\{\|v\|^2=1\right\} \sststile{2t}{v} \left\{\hat{\mathbb{E}}\langle \hat{W} (\bm y^0 - \hat{\mathbb{E}} \bm y^0), v\rangle^{2t} \geq (1-\eta) \cdot \hat{\mathbb{E}}\langle \hat{W} (\bm y^0 - \mathbb{E} \bm y^0), v\rangle^{2t} - \eta\right\}.\]
\end{lemma}

\begin{proof}
See \cref{proofs-colinear}.
\end{proof}

\begin{lemma}%
\label{lemma:finite-moments-change-covariance}
Let $\eta < 0.001$. For
\[n \geq kd \log^2(d) \epsilon^{-1} \cdot \left(\frac{t \pmin^{-1} d \cdot \|\operatorname{cov}(\bm{y}^0)\| \cdot \|\operatorname{cov}(\bm{y}^0)^{-1}\|}{\eta}\right)^{O(1)},\]
with probability $1-\epsilon$,
\[\left\{\|v\|^2=1\right\} \sststile{2t}{v} \left\{\hat{\mathbb{E}}\langle \hat{W} (\bm y^0 - \mathbb{E} \bm y^0), v\rangle^{2t} \leq (1+\eta) \cdot \hat{\mathbb{E}}\langle W (\bm y^0 - \mathbb{E} \bm y^0), v\rangle^{2t} + \eta\right\},\]
\[\left\{\|v\|^2=1\right\} \sststile{2t}{v} \left\{\hat{\mathbb{E}}\langle \hat{W} (\bm y^0 - \mathbb{E} \bm y^0), v\rangle^{2t} \geq (1-\eta) \cdot \hat{\mathbb{E}}\langle W (\bm y^0 - \mathbb{E} \bm y^0), v\rangle^{2t} - \eta\right\}.\]
\end{lemma}

\begin{proof}
See \cref{proofs-colinear}.
\end{proof}

\begin{lemma}
\label{lemma:finite-sample-direction-change}
Let $\eta < 0.001$. Let $v \in \mathbb{R}^d$ be a unit vector. For
\[n \geq kd \log^2(d/\epsilon) \cdot \left(\frac{\|\operatorname{cov}(\bm{y}^0)\| \cdot \|\operatorname{cov}(\bm{y}^0)^{-1}\|}{\eta}\right)^{O(1)},\]
with probability $1-\epsilon$, for all $i, j \in [k]$,
\[|\langle W(\mu_i^0 - \mu_j^0), v\rangle - \langle \hat{W}(\mu_i^0 - \mu_j^0), v\rangle| \leq \eta.\]
\end{lemma}
\begin{proof}
See \cref{proofs-colinear}.
\end{proof}

\begin{lemma}
\label{lemma:finite-sample-variance-change}
Let $\eta < 0.001$. Let $v \in \mathbb{R}^d$ be a unit vector. For
\[n \geq kd \log^2(d/\epsilon) \cdot \left(\frac{\pmin^{-1} \cdot \|\operatorname{cov}(\bm{y}^0)\| \cdot \|\operatorname{cov}(\bm{y}^0)^{-1}\|}{\eta}\right)^{O(1)},\]
with probability $1-\epsilon$,
\[\| v^\top W (\Sigma^0)^{1/2} - v^\top \hat{W} (\Sigma^0)^{1/2} \| \leq \eta.\]
\end{lemma}
\begin{proof}
See \cref{proofs-colinear}.
\end{proof}

\subsection{Proof of Theorem~\ref{thm:colinear-main}}
\label{subsec:proof-col}

\begin{proof}[Proof of \cref{thm:colinear-main}]
The first step of the algorithm is to apply the approximate istropic position transformation described in \cref{sec:finitesample-isotropic-transformation} to input samples $y_1^0, ..., y_n^0$.
Then, the new samples are $\hat{W}(y_1^0 - \hat{\mathbb{E}} y^0), ..., \hat{W}(y_n^0 - \hat{\mathbb{E}} y^0)$.
After that, the algorithm is the same as that of \cref{thm:parallel-pancakes-main}.

We now argue that, for $n \geq n_0$, the same guarantees as in \cref{thm:parallel-pancakes-main} hold.
Recall that \cref{thm:parallel-pancakes-main} is composed of two parts:
the algorithm of \cref{thm:pp-finite-sample-high-correlation} which computes a unit vector $\hat{u}$ with correlation $1 - 320\min(\sigma^2, \frac{1}{k})$ with $u$,
and the clustering algorithm of \cref{thm:clustering-main}, which uses a unit vector $\hat{u}$ with such correlation in order to cluster the samples.

For the algorithm of \cref{thm:pp-finite-sample-high-correlation}, we note that by \cref{lemma:finite-moments-change}, for $n \geq n_0$, we have sum-of-squares proofs that, for $t = O(\log \pmin^{-1})$,
\[\hat{\mathbb{E}}\langle \hat{W} (\bm y^0 - \hat{\mathbb{E}} \bm y^0), v\rangle^{2t} \leq \left(1+\frac{\sigma^2}{10000M}\right) \mathbb{E}\langle W (\bm y^0 - \mathbb{E} \bm y^0), v\rangle^{2t} + \frac{\sigma^2}{10000M}\]
and
\[\hat{\mathbb{E}}\langle \hat{W} (\bm y^0 - \hat{\mathbb{E}} \bm y^0), v\rangle^{2t} \geq \left(1-\frac{\sigma^2}{10000M}\right) \mathbb{E}\langle W (\bm y^0 - \mathbb{E} \bm y^0), v\rangle^{2t} - \frac{\sigma^2}{10000M},\]
where $\hat{W} (\bm y^0 - \hat{\mathbb{E}} \bm y^0)$ corresponds to the mixture in approximate isotropic position and $W (\bm y^0 - \mathbb{E} \bm y^0)$ corresponds to the mixture in exact isotropic position. It is easy to verify, similarly to the analysis of the errors in the proof of \cref{thm:pp-finite-sample-high-correlation}, that these errors are tolerated by the algorithm and that it behaves as if the distribution was in exact isotropic position.

For the clustering algorithm of \cref{thm:clustering-main}, the main issue is that the samples are colinear in direction $\frac{\hat{W}u^0}{\|\hat{W}u^0\|}$, but $\hat{u}$ is guaranteed to have large correlation with $\frac{Wu^0}{\|Wu^0\|}$. We prove, nevertheless, that the algorithm has the same guarantees. First, we show that the variance of the one-dimensional components $\hat{u}^\top \hat{W}^\top \Sigma^0 \hat{W} \hat{u}/(2(C+1)\sigma^2)$ is upper bounded by $1$, as required by the algorithm of \cite{MR3826314-Hopkins18}. We have by \cref{lemma:finite-sample-variance-change} that, for $n \geq n_0$, with high probability
\[ | \|\hat{u}^\top W (\Sigma^0)^{1/2}\| - \|\hat{u}^\top \hat{W} (\Sigma^0)^{1/2}\| |
\leq \|\hat{u}^\top W (\Sigma^0)^{1/2} - \hat{u}^\top \hat{W} (\Sigma^0)^{1/2}\| \leq \sigma/100,\]
so using that $\hat{u}^\top W \Sigma^0 W^T \hat{u} \geq \sigma^2$,
\[\hat{u}^\top W \Sigma^0 W^T \hat{u} \geq \frac{1}{2} \hat{u}^\top \hat{W} \Sigma^0 \hat{W}^T \hat{u}.\]
Therefore, using from the proof of \cref{thm:clustering-main} that $(\hat{u}^\top W^\top \Sigma^0 W \hat{u})/(2(C+1)\sigma^2) \leq 0.5$,
\[\frac{\hat{u}^\top \hat{W}^\top \Sigma^0 \hat{W} \hat{u}}{2(C+1)\sigma^2} \leq 2 \frac{\hat{u}^\top W^\top \Sigma^0 W \hat{u}}{2(C+1)\sigma^2} \leq 1.\]
Second, we show that the one-dimensional means $\langle \hat{W}(\mu_i^0 - \hat{\mathbb{E}}\bm{y}^0), \hat{u}\rangle$ have large separation. We have by the proof of \cref{thm:clustering-main} that, for $n \geq n_0$, with high probability
\[\langle W(\mu_i^0 - \mu_j^0), \hat{u}\rangle^2 \geq \frac{C_{sep}}{2} \sigma^2 \log \pmin^{-1}.\] 
We are interested in a similar bound with $W$ changed into $\hat{W}$. By \cref{lemma:finite-sample-direction-change}, for $n \geq n_0$, we have with high probability 
\[|\langle W(\mu_i^0 - \mu_j^0), \hat{u}\rangle - \langle \hat{W}(\mu_i^0 - \mu_j^0), \hat{u}\rangle| \leq \sigma/100,\]
so using that $\langle W(\mu_i^0 - \mu_j^0), \hat{u}\rangle^2 \geq \sigma^2$,
\[\langle \hat{W}(\mu_i^0 - \mu_j^0), \hat{u}\rangle^2 \geq \frac{1}{2} \langle W(\mu_i^0 - \mu_j^0), \hat{u}\rangle^2.\]
Therefore, 
\[\langle \hat{W}(\mu_i^0 - \mu_j^0), \hat{u}\rangle^2 \geq \frac{C_{sep}}{4} \sigma^2 \log \pmin^{-1},\]
so
\[\left(\frac{\langle \hat{W}(\mu_i^0 - \mu_j^0), \hat{u}\rangle}{\sqrt{2(C+1)\sigma^2}}\right)^2 \geq \frac{C_{sep}}{8(C+1)} \log \pmin^{-1}.\]
For $C_{sep}/C$ larger than some universal constant, the separation coefficient $C_{sep}/(8(C+1))$ is large enough for the guarantees of Theorem 5.1 of \cite{MR3826314-Hopkins18} to hold as before with $t = O(\log \pmin^{-1})$.

Then, overall, the same guarantees as in \cref{thm:parallel-pancakes-main} hold.
\end{proof}

\section{Small radius}
\label{sec:small-radius}

\paragraph{Setting.} We consider a mixture of \(k\) Gaussian distributions \(N(\mu_i, \Sigma)\) with mixing weights \(p_i\) for \(i=1,...,k\),
where \(\mu_i \in \mathbb{R}^d\), \(\Sigma \in \mathbb{R}^{d \times d}\) is positive definite, and \(p_i \geq 0\) and \(\sum_{i=1}^k p_i = 1\). 
Let \(\pmin = \min_i p_i\).

The distribution also satisfies mean separation and a small radius condition:
\begin{itemize}
  \tightlist
  \item
    Mean separation:
    for some $C_{sep} > 0$ and for all \(i \neq j\),
    \[\left\|\Sigma^{-1/2} (\mu_i - \mu_j)\right\|^2 \geq C_{sep} \log \pmin^{-1}.\]
  \item
    Small radius:
    for some $R > 0$ and for all $i$,
    \[\left\|\Sigma^{-1/2} \mu_i\right\| \leq R.\]
  \end{itemize}

\begin{theorem}[Small radius algorithm]
\label{thm:small-radius-main}
Consider the Gaussian mixture model defined above, with $C_{sep}$ larger than some universal constant.
Let $n_0 = (\pmin^{-1} d)^{O(R^2 + \log \pmin^{-1})}$.
Given a sample of size $n \geq n_0$ from the mixture, there exists an algorithm that runs in time $n^{O(R^2 + \log \pmin^{-1})}$ and returns with high probability a partition of \([n]\)  into $k$ sets $C_1, ..., C_k$ such that, if the true clustering of the samples is $S_1, ..., S_k$, then there exists a permutation $\pi$ of $[k]$ such that 
\[1 - \frac{1}{n} \sum_{i=1}^k |C_i \cap S_{\pi(i)}| \leq \left(\frac{\pmin}{k}\right)^{O(1)}.\]
\end{theorem}

We introduce some further notation for this section.
Let $\bm{y}$ be distributed according to the mixture.
We specify the model as \(\bm{y} = \bm{\mu} + \bm{w}\), where \(\bm{\mu}\) takes value \(\mu_i\) with probability \(p_i\) and \(\bm{w} \sim N(0, \Sigma)\), with \(\bm{\mu}\) and \(\bm{w}\) independent of each other.

\subsection{Component covariance estimation}

\cref{lemma:small-radius-cov-est} gives a sum-of-squares proof that, for $t = \Omega(R^4)$, the directional moment $\mathbb{E}\langle \bm{y}, v\rangle^{2t}$ approximates the $t$-th power of the variance of the components in direction $v$. This is the main ingredient of the algorithm, and it shows that the $t$-th moment of the distribution identifies within constant factors the covariance matrix of the components. \cref{lemma:bounded-mean-term} is a simple upper bound on the means of the mixture, used in the proof of \cref{lemma:small-radius-cov-est}.

\begin{lemma}[Bounded mean term]
  \label{lemma:bounded-mean-term}
  For $t \geq 1$ integer,
  \[\sststile{2t}{v} \mathbb{E} \langle \bm{\mu}, v \rangle^{2t} \leq R^{2t} (v^\top \Sigma v)^t.\]
\end{lemma}
\begin{proof}
\[\sststile{2t}{v} \mathbb{E} \langle \bm{\mu}, v\rangle^{2t}
= \mathbb{E} \langle \Sigma^{1/2}\Sigma^{-1/2}\bm{\mu}, v\rangle^{2t}
= \mathbb{E} \langle \Sigma^{-1/2}\bm{\mu}, \Sigma^{1/2} v\rangle^{2t}
\leq \mathbb{E} \|\Sigma^{-1/2} \bm{\mu}\|^{2t} \|\Sigma^{1/2} v\|^{2t},\]
where in the inequality we used \cref{lemma:sos-cs}. The conclusion follows by noting that $\|\Sigma^{-1/2} \bm{\mu}\| \leq R$ and $\|\Sigma^{1/2} v\|^2 = v^\top \Sigma v$.
\end{proof}

\begin{lemma}[Small radius component covariance estimation]
\label{lemma:small-radius-cov-est}
For $t \geq 4R^2$ integer,
\[\sststile{2t}{v} \frac{1}{4^t} (v^\top \Sigma v)^t \leq \frac{1}{t^t} \mathbb{E}\langle \bm{y}, v\rangle^{2t} \leq 4^t (v^\top \Sigma v)^t.\]
\end{lemma}
\begin{proof}
For the upper bound, by \cref{lemma:moment-upper-bound-weak} and \cref{lemma:bounded-mean-term},
\begin{align*}
\sststile{2t}{v} \mathbb{E} \langle \bm{y}, v \rangle^{2t}
&\leq 2^{2t-1} \mathbb{E} \langle \bm\mu, v\rangle^{2t} + 2^{2t-1} (v^\top \Sigma v)^t t^t\\
&\leq 2^{2t-1} (R^{2t} + t^t) (v^\top \Sigma v)^t\\
&\leq 4^t t^t (v^\top \Sigma v)^t.
\end{align*}
For the lower bound, by \cref{lemma:moment-lower-bound-weak} and \cref{lemma:bounded-mean-term},
\begin{align*}
\sststile{2t}{v} \mathbb{E} \langle \bm{y}, v \rangle^{2t}
&\geq \mathbb{E} \langle \bm{\mu}, v\rangle^{2t} + (v^\top \Sigma v)^t \frac{t^t}{2^t}\\
&\geq \left(- R^{2t} + \frac{1}{2^t} t^t\right) (v^\top \Sigma v)^t\\
&\geq \frac{1}{4^t} t^t (v^\top \Sigma v)^t.
\end{align*}
\end{proof}

We note that the term $\frac{1}{t^t} \mathbb{E} \langle \bm{y}, v\rangle^{2t}$ can be rewritten as 
\[\frac{1}{t^t} \mathbb{E} \langle \bm{y}, v\rangle^{2t} = \left\langle v^{\otimes t}, \frac{1}{t^t} \mathbb{E} (\bm{y}\bm{y}^\top)^{\otimes t} v^{\otimes t} \right\rangle = \left\langle \frac{1}{t^t} \mathbb{E} (\bm{y}\bm{y}^\top)^{\otimes t}, (vv^\top)^{\otimes t} \right\rangle.\]

\subsection{Finite sample bounds}

\cref{lemma:small-radius-finite-sample} gives a sum-of-squares proof that the empirical moments of the distribution are close to the population moments. We defer the proof to the appendix.

\begin{lemma}[Closeness of empirical moments and population moments]
\label{lemma:small-radius-finite-sample}
For \(n \geq (\pmin^{-1} d)^{O(t)} \eta^{-2} \epsilon^{-1}\), with probability $1-\epsilon$,
\[\sststile{2t}{v} (1-\eta)\cdot \mathbb{E}\langle \bm{y}, v\rangle^{2t} \leq \hat{\mathbb{E}}\langle \bm{y}, v\rangle^{2t} \leq (1+\eta)\cdot \mathbb{E}\langle \bm{y}, v\rangle^{2t}.\]
\end{lemma}
\begin{proof}
See \cref{proofs-small-radius}.
\end{proof}

\subsection{Proof of Theorem~\ref{thm:small-radius-main}}

If the covariance matrix of the components were known, we could apply an affine transformation to the samples and change the distribution into a mixture of spherical Gaussians.
After that we could simply apply an algorithm for clustering mixtures of spherical Gaussians.

It might look like the covariance matrix approximation of \cref{lemma:small-radius-cov-est} could be used to design a sum-of-squares program that identifies this covariance matrix.
However, because \cref{lemma:small-radius-cov-est} only gives an approximation in each direction for the $t$-th power of the variance of the components, and because it is non-trivial to take $t$-th roots in sum-of-squares proofs, we found it challenging to obtain a low-degree sum-of-squares proof of identifiability for the covariance matrix.

Instead, we observe that the sum-of-squares algorithm of \cite{MR3826314-Hopkins18} for clustering mixtures of spherical Gaussians only uses as axioms upper bounds on the $t$-th moments of the distribution of the components.
It is not difficult to adapt this algorithm to work with the $t$-th power approximations that we obtain from \cref{lemma:small-radius-cov-est}.

\begin{proof}[Proof of \cref{thm:small-radius-main}.]
The algorithm is:
\begin{enumerate}
  \tightlist
  \item Set $t = O(R^2 + \log \pmin^{-1})$ large enough.
  \item Estimate $D = \frac{1}{t^t} \hat{\mathbb{E}} (\bm{y}\bm{y}^\top)^{\otimes t}$.
  \item Apply the algorithm from Theorem 5.1 of \cite{MR3826314-Hopkins18}, but with the following moment constraint in the set of axioms instead of the original moment constraint:
  \[\forall v \in \mathbb{R}^d, \frac{1}{\alpha n} \sum_{i=1}^n w_i \langle y_i - \mu, v\rangle^{2t} \leq 4(8t)^{t} \langle v^{\otimes t}, D v^{\otimes t}\rangle,\]
  where $\alpha$ is a parameter and $w_i$ and $\mu$ are system variables, as in the original axioms.
  \item Return the clustering that this algorithm computes as an intermediate step.
\end{enumerate}

We now analyze the algorithm.
We first discuss the new constraint.
The universal quantifier over $v \in \mathbb{R}^d$ can be modeled by requiring that there exists a sum-of-squares proof in $v$ of the constraint (see \cite{DBLP:journals/fttcs/FlemingKP19}).
For the random variable $\Sigma^{-1/2}\bm{y}$, which is distributed according to a mixture of spherical Gaussians with covariance matrix $I_d$, it follows by standard arguments (see \cite{MR3826314-Hopkins18}) that, for our choice of $n$, with high probability
\[\sststile{2t}{v} \frac{1}{\alpha n} \sum_{i=1}^n w_i \langle \Sigma^{-1/2} (y_i - \mu), v\rangle^{2t} \leq 2 (2t)^{t} \|v\|^{2t}\]
when $\alpha$ is the fraction of samples coming from one of the components, $w_i$ is $1$ for all samples from that component and $0$ for all other samples, and $\mu$ is the mean of that component.
Then, by a change of variables $v \to \Sigma^{1/2} v$,
\[\sststile{2t}{v} \frac{1}{\alpha n} \sum_{i=1}^n w_i \langle y_i - \mu, v\rangle^{2t} \leq 2 (2t)^{t} (v^\top \Sigma v)^t.\]
We connect now this to $D$. By \cref{lemma:small-radius-finite-sample}, for our choice of $n$, with high probability
\[\sststile{2t}{v} \frac{1}{2} \cdot \mathbb{E}\langle \bm{y}, v\rangle^{2t} \leq t^t \langle v^{\otimes t}, D v^{\otimes t} \rangle \leq 2 \cdot \mathbb{E}\langle \bm{y}, v\rangle^{2t}.\]
By combining this result with \cref{lemma:small-radius-cov-est},
\[\sststile{2t}{v} \frac{1}{2 \cdot 4^t} (v^\top \Sigma v)^t \leq \langle v^{\otimes t}, D v^{\otimes t} \rangle \leq 2 \cdot 4^t \cdot (v^\top \Sigma v)^t.\]
Therefore,
\[\sststile{2t}{v} \frac{1}{\alpha n} \sum_{i=1}^n w_i \langle y_i - \mu, v\rangle^{2t} \leq 4 (8t)^{t} \langle v^{\otimes t}, D v^{\otimes t}\rangle,\]
so the constraint is valid.

To study the guarantees of the algorithm, we show that the new constraint implies a constraint of the form required by the original algorithm, which only includes a term $\|v\|^{2t}$ on the right-hand side.
We use that $\sststile{2t}{v} \langle v^{\otimes t}, D v^{\otimes t}\rangle \leq 2 \cdot 4^t \cdot (v^\top \Sigma v)^t$ in
\[\sststile{2t}{v} \frac{1}{\alpha n} \sum_{i=1}^n w_i \langle y_i - \mu, v\rangle^{2t} \leq 4 (8t)^{t} \langle v^{\otimes t}, D v^{\otimes t}\rangle\]
to obtain 
\[\sststile{2t}{v} \frac{1}{\alpha n} \sum_{i=1}^n w_i \langle y_i - \mu, v\rangle^{2t} \leq 8 (32t)^{t} (v^\top \Sigma v)^t.\]
By a change of variables $v \to \Sigma^{-1/2} v$,
\[\sststile{2t}{v} \frac{1}{\alpha n} \sum_{i=1}^n w_i \langle \Sigma^{-1/2}(y_i - \mu), v\rangle^{2t} \leq 8 (32t)^{t} \|v\|^{2t}.\]
Finally, by dividing both sides by $4 \cdot 16^t$ we obtain 
\[\sststile{2t}{v} \frac{1}{\alpha n} \sum_{i=1}^n w_i \left\langle \frac{1}{2^{1/t} \cdot 4} \Sigma^{-1/2}(y_i - \mu), v\right\rangle^{2t} \leq 2 (2t)^{t} \|v\|^{2t}.\]
Then the algorithm from Theorem 5.1 of \cite{MR3826314-Hopkins18} behaves as if we had samples from $\frac{1}{2^{1/t} \cdot 4} \Sigma^{-1/2} \bm{y}$, which is distributed according to a mixture of well-separated spherical Gaussians with covariance matrix $\frac{1}{4^{1/t} \cdot 16} I_d$.
It is easy then to verify that we inherit the guarantees of the algorithm of \cite{MR3826314-Hopkins18} and, for $t = O(R^2 + \log \pmin^{-1})$ large enough, we return a clustering that satisfies the statement of our theorem. 

The algorithm of \cite{MR3826314-Hopkins18} requires $n \geq (\pmin^{-1})^{O(1)} \cdot d^{O(t)} = (\pmin^{-1})^{O(1)} \cdot d^{O(R^2 + \log \pmin^{-1})}$, so our choice of $n$ is large enough to satisfy this.
The time complexity is dominated by the algorithm of \cite{MR3826314-Hopkins18}, which has a time complexity of $n^{O(t)} = n^{O(R^2 + \log \pmin^{-1})}$.

\end{proof}

\section*{Acknowledgement}
\addcontentsline{toc}{section}{Acknowledgement}

We thank Samuel B. Hopkins for the discussions related to this project.

\phantomsection
\addcontentsline{toc}{section}{References}
\bibliographystyle{amsalpha}
\bibliography{bib/mathreview,bib/dblp,bib/custom,bib/scholar}

\newcommand{\etalchar}[1]{$^{#1}$}
\providecommand{\bysame}{\leavevmode\hbox to3em{\hrulefill}\thinspace}
\providecommand{\MR}{\relax\ifhmode\unskip\space\fi MR }
\providecommand{\MRhref}[2]{%
  \href{http://www.ams.org/mathscinet-getitem?mr=#1}{#2}
}
\providecommand{\href}[2]{#2}
\begin{thebibliography}{DKK{\etalchar{+}}19}

\bibitem[ABH{\etalchar{+}}20]{DBLP:journals/jacm/AshtianiBHLMP20}
Hassan Ashtiani, Shai Ben{-}David, Nicholas J.~A. Harvey, Christopher Liaw,
  Abbas Mehrabian, and Yaniv Plan, \emph{Near-optimal sample complexity bounds
  for robust learning of gaussian mixtures via compression schemes}, J. {ACM}
  \textbf{67} (2020), no.~6, 32:1--32:42.

\bibitem[BDH{\etalchar{+}}20]{MR4232031-DiakonikolasHopkinsKothari2020}
Ainesh Bakshi, Ilias Diakonikolas, Samuel~B. Hopkins, Daniel Kane, Sushrut
  Karmalkar, and Pravesh~K. Kothari, \emph{Outlier-robust clustering of
  {G}aussians and other non-spherical mixtures}, 2020 {IEEE} 61st {A}nnual
  {S}ymposium on {F}oundations of {C}omputer {S}cience, IEEE Computer Soc., Los
  Alamitos, CA, [2020] \copyright 2020, pp.~149--159. \MR{4232031}

\bibitem[BDJ{\etalchar{+}}20]{DBLP:journals/corr/abs-2012-02119}
Ainesh Bakshi, Ilias Diakonikolas, He~Jia, Daniel~M. Kane, Pravesh~K. Kothari,
  and Santosh~S. Vempala, \emph{Robustly learning mixtures of k arbitrary
  gaussians}, CoRR \textbf{abs/2012.02119} (2020).

\bibitem[BK20]{DBLP:journals/corr/abs-2005-02970}
Ainesh Bakshi and Pravesh Kothari, \emph{Outlier-robust clustering of
  non-spherical mixtures}, CoRR \textbf{abs/2005.02970} (2020).

\bibitem[BRST21]{bruna2021continuous}
Joan Bruna, Oded Regev, Min~Jae Song, and Yi~Tang, \emph{Continuous lwe},
  Proceedings of the 53rd Annual ACM SIGACT Symposium on Theory of Computing,
  2021, pp.~694--707.

\bibitem[BS10]{DBLP:conf/colt/BelkinS10}
Mikhail Belkin and Kaushik Sinha, \emph{Toward learning gaussian mixtures with
  arbitrary separation}, {COLT}, Omnipress, 2010, pp.~407--419.

\bibitem[BS14]{MR3727623-BarakSteurerICM14}
Boaz Barak and David Steurer, \emph{Sum-of-squares proofs and the quest toward
  optimal algorithms}, Proceedings of the {I}nternational {C}ongress of
  {M}athematicians---{S}eoul 2014. {V}ol. {IV}, Kyung Moon Sa, Seoul, 2014,
  pp.~509--533. \MR{3727623}

\bibitem[BV08]{DBLP:conf/focs/BrubakerV08}
S.~Charles Brubaker and Santosh Vempala, \emph{Isotropic {PCA} and
  affine-invariant clustering}, {FOCS}, {IEEE} Computer Society, 2008,
  pp.~551--560.

\bibitem[Das99]{DBLP:conf/focs/Dasgupta99}
Sanjoy Dasgupta, \emph{Learning mixtures of gaussians}, {FOCS}, {IEEE} Computer
  Society, 1999, pp.~634--644.

\bibitem[DHKK20]{DBLP:journals/corr/abs-2005-06417}
Ilias Diakonikolas, Samuel~B. Hopkins, Daniel Kane, and Sushrut Karmalkar,
  \emph{Robustly learning any clusterable mixture of gaussians}, CoRR
  \textbf{abs/2005.06417} (2020).

\bibitem[DK20]{MR4232034-Diakonikolas20}
Ilias Diakonikolas and Daniel~M. Kane, \emph{Small covers for near-zero sets of
  polynomials and learning latent variable models}, 2020 {IEEE} 61st {A}nnual
  {S}ymposium on {F}oundations of {C}omputer {S}cience, IEEE Computer Soc., Los
  Alamitos, CA, [2020] \copyright 2020, pp.~184--195. \MR{4232034}

\bibitem[DK22]{diakonikolas2022non}
Ilias Diakonikolas and Daniel Kane, \emph{Non-gaussian component analysis via
  lattice basis reduction}, Conference on Learning Theory, PMLR, 2022,
  pp.~4535--4547.

\bibitem[DKK{\etalchar{+}}19]{MR3945261-DiakonikolasKaneJournal19}
Ilias Diakonikolas, Gautam Kamath, Daniel Kane, Jerry Li, Ankur Moitra, and
  Alistair Stewart, \emph{Robust estimators in high-dimensions without the
  computational intractability}, SIAM J. Comput. \textbf{48} (2019), no.~2,
  742--864. \MR{3945261}

\bibitem[DKS17]{MR3734219-DiakonikolasKane17}
Ilias Diakonikolas, Daniel~M. Kane, and Alistair Stewart, \emph{Statistical
  query lower bounds for robust estimation of high-dimensional {G}aussians and
  {G}aussian mixtures (extended abstract)}, 58th {A}nnual {IEEE} {S}ymposium on
  {F}oundations of {C}omputer {S}cience---{FOCS} 2017, IEEE Computer Soc., Los
  Alamitos, CA, 2017, pp.~73--84. \MR{3734219}

\bibitem[DKS18]{MR38263160-Diakonikolas18}
\bysame, \emph{List-decodable robust mean estimation and learning mixtures of
  spherical {G}aussians}, S{TOC}'18---{P}roceedings of the 50th {A}nnual {ACM}
  {SIGACT} {S}ymposium on {T}heory of {C}omputing, ACM, New York, 2018,
  pp.~1047--1060. \MR{3826316}

\bibitem[DL67]{delves1967numerical}
LM~Delves and JN~Lyness, \emph{A numerical method for locating the zeros of an
  analytic function}, Mathematics of computation \textbf{21} (1967), no.~100,
  543--560.

\bibitem[FKP19]{DBLP:journals/fttcs/FlemingKP19}
Noah Fleming, Pravesh Kothari, and Toniann Pitassi, \emph{Semialgebraic proofs
  and efficient algorithm design}, Found. Trends Theor. Comput. Sci.
  \textbf{14} (2019), no.~1-2, 1--221.

\bibitem[GHK15]{DBLP:conf/stoc/GeHK15}
Rong Ge, Qingqing Huang, and Sham~M. Kakade, \emph{Learning mixtures of
  gaussians in high dimensions}, {STOC}, {ACM}, 2015, pp.~761--770.

\bibitem[GVV22]{gupte2022continuous}
Aparna Gupte, Neekon Vafa, and Vinod Vaikuntanathan, \emph{Continuous lwe is as
  hard as lwe \& applications to learning gaussian mixtures}, arXiv preprint
  arXiv:2204.02550 (2022).

\bibitem[HK13]{MR3385380-Hsu13}
Daniel Hsu and Sham~M. Kakade, \emph{Learning mixtures of spherical
  {G}aussians: moment methods and spectral decompositions},
  I{TCS}'13---{P}roceedings of the 2013 {ACM} {C}onference on {I}nnovations in
  {T}heoretical {C}omputer {S}cience, ACM, New York, 2013, pp.~11--19.
  \MR{3385380}

\bibitem[HL18]{MR3826314-Hopkins18}
Samuel~B. Hopkins and Jerry Li, \emph{Mixture models, robustness, and sum of
  squares proofs}, S{TOC}'18---{P}roceedings of the 50th {A}nnual {ACM}
  {SIGACT} {S}ymposium on {T}heory of {C}omputing, ACM, New York, 2018,
  pp.~1021--1034. \MR{3826314}

\bibitem[KMV10]{DBLP:conf/stoc/KalaiMV10}
Adam~Tauman Kalai, Ankur Moitra, and Gregory Valiant, \emph{Efficiently
  learning mixtures of two gaussians}, {STOC}, {ACM}, 2010, pp.~553--562.

\bibitem[KS17]{DBLP:journals/corr/abs-1711-11581}
Pravesh~K. Kothari and David Steurer, \emph{Outlier-robust moment-estimation
  via sum-of-squares}, CoRR \textbf{abs/1711.11581} (2017).

\bibitem[KSS18]{MR3826315-KothariSS18}
Pravesh~K. Kothari, Jacob Steinhardt, and David Steurer, \emph{Robust moment
  estimation and improved clustering via sum of squares},
  S{TOC}'18---{P}roceedings of the 50th {A}nnual {ACM} {SIGACT} {S}ymposium on
  {T}heory of {C}omputing, ACM, New York, 2018, pp.~1035--1046. \MR{3826315}

\bibitem[LL22]{MR4490076-LiuLi22}
Allen Liu and Jerry Li, \emph{Clustering mixtures with almost optimal
  separation in polynomial time}, S{TOC} '22---{P}roceedings of the 54th
  {A}nnual {ACM} {SIGACT} {S}ymposium on {T}heory of {C}omputing, ACM, New
  York, [2022] \copyright 2022, pp.~1248--1261. \MR{4490076}

\bibitem[LLL82]{MR682664-Lenstra82}
A.~K. Lenstra, H.~W. Lenstra, Jr., and L.~Lov\'asz, \emph{Factoring polynomials
  with rational coefficients}, Math. Ann. \textbf{261} (1982), no.~4, 515--534.
  \MR{682664}

\bibitem[LM21]{DBLP:conf/stoc/LiuM21}
Allen Liu and Ankur Moitra, \emph{Settling the robust learnability of mixtures
  of gaussians}, {STOC} '21: 53rd Annual {ACM} {SIGACT} Symposium on Theory of
  Computing, Virtual Event, Italy, June 21-25, 2021, {ACM}, 2021, pp.~518--531.

\bibitem[LR12]{laurent2012approach}
Monique Laurent and Philipp Rostalski, \emph{The approach of moments for
  polynomial equations}, Handbook on Semidefinite, Conic and Polynomial
  Optimization, Springer, 2012, pp.~25--60.

\bibitem[Mea92]{mead1992newton}
DG~Mead, \emph{Newton's identities}, The American mathematical monthly
  \textbf{99} (1992), no.~8, 749--751.

\bibitem[MV10]{DBLP:conf/focs/MoitraV10}
Ankur Moitra and Gregory Valiant, \emph{Settling the polynomial learnability of
  mixtures of gaussians}, {FOCS}, {IEEE} Computer Society, 2010, pp.~93--102.

\bibitem[Pea94]{10.2307/90667}
Karl Pearson, \emph{Contributions to the mathematical theory of evolution},
  Philosophical Transactions of the Royal Society of London. A \textbf{185}
  (1894), 71--110.

\bibitem[Pol02]{pollard2002user}
David Pollard, \emph{A user's guide to measure theoretic probability}, no.~8,
  Cambridge University Press, 2002.

\bibitem[RSS18]{MR3966537-RaghavendraSchrammSteurerICM18}
Prasad Raghavendra, Tselil Schramm, and David Steurer, \emph{High dimensional
  estimation via sum-of-squares proofs}, Proceedings of the {I}nternational
  {C}ongress of {M}athematicians---{R}io de {J}aneiro 2018. {V}ol. {IV}.
  {I}nvited lectures, World Sci. Publ., Hackensack, NJ, 2018, pp.~3389--3423.
  \MR{3966537}

\bibitem[VW02]{DBLP:conf/focs/VempalaW02}
Santosh Vempala and Grant Wang, \emph{A spectral algorithm for learning
  mixtures of distributions}, {FOCS}, {IEEE} Computer Society, 2002, p.~113.

\bibitem[ZSWB22]{zadik2022lattice}
Ilias Zadik, Min~Jae Song, Alexander~S Wein, and Joan Bruna,
  \emph{Lattice-based methods surpass sum-of-squares in clustering}, Conference
  on Learning Theory, PMLR, 2022, pp.~1247--1248.

\end{thebibliography}

\appendix

\section{Appendix}
\label{sec:appendix}

\subsection{Moment matching}

\cref{lemma:moment-matching} shows that there exists an equally-weighted discrete distribution on $k$ points whose first $k$ moments are equal to those of an equally-weighted mixture of $k$ Gaussian distributions with variance $1$.

We give the statement of \cref{lemma:moment-matching} below.
Then we state an auxiliarly definition and lemma, and finally we prove \cref{lemma:moment-matching}.

\begin{lemma}[Moment matching]
\label{lemma:moment-matching}
For $k \in \mathbb{N}$, there exist $x_1, ..., x_k \in \mathbb{R}$ and $y_1, ..., y_k \in \mathbb{R}$ such that the first $k$ moments of the equally-weighted discrete distribution on $\{x_1, ..., x_k\}$ are equal to the first $k$ moments of the equally-weighted mixture of Gaussian distributions $N(y_1, 1)$, ..., $N(y_k, 1)$. Furthermore, $|x_1|, ..., |x_k|$ and $|y_1|, ..., |y_k|$ are upper bounded by $2^{\operatorname{poly}(k)}$.
\end{lemma}

\begin{definition}[Newton's identities moment matrix]
\label{def:moment-matrix}
Given $m_1, ..., m_k \in \mathbb{R}$, define the moment matrix $\mathcal{M}(m_1, ..., m_k) \in \mathbb{R}^{k\times k}$ as follows.
Let $p_0 = k$ and $p_t = k \cdot m_t$ for all $t \in [k]$.
Also let $e_0 = 1$ and
\[e_t = \frac{1}{k} \sum_{i=1}^t (-1)^{i-1} e_{t-i} p_i, \quad \forall t \in [k],\]
\[p_t = \sum_{i=t-k}^{i-1} (-1)^{t+i-1} e_{t-i} p_i, \quad \forall t > k.\]
Then $\mathcal{M}(m_1, ..., m_k)$ has entries $\mathcal{M}(m_1, ..., m_k)_{i,j}=\frac{1}{k} p_{i+j-2}$. 
\end{definition}

\begin{lemma}[Newton's identities moment matrix condition]
\label{lemma:newton-identities}
Given $m_1, ..., m_k \in \mathbb{R}$, there exist $x_1, ..., x_k \in \mathbb{R}$ such that for all $t=1,...,k$
\[\frac{1}{k} \sum_{i=1}^k x_i^t = m_t\]
if and only if $\mathcal{M}(m_1, ..., m_k)$ is positive semi-definite.
\end{lemma}
\begin{proof}
The following technique has been used before, for example, in \cite{delves1967numerical}.
Let $e_i$ and $p_i$ be as in \cref{def:moment-matrix}.
Let the degree-$k$ monic polynomial $\sum_{i=0}^k (-1)^i e_i x^{k-i}$ have roots ${x_1, ..., x_k \in \mathbb{C}}$.
By expanding $\prod_{i=1}^k (x-x_i)$ and matching coefficients, these roots satisfy for all $t=1,...,k$
\[\sum_{\substack{S \subseteq [k]\\|S|=t}} \prod_{i \in S} x_i = e_t.\]
Then, by Newton's identities \cite{mead1992newton}, for all $t \in \mathbb{N}$, $\sum_{i=1}^k x_i^t = p_t$.
Note that $\frac{1}{k} p_1, ..., \frac{1}{k} p_k$ coincide with $m_1, ..., m_k$.

It remains to consider whether $x_1, ..., x_k$ are real.
It is known \cite{laurent2012approach} that the roots $x_1, ..., x_k$ are real if and only if their $k\times k$ moment matrix is positive semi-definite.
This moment matrix is defined as the matrix $M \in \mathbb{R}^{k \times k}$ with entries $M_{i,j} = \frac{1}{k} \sum_{\ell=1}^k x_\ell^{i+j-2}$, and by the fact that $\sum_{i=1}^k x_i^t = p_t$ and by \cref{def:moment-matrix} it follows that this matrix coincides with $\mathcal{M}(m_1, ..., m_k)$.
This completes the proof.
\end{proof}

\begin{proof}[Proof of \cref{lemma:moment-matching}]
Let $y_i = i$ for all $i \in [k]$.
Denote by $m_{t, \sigma^2}$ the $t$-th moment of the equally-weighted mixture of Guassian distributions $N(y_1, \sigma^2)$, ..., $N(y_k, \sigma^2)$.
Denote $\mathcal{M}_{\sigma^2} = \mathcal{M}(m_{1, \sigma^2}, ..., m_{k,\sigma^2})$.

We start by proving that the minimum eigenvalue of $\mathcal{M}_0$ is lower bounded by $2^{-\operatorname{poly}(k)}$. By the proof of \cref{lemma:newton-identities}, $\mathcal{M}_0$ is the moment matrix of the equally-weighted discrete distribution on $\{y_1, ..., y_k\}$, so 
\[\mathcal{M}_0 = \frac{1}{k} \left[\begin{array}{ccccc}
1 & y_1 & y_1^2 & \cdots & y_1^{k-1}\\
1 & y_2 & y_2^2 & \cdots & y_2^{k-1}\\
\vdots & \vdots & \vdots & \ddots & \vdots\\
1 & y_k & y_k^2 & \cdots & y_k^{k-1}
\end{array}\right]^\top\left[\begin{array}{ccccc}
1 & y_1 & y_1^2 & \cdots & y_1^{k-1}\\
1 & y_2 & y_2^2 & \cdots & y_2^{k-1}\\
\vdots & \vdots & \vdots & \ddots & \vdots\\
1 & y_k & y_k^2 & \cdots & y_k^{k-1}
\end{array}\right].\]
The expression above is a Cholesky decomposition, so $\mathcal{M}_0$ is positive semi-definite. Furthermore, the matrices in the decomposition are Vandermonde matrices with determinant $\prod_{1 \leq i < j \leq k} (y_j - y_i)$, for which a very weak lower bound is $1$. Then the determinant of $\mathcal{M}_0$ is at least $\frac{1}{k^k}$. We also have that the trace of $\mathcal{M}_0$ is $\frac{1}{k} \sum_{t=0}^{k-1} \sum_{i=1}^k y_i^{2t} \leq k^{O(k)}$, which is also an upper bound on the largest eigenvalue of $\mathcal{M}_0$. Then the minimum eigenvalue of $\mathcal{M}_0$ is at least $\frac{1}{k^k} \cdot \frac{1}{(k^{O(k)})^k} = 2^{-\operatorname{poly}(k)}$.

We show now that, for some $\sigma^2 = 2^{-\operatorname{poly}(k)}$ small enough, $\|\mathcal{M}_{\sigma^2} - \mathcal{M}_0\|$ is smaller than the minimum eigenvalue of $\mathcal{M}_0$. This implies that $\mathcal{M}_{\sigma^2}$ is also positive semi-definite. An inspection of $\frac{1}{\sigma^2}(\mathcal{M}_{\sigma^2} - \mathcal{M}_0)$ shows that its entries are polynomials in $y_1, ..., y_k$ and $\sigma^2$ of degree $\operatorname{poly(k)}$ with coefficients bounded in absolute value by $2^{\operatorname{poly}(k)}$, so the entries themselves are bounded in absolute value by $2^{\operatorname{poly}(k)}$. We have that $\|\frac{1}{\sigma^2}(\mathcal{M}_{\sigma^2} - \mathcal{M}_0)\|$ is bounded by $k$ times the maximum absolute value of an entry, so this spectral norm is also bounded by $2^{\operatorname{poly}(k)}$. Then, by choosing $\sigma^2 = 2^{-\operatorname{poly}(k)}$ small enough, we can ensure that $\|\mathcal{M}_{\sigma^2} - \mathcal{M}_0\|$ is smaller than the minimum eigenvalue of $\mathcal{M}_0$.

Therefore $\mathcal{M}_{\sigma^2}$ is positive semi-definite. Then, by \cref{lemma:newton-identities}, there exist $x_1, ..., x_k \in \mathbb{R}$ such that the first $k$ moments of the equally-weighted discrete distribution on $\{x_1, ..., x_k\}$ are equal to the first $k$ moments of the equally-weighted mixture of Gaussian distributions $N(y_1, \sigma^2)$, ..., $N(y_k, \sigma^2)$. We note that $|x_1|, ..., |x_k|$ must be bounded by $\operatorname{poly}(k)$, otherwise the second moments of the discrete distribution would be larger than those of the mixture of Gaussian distributions. Finally, by scaling, the desired moment matching holds for $\frac{1}{\sigma} x_1, ..., \frac{1}{\sigma} x_k$ and $\frac{1}{\sigma} y_1, ..., \frac{1}{\sigma} y_k$. These values clearly satisfy the stated upper bound.
\end{proof}

\subsection{Sum-of-squares lemmas}

We first prove a number of useful SOS facts.

\begin{lemma}[Restatement of Lemma A.1 in \cite{DBLP:journals/corr/abs-1711-11581}]
\label{lemma:sos-triangle-amgm}
For variables $X_1, ..., X_t \in \mathbb{R}$,
\[\sststile{t}{X} X_1 \cdot ... \cdot X_t \leq \frac{1}{t} (X_1^t + ... + X_t^t).\]
\end{lemma}

\begin{lemma}[Restatement of Lemma A.2 in \cite{DBLP:journals/corr/abs-1711-11581}]
\label{lemma:sos-triangle-two}
For variables $A, B \in \mathbb{R}$ and $t \geq 2$ even,
\[\sststile{t}{A, B} (A+B)^{t} \leq 2^{t-1} A^{t} + 2^{t-1} B^{t}.\]
\end{lemma}

\begin{lemma}
\label{lemma:sos-triangle}
For variables $A, B \in \mathbb{R}$ and $\delta > 0$ and $t \geq 2$ even,
\[\sststile{t}{A, B} (A+B)^{t} \leq (1+\delta)^{t-1} A^{t} + \left(1+\frac{1}{\delta}\right)^{t-1} B^{t}.\]
\end{lemma}
\begin{proof}
\begin{align*}
\sststile{t}{A, B} (A+B)^{t}
&= \sum_{s=0}^{t} \binom{t}{s} A^{s} B^{t-s}\\
&= \sum_{s=0}^{t} \binom{t}{s} \left(\delta^{1-s/t} A\right)^s \left(\frac{1}{\delta^{s/t}} B\right)^{t-s}\\
&\stackrel{(1)}{\leq} \sum_{s=0}^{t} \binom{t}{s} \left(\frac{s}{t} \left(\delta^{1-s/t} A\right)^t + \frac{t-s}{s}\left(\frac{1}{\delta^{s/t}} B\right)^t\right)\\
&= \left(\sum_{s=0}^{t} \binom{t}{s} \frac{s}{t} \delta^{t-s}\right) A^t + \left(\sum_{s=0}^{t} \binom{t}{s} \frac{t-s}{t} \frac{1}{\delta^{s}}\right) B^t\\
&\stackrel{(2)}{=} (1+\delta)^{t-1} A^t + \left(1+\frac{1}{\delta}\right)^{t-1} B^t
\end{align*}
where in (1) we used \cref{lemma:sos-triangle-amgm} and in (2) we used the identities 
\[\sum_{s=0}^t \binom{t}{s} \frac{s}{t} x^{t-s} = \sum_{s=0}^{t-1} \binom{t-1}{s} x^{t-1-s} = (1+x)^{t-1}\]
and  
\[\sum_{s=0}^t \binom{t}{s} \frac{t-s}{t} x^{s} = \sum_{s=0}^{t-1} \binom{t-1}{s} x^{s} = (1+x)^{t-1}.\]
\end{proof}

\begin{lemma}
\label{lemma:sos-square-power-lower-bound}
For variable $X \in \mathbb{R}$ and $t \geq 0$ integer,
\[\{X \geq 0\} \sststile{t}{X} \{X^t \geq 0\}.\]
\end{lemma}
\begin{proof}
For $t$ even, $\sststile{t}{X} X^t \geq 0$ is trivial. For $t$ odd, we have that
$\sststile{t}{X} X^t = X^{t-1} X \geq 0$, where we used that $\sststile{t-1}{X} X^{t-1} \geq 0$ because $t-1$ is even.
\end{proof}

\begin{lemma}
\label{lemma:sos-square-power-upper-bound}
For variable $X \in \mathbb{R}$ and $t \geq 1$ integer,
\[\{0 \leq X \leq 1\} \sststile{t}{X} \{X^t \leq 1\}.\]
\end{lemma}
\begin{proof}
We have $\sststile{X}{t} 1 - X^t = (1-X)(1+X+...+X^{t-1}) \geq 0$, where we used that, by \cref{lemma:sos-square-power-lower-bound}, $\sststile{i}{X} X^i \geq 0$ for $i \in \{0, ..., t-1\}$.
\end{proof}

\begin{lemma}[Restatement of Lemma A.3 in \cite{DBLP:journals/corr/abs-1711-11581}]
\label{lemma:sos-square-root-bound}
For variable $X \in \mathbb{R}$ and $t \geq 2$ even,
\[\{X^t \leq 1\} \sststile{t}{X} \{X \leq 1\}.\]
\end{lemma}

\begin{lemma}
\label{lemma:sos-cs}
For variables $u, v \in \mathbb{R}^d$,
\[\sststile{2}{u,v} \langle u, v\rangle^2 \leq \|u\|^2 \cdot \|v\|^2.\]
\end{lemma}
\begin{proof}
By Lagrange's identity, 
\[\langle u, v\rangle^2 = \|u\|^2 \cdot \|v\|^2 + \sum_{i=1}^{d-1} \sum_{j=i+1}^d (u_i v_j - u_j v_j)^2,\]
so 
\[\sststile{2}{u,v} \langle u, v\rangle^2 \leq \|u\|^2 \cdot \|v\|^2.\]
\end{proof}

\begin{lemma}
\label{lemma:sample-direction-lower-bound-root}
For variable \(X \in \mathbb{R}\) and $\delta \in \mathbb{R}$ and $t \geq 1$ integer,
\[\{0 \leq X \leq 1, X^t \geq \delta\} \sststile{t}{X} \{X \geq \delta\}.\]
\end{lemma}

\begin{proof}
We have $\sststile{t}{X} X = (X^t - \delta) + (1-X) (1 + X + ... + X^{t-1}) + \delta \geq \delta$, where we used that, by \cref{lemma:sos-square-power-lower-bound}, $\sststile{i}{X} X^i \geq 0$ for $i \in \{0, ..., t-1\}$.
\end{proof}

\begin{lemma}
\label{lemma:sample-direction-boost}
For variable \(X \in \mathbb{R}\) and $C \geq 2$ and $t \geq 1$ integer,
\[\left\{0 \leq X \leq \frac{1}{Ct}\right\} \sststile{t}{X} \left\{(1-X)^t \leq 1 - \frac{C-2}{C-1} tX\right\}.\]
\end{lemma}

\begin{proof}
We have that
\begin{align*}
\sststile{t}{X}
(1-X)^t &= 1 - t X + \sum_{i=2}^t \binom{t}{i} (-1)^i X^i
\stackrel{(1)}{\leq} 1 - t X + \sum_{i=2}^t \binom{t}{i} X^i\\
&\stackrel{(2)}{\leq} 1 - t X + \sum_{i=2}^t t^i X^i
= 1 - t X + tX \sum_{i=1}^{t-1} t^i X^i\\
&\stackrel{(3)}{\leq} 1 - tX + tX \sum_{i=1}^{t-1} \frac{1}{C^i}
\stackrel{(4)}{\leq} 1 - tX + \frac{1}{C-1} tX\\
&= 1 - \frac{C-2}{C-1} tX.
\end{align*}

We use throughout that, by \cref{lemma:sos-square-power-lower-bound}, $\sststile{i}{X} X^i \geq 0$ for $i \in \{0, ..., t\}$. In (1) we used that \(\sststile{i}{X} -X^i \leq X^i\). In (2) we used that \(\binom{t}{i} \leq t^i\). In (3) we used that \(\sststile{1}{X} 0 \leq X \leq \frac{1}{Ct}\) implies that \(\sststile{i+1}{X} 0 \leq X^{i+1} \leq \frac{1}{(Ct)^i} X\). The upper bound is true because
\[\sststile{i+1}{X} \frac{1}{(Ct)^i} X - X^{i+1} = \left(\frac{1}{Ct} - X\right) \left(\sum_{j=0}^{i-1} \frac{1}{(Ct)^j} X^{i-1-j} \right) X \geq 0.\]
In (4) we used that $\sum_{i=1}^{t-1} \frac{1}{C^i} \leq \sum_{i=1}^\infty \frac{1}{C^i} = \frac{1}{1-\frac{1}{C}}-1=\frac{1}{C-1}$.
\end{proof}

\begin{lemma}[Restatement of Claim 1.5 in \cite{MR3966537-RaghavendraSchrammSteurerICM18}]
\label{lemma:sos-pe-cs}
If $\tilde{\mathbb{E}}$ is a degree-$d$ pseudo-expectation and if $p, q$ are polynomials of degree at most $\frac{d}{2}$, then $\tilde{\mathbb{E}}[q(x) \cdot p(x)] \leq \frac{1}{2} \tilde{\mathbb{E}}[q(x)^2] + \frac{1}{2} \tilde{\mathbb{E}}[p(x)^2]$.
\end{lemma}

\begin{lemma}
\label{lemma:sos-pe-jensen}
If $\tilde{\mathbb{E}}$ is a degree-$d$ pseudo-expectation and if $p$ is a polynomial of degree at most $\frac{d}{2}$, then $(\tilde{\mathbb{E}}[p(x)])^2 \leq \tilde{\mathbb{E}}[p(x)^2]$.
\end{lemma}
\begin{proof}
Let $\tilde{\mathbb{E}}_x$ be the given pseudo-expectation over $x$, and let $\tilde{\mathbb{E}}_{x'}$ be a copy of the given pseudo-expectaiton but over $x'$ instead of $x$. Then we have
\[(\tilde{\mathbb{E}}_x[p(x)])^2 = (\tilde{\mathbb{E}}_x[p(x)])(\tilde{\mathbb{E}}_{x'}[p(x')]) = \tilde{\mathbb{E}}_{x,x'} [p(x) p(x')].\]
Then, by \cref{lemma:sos-pe-cs},
\[(\tilde{\mathbb{E}}_x[p(x)])^2 \leq \frac{1}{2} \tilde{\mathbb{E}}_{x,x'}[p(x)^2] + \frac{1}{2} \tilde{\mathbb{E}}_{x,x'}[p(x')^2] = \tilde{\mathbb{E}}_{x}[p(x)^2].\]
\end{proof}

\begin{lemma}[Restatement of Lemma 4.5 in \cite{MR3727623-BarakSteurerICM14}]
\label{lemma:sos-pe-triangle}
If $\tilde{\mathbb{E}}$ is a degree-$d$ pseudo-expectation over vectors $u$, $v$, then 
\[\left(\tilde{\mathbb{E}}\Norm{u+v}_d^d\right)^{1/d} \leq \left(\tilde{\mathbb{E}}\Norm{u}_d^d\right)^{1/d} + \left(\tilde{\mathbb{E}}\Norm{v}_d^d\right)^{1/d}.\]
\end{lemma}

We give now some sum-of-squares proofs that are more specific to our setting.
The purpose of \cref{lemma:find-direction-sos-power-aid-lower-bound} and \cref{lemma:find-direction-sos-power-aid-upper-bound} is to aid in transforming some sum-of-squares proofs about polynomials \(p(x)\) and \(q(x)\) into sum-of-squares proofs about polynomials \(p(x)^t\) and \(q(x)^t\).
\cref{lemma:find-direction-sos-power-aid-lower-bound} shows that, under some conditions, if \(\{p(x) \geq 1\} \sststile{}{x} \{q(x) \geq 1\}\), then also \(\{p(x)^t \geq 1\} \sststile{}{x} \{q(x)^t \geq 1\}\),
while \cref{lemma:find-direction-sos-power-aid-upper-bound} shows that, again under some conditions, if \(\{p(x) \leq 1\} \sststile{}{x} \{q(x) \geq 1\}\), then also \(\{p(x)^t \leq 1\} \sststile{}{x} \{q(x)^t \geq 1\}\).
These are used in \cref{lemma:find-direction-sos-power-main-lower-bound} and \cref{lemma:find-direction-sos-power-main-upper-bound}, which implement sum-of-squares proofs with some polynomials raised to the \(t\)-th power.

\begin{lemma}
\label{lemma:find-direction-sos-power-aid-lower-bound}
Let \(p, q: \mathbb{R} \to \mathbb{R}\) with \(p(x) \geq 0\) for all \(x \in \mathbb{R}\).
Let \(\gamma > 1\) be a real number and \(t \geq 2\) be an even integer.
Suppose that, for all \(x \in \mathbb{R}\), $q(x) - 1 - \gamma (p(x) - 1) \geq 0$.
Then, for all \(x \in \mathbb{R}\),
\[q(x)^t - 1 - \gamma (p(x)^t - 1) \geq 0.\]
\end{lemma}

\begin{proof}
We consider two cases.
First, suppose that \(1 + \gamma (p(x) - 1) < 0\).
This implies that \(p(x) < 1 - \frac{1}{\gamma}\), which implies that \(1+\gamma(p(x)^t - 1) < 1 + \gamma (p(x) - 1) < 0\).
Therefore $q(x)^t \geq 1 + \gamma(p(x)^t - 1)$ is satisfied trivially for \(t\) even.

Second, suppose that \(1 + \gamma(p(x) - 1) \geq 0\).
Then the given assumption implies that ${q(x)^t \geq \left(1 + \gamma (p(x) - 1)\right)^t}$.
Then
\[q(x)^t - 1 - \gamma (p(x)^t - 1) \geq \left(1 + \gamma (p(x) - 1)\right)^t - 1 - \gamma(p(x)^t - 1).\]

To show that the expression on the right-hand side is non-negative, it suffices to show that
\[f(x) = \left(1 + \gamma (x - 1)\right)^t - 1 - \gamma(x^t - 1)\]
is non-negative everywhere.
For \(\gamma > 1\), we have that \(\lim_{x \to -\infty} f(x) = \infty\) and \(\lim_{x\to\infty} f(x) = \infty\).
Then, it suffices to show that \(f(x)\) is non-negative at all its critical points. 
We have
\[\frac{d}{dx} f(x) = \gamma t (\gamma(x-1)+1)^{t-1} - \gamma t x^{t-1},\]
so
\[\frac{d}{dx} f(x) = 0 \Longleftrightarrow \gamma(x-1)+1 = x \Longleftrightarrow x=1.\]
We have \(f(1) = 0 \geq 0\).
Therefore, \(f(x) \geq 0\) for all \(x \in \mathbb{R}\).

\end{proof}

\begin{lemma}
\label{lemma:find-direction-sos-power-aid-upper-bound}
Let \(p, q: \mathbb{R} \to \mathbb{R}\) for all \(x \in \mathbb{R}\).
Let \(\gamma > 0\) be a real number and \(t \geq 2\) be an even integer.
Suppose that, for all \(x \in \mathbb{R}\), $q(x) - 1 - \gamma (1- p(x)) \geq 0$.
Then, for all \(x \in \mathbb{R}\),
\[q(x)^t - 1 - \gamma (1 - p(x)^t) \geq 0.\]
\end{lemma}

\begin{proof}
We consider two cases.
First, suppose that \(1 + \gamma (1 - p(x)) < 0\).
This implies that \(p(x) > 1 + \frac{1}{\gamma}\), which implies that \(1+\gamma(1-p(x)^t) < 1 + \gamma (1-p(x)) < 0\).
Therefore $q(x)^t \geq 1 + \gamma(1 - p(x)^t)$ is satisfied trivially for \(t\) even.

Second, suppose that \(1 + \gamma(1 - p(x)) \geq 0\).
Then the given assumption implies that ${q(x)^t \geq \left(1 + \gamma (1 - p(x))\right)^t}$.
Then
\[q(x)^t - 1 - \gamma (1 - p(x)^t) \geq \left(1 + \gamma (1 - p(x))\right)^t - 1 - \gamma(1 - p(x)^t).\]

To show that the expression on the right-hand side, it suffices to show that
\[f(x) = \left(1 + \gamma (1 - x)\right)^t - 1 - \gamma(1 - x^t)\]
is non-negative everywhere.
For \(\gamma > 0\), we have that \(\lim_{x \to -\infty} f(x) = \infty\) and \(\lim_{x\to\infty} f(x) = \infty\).
Then, it suffices to show that \(f(x)\) is non-negative at all its critical points. 
We have
\[\frac{d}{dx} f(x) = \gamma t x^{t-1} - \gamma t (\gamma(1-x)+1)^{t-1},\]
so
\[\frac{d}{dx} f(x) = 0 \Longleftrightarrow x = \gamma(1-x)+1 \Longleftrightarrow x=1.\]
We have \(f(1) = 0 \geq 0\).
Therefore, \(f(x) \geq 0\) for all \(x \in \mathbb{R}\).

\end{proof}

\cref{lemma:find-direction-sos-power-main-lower-bound}, which is used in \cref{lemma:find-direction-moment-maximization}, provides a sum-of-squares proof that if \(\left(x^2+\frac{1}{M}(1-(1-\sigma^2)x^2)\right)^t \geq \frac{1}{\gamma^t}\), then \(x^{2t} \geq \left(\frac{M-\gamma}{\gamma}\frac{1}{M-1+\sigma^2}\right)^t\).

\begin{lemma}
\label{lemma:find-direction-sos-power-main-lower-bound}
For a variable \(x \in \mathbb{R}\) and for \(0 \leq \sigma^2 < 1\) and \(0 < \gamma < M\) and \(M \geq 2\), we have that
\[\left\{\left(\gamma\left(x^2+\frac{1}{M}(1-(1-\sigma^2)x^2)\right)\right)^t \geq 1\right\} \sststile{2t}{x} \left\{\left(\frac{\gamma}{M-\gamma}(M-1+\sigma^2)x^2\right)^t \geq 1\right\}.\]
\end{lemma}

\begin{proof}
Let
\[p(x) = \gamma\left(x^2+\frac{1}{M}(1-(1-\sigma^2)x^2)\right) = \gamma\left(\frac{M-1+\sigma^2}{M} x^2 + \frac{1}{M}\right)\]
and
\[q(x) = \delta (M-1+\sigma^2) x^2,\]
for some \(\delta > 0\) to be determined later.
Note that \(p(x) \geq 0\) for all \(x \in \mathbb{R}\).

We check now that, for all \(x \in \mathbb{R}\),
\[q(x) - 1 - \frac{M\delta}{\gamma}(p(x)-1)\geq 0,\]
which corresponds to a sum-of-squares proof that \(\{p(x) \geq 1\} \sststile{}{x} \{q(x) \geq 1\}\).
We note that the coefficient \(\frac{M\delta}{\gamma}\) was chosen such that \(x^2\) cancels.
We have then
\begin{align*}
&q(x) - 1 - \frac{M\delta}{\gamma}(p(x)-1) = - 1 - \frac{M\delta}{\gamma} \left(\frac{\gamma}{M} - 1\right) = \frac{M\delta}{\gamma} - 1 - \delta.
\end{align*}
Set \(\delta = \frac{\gamma}{M-\gamma}\), which makes the term equal to $0$. Therefore, for all \(x \in \mathbb{R}\),
\[q(x) - 1 - \frac{2\frac{\gamma}{M-\gamma}}{\gamma}(p(x)-1) \geq 0.\]
Therefore, by \cref{lemma:find-direction-sos-power-aid-lower-bound}, for all \(x \in \mathbb{R}\),
\[f(x) = q(x)^t - 1 - \frac{2\frac{\gamma}{M-\gamma}}{\gamma}(p(x)^t-1) \geq 0.\]
Because \(f(x)\) is a univariate polynomial of degree \(2t\), there also exists a sum-of-squares proof of degree at most \(2t\) that \(f(x) \geq 0\).
Note that this constitutes a degree-\(2t\) sum-of-squares proof that \(\{p(x)^t\geq 1\} \sststile{}{x} \{q(x)^t\geq 1\}\).
This concludes the proof.

\end{proof}

\cref{lemma:find-direction-sos-power-main-upper-bound}, which is used in \cref{lemma:find-direction-moment-minimization}, provides a sum-of-squares proof that if \(\left(\frac{x^2+\Delta(1-(1-\sigma^2)x^2)}{1+8\Delta \sigma^2}\right)^t \leq \frac{1}{\gamma^t}\), then \(x^{2t} \geq \left(\frac{\gamma\Delta - 1}{\gamma(\Delta - 1)}(1-10\sigma^2)\right)^t\).

\begin{lemma}
\label{lemma:find-direction-sos-power-main-upper-bound}
For a variable \(x \in \mathbb{R}\) and for \(0 \leq \sigma^2 < 0.1\) and \(\Delta \geq 10\) and \(t\) even and \(\gamma \geq 0.9\), we have that
\[\left\{\left(\gamma \frac{x^2+\Delta(1-(1-\sigma^2)x^2)}{1+8\Delta \sigma^2}\right)^t \leq 1\right\} \sststile{2t}{x} \left\{\left(\frac{\gamma(\Delta-1)}{\gamma\Delta - 1} \frac{x^2}{1 - 10 \sigma^2}\right)^t \geq 1\right\}.\]
\end{lemma}

\begin{proof}
Note that we need \(\sigma^2 < 0.1\) in order to have \(1-10\sigma^2 > 0\).

Let
\[p(x) = \gamma \frac{x^2+\Delta(1-(1-\sigma^2)x^2)}{1+8\Delta \sigma^2} = \gamma \frac{\left(1-\Delta(1-\sigma^2)\right)x^2 + \Delta}{1+8\Delta \sigma^2}\]
and
\[q(x) = \delta \frac{x^2}{1 - 10 \sigma^2},\]
for some \(\delta > 0\) to be determined later.

We check now that, for all \(x \in \mathbb{R}\),
\[q(x) - 1 - \frac{\delta(1+8\Delta \sigma^2)}{\gamma(\Delta(1-\sigma^2)-1)(1 - 10 \sigma^2)} (1-p(x))\geq 0,\]
which corresponds to a sum-of-squares proof that \(\{p(x) \leq 1\} \sststile{}{x} \{q(x) \geq 1\}\).
We note that the coefficient \(\frac{\delta(1+8\Delta \sigma^2)}{\gamma(\Delta(1-\sigma^2)-1)(1 - 10 \sigma^2)}\) was chosen such that \(x^2\) cancels.
We have then
\begin{align*}
&q(x) - 1 - \frac{\delta(1+8\Delta \sigma^2)}{\gamma(\Delta(1-\sigma^2)-1)(1 - 10 \sigma^2)} (1-p(x))\\
&\quad = -1 - \frac{\delta(1+8\Delta \sigma^2)}{\gamma(\Delta(1-\sigma^2)-1)(1 - 10 \sigma^2)} \left(1- \frac{\gamma\Delta}{1+8\Delta \sigma^2}\right)\\
&\quad = -1 - \frac{\delta(1+8\Delta \sigma^2)}{\gamma(\Delta(1-\sigma^2)-1)(1 - 10 \sigma^2)} \frac{1+8\Delta \sigma^2-\gamma \Delta}{1+8\Delta \sigma^2}\\
&\quad = -1 - \frac{\delta(1+8\Delta \sigma^2-\gamma\Delta)}{\gamma(\Delta(1-\sigma^2)-1)(1 - 10 \sigma^2)}\\
&\quad = \frac{-\gamma(\Delta(1-\sigma^2)-1)(1 - 10 \sigma^2)-\delta(1+8\Delta \sigma^2-\gamma\Delta)}{\gamma(\Delta(1-\sigma^2)-1)(1 - 10 \sigma^2)}\\
&\quad = \frac{(-10\gamma\Delta) \sigma^4 + (11\gamma\Delta-10\gamma-8\delta\Delta)\sigma^2 + (\gamma\delta\Delta - \gamma\Delta +\gamma -\delta)}{\gamma(\Delta(1-\sigma^2)-1)(1 - 10 \sigma^2)}.
\end{align*}
Note that the denominator is positive.
Set \(\delta = \frac{\gamma(\Delta-1)}{\gamma\Delta - 1}\).
Then the numerator, viewed as a quadratic in \(\sigma^2\), has roots at \(0\) and at \(\frac{11\gamma\Delta^2 -10\gamma\Delta-8\Delta^2-3\Delta+10}{10\Delta(\gamma\Delta-1)}\).
Furthermore, when the second root is positive, the quadratic is also positive for all \(\sigma^2\) between the two roots.
Hence, in order to prove that the expression is positive for all \(0 \leq \sigma^2 < 0.1\), it suffices to show that the second root is at least \(0.1\) in our setting.
Indeed, for all \(\gamma \geq 0.9\) and all \(\Delta \geq 10\), we have that \(\frac{11\gamma\Delta^2 -10\gamma\Delta-8\Delta^2-3\Delta+10}{10\Delta(\gamma\Delta-1)} \geq 0.1\).

Therefore, for all \(x \in \mathbb{R}\),
\[f(x) = q(x) - 1 - \frac{\frac{\gamma(\Delta-1)}{\gamma\Delta - 1}(1+8\Delta \sigma^2)}{\gamma(\Delta(1-\sigma^2)-1)(1 - 10 \sigma^2)}(1-p(x)) \geq 0.\]
Therefore, by \cref{lemma:find-direction-sos-power-aid-upper-bound}, for all \(x \in \mathbb{R}\),
\[f(x) = q(x)^t - 1 - \frac{\frac{\gamma(\Delta-1)}{\gamma\Delta - 1}(1+8\Delta \sigma^2)}{\gamma(\Delta(1-\sigma^2)-1)(1 - 10 \sigma^2)}(1-p(x)^t) \geq 0.\]
Because \(f(x)\) is a univariate polynomial of degree \(2t\), there also exists a sum-of-squares proof of degree at most \(2t\) that \(f(x) \geq 0\).
Note that this constitutes a degree-\(2t\) sum-of-squares proof that \(\{p(x)^t\leq 1\} \sststile{}{x} \{q(x)^t\geq 1\}\).
This concludes the proof.

\end{proof}

\subsection{Finite sample lemmas}

\begin{lemma}[Restatement of Theorem 4 in \cite{DBLP:conf/focs/BrubakerV08}]
\label{lemma:finite-sample-isotopic-samples}
For \(n \geq C\frac{kd \log^2(d/\delta)}{\epsilon^2}\), with probability \(1-\delta\),
\[\|\operatorname{cov}(\bm y^0)^{-1/2} (\hat{\mathbb{E}} \bm y^0 - \mathbb{E} \bm y^0)\| \leq \epsilon\]
and
\[\|I_d - \operatorname{cov}(\bm y^0)^{-1/2} \widehat{\operatorname{cov}}(\bm y^0) \operatorname{cov}(\bm y^0)^{-1/2}\| \leq \epsilon.\]
\end{lemma}

\begin{lemma}
\label{lemma:finite-sample-isotopic-samples-inverted}
For \(n \geq C\frac{kd \log^2(d/\delta)}{\epsilon^2}\), with probability \(1-\delta\),
\[\|I_d - \widehat{\operatorname{cov}}(\bm y^0)^{-1/2} \operatorname{cov}(\bm y^0) \widehat{\operatorname{cov}}(\bm y^0)^{-1/2}\| \leq \epsilon.\]
\end{lemma}
\begin{proof}
By \cref{lemma:finite-sample-isotopic-samples},
\[\|I_d - \operatorname{cov}(\bm y^0)^{-1/2} \widehat{\operatorname{cov}}(\bm y^0) \operatorname{cov}(\bm y^0)^{-1/2}\| \leq 2\epsilon.\]
Then 
\[(1-\epsilon)I_d \preceq \operatorname{cov}(\bm y^0)^{-1/2} \widehat{\operatorname{cov}}(\bm y^0) \operatorname{cov}(\bm y^0)^{-1/2} \preceq (1+\epsilon)I_d,\]
\[(1-\epsilon)\operatorname{cov}(\bm y^0) \preceq \widehat{\operatorname{cov}}(\bm y^0) \preceq (1+\epsilon)\operatorname{cov}(\bm y^0),\]
\[\frac{1}{1+\epsilon} \widehat{\operatorname{cov}}(\bm y^0) \preceq \operatorname{cov}(\bm y^0) \preceq \frac{1}{1-\epsilon} \widehat{\operatorname{cov}}(\bm y^0).\]
Using that $\frac{1}{1+\epsilon} \geq 1-2\epsilon$ and $\frac{1}{1-\epsilon} \leq 1+2\epsilon$ for $\epsilon \leq 1/2$,
\[(1-2\epsilon) \widehat{\operatorname{cov}}(\bm y^0) \preceq \operatorname{cov}(\bm y^0) \preceq (1+2\epsilon) \widehat{\operatorname{cov}}(\bm y^0),\]
\[(1-2\epsilon) I_d \preceq \widehat{\operatorname{cov}}(\bm y^0)^{-1/2} \operatorname{cov}(\bm y^0) \widehat{\operatorname{cov}}(\bm y^0)^{-1/2} \preceq (1+2\epsilon) I_d,\]
\[\|I_d - \widehat{\operatorname{cov}}(\bm y^0)^{-1/2} \operatorname{cov}(\bm y^0) \widehat{\operatorname{cov}}(\bm y^0)^{-1/2}\| \leq 2\epsilon.\]
\end{proof}

\begin{lemma}[Restatement of Lemma 22 in \cite{DBLP:conf/focs/MoitraV10}]
\label{lemma:finite-sample-isotropic-onedimensional}
Let the random variable \(\overline{\bm{y}} \in \mathbb{R}\) be distributed according to an istotropic mixture of \(k\) one-dimensional Gaussian distributions with minimum mixing weight \(\pmin\).
Let \(\overline{y}_1, ..., \overline{y}_n \in \mathbb{R}\) be generated i.i.d. according to the distribution of \(\overline{\bm{y}}\).
Then, with probability \(1-\delta\),
\[\left( \frac{1}{n} \sum_{i=1}^n \overline{y}_i^t - \mathbb{E} \overline{\bm{y}}^t \right)^2 \leq \frac{1}{n \delta} \pmin^{-O(t)}.\]
\end{lemma}

\begin{lemma}
\label{lemma:finite-sample-isotropic-multidimensional}
Let the random variable \(\bm{y} \in \mathbb{R}^d\) be distributed according to an istotropic mixture of \(k\) \(d\)-dimensional Gaussian distributions with minimum mixing weight \(\pmin\). Let \(y_1, ..., y_n \in \mathbb{R}^d\) be generated i.i.d. according to the distribution of \(\bm{y}\). Then, with probability \(1-d^{t}\delta\),
\[\left\| \frac{1}{n} \sum_{i=1}^n y_i^{\otimes t} - \mathbb{E} \bm{y}^{\otimes t} \right\|^2 \leq \frac{1}{n\delta} (\pmin^{-1} d)^{O(t)}.\]
\end{lemma}

\begin{proof} The proof is similar to the proof of Lemma 22 in \cite{DBLP:conf/focs/MoitraV10}.

We denote by \(\bm{y}^{(j)}\) the \(j\)-th coordinate of \(\bm{y}\). Let \(\alpha \in \mathbb{N}^{d}\) satisfy \(\sum_{j=1}^d \alpha_j = t\). Let \(\bm{z}^{\alpha} = \prod_{j=1}^d (\bm{y}^{(j)})^{\alpha_j}\). By Chebyshev's inequality, with probability at least \(1-\delta\),
\[\left(\frac{1}{n} \sum_{i=1}^n z_i^{\alpha} - \mathbb{E} \bm{z}^\alpha \right)^2 \leq \frac{1}{\delta} \mathbb{E}\left[\left(\frac{1}{n} \sum_{i=1}^n z_i^{\alpha} - \mathbb{E} \bm{z}^\alpha\right)^2\right].\]
We now bound the right-hand side. Note that \(\mathbb{E}[\frac{1}{n} \sum_{i=1}^n z_i^{\alpha} - \mathbb{E} \bm{z}^\alpha] = 0\). Using that for independent random variables the variance of the sum is equal to the sum of the variances,
\begin{align*}
\mathbb{E}\left[\left(\frac{1}{n} \sum_{i=1}^n z_i^{\alpha} - \mathbb{E} \bm{z}^\alpha\right)^2\right]
= \frac{1}{n} \mathbb{E}\left[\left(\bm{z}^{\alpha} - \mathbb{E} \bm{z}^\alpha\right)^2\right]
\leq \frac{1}{n} \mathbb{E}\left[\left(\bm{z}^{\alpha}\right)^2\right]
\leq \frac{1}{n} \pmin^{-O(t)}.
\end{align*}
The last inequality follows by using that \(\mathbb{E} (\bm{y}^{(j)})^t \leq \pmin^{-O(t)}\) for all \(j\) and that, for random varaibles \(\bm{x}_1, ..., \bm{x}_t \in \mathbb{R}\),
$|\mathbb{E}[\bm{x}_1 \cdot ... \cdot \bm{x}_t]| \leq (\mathbb{E} \bm{x}_1^t \cdot ... \cdot \mathbb{E} \bm{x}_t^t)^{1/t}$. Then, by a union bound, with probability at least \(1-d^{t}\delta\),
\[\left\| \frac{1}{n} \sum_{i=1}^n y_i^{\otimes t} - \mathbb{E} \bm{y}^{\otimes t} \right\|^2 \leq \frac{1}{n\delta} d^t \pmin^{-O(t)}.\]

\end{proof}

\begin{lemma}
\label{lemma:finite-sample-isotorpic-moments-noncentered}
Let the random variable \(\bm{y} \in \mathbb{R}^d\) be distributed according to an istotropic mixture of \(k\) \(d\)-dimensional Gaussian distributions with minimum mixing weight \(\pmin\). Let \(y_1, ..., y_n \in \mathbb{R}^d\) be generated i.i.d. according to the distribution of \(\bm{y}\). Then, for $n \geq \frac{1}{\delta}$, with probability \(1-d \delta\),
\[\frac{1}{n} \sum_{i=1}^n \left\| y_i \right\|^{2t} \leq (\pmin^{-1} d)^{O(t)}.\]
\end{lemma}

\begin{proof}
Denote by \(\bm{y}^{(j)}\) the \(j\)-th coordinate of \(\bm{y}\). We have
\begin{align*}
\frac{1}{n} \sum_{i=1}^n \|y_i\|^{2t}
&= \frac{1}{n} \sum_{i=1}^n \left( \sum_{j=1}^d  (y_i^{(j)})^2 \right)^{t}
\leq \frac{1}{n} \sum_{i=1}^n d^{t-1} \sum_{j=1}^d  (y_i^{(j)})^{2t}
= d^{t-1} \sum_{j=1}^d \frac{1}{n} \sum_{i=1}^n  (y_i^{(j)})^{2t}.
\end{align*}
Note that, for each \(j\), \(\bm{y}^{(j)}\) is distributed according to an isotropic mixture of \(k\) one-dimensional Gaussian distributions. By a union bound, with probability \(1-d\delta\) the result in \cref{lemma:finite-sample-isotropic-onedimensional} holds for each coordinate \(\bm{y}^{(j)}\). Then
\begin{align*}
\frac{1}{n} \sum_{i=1}^n \|y_i\|^{2t}
&\leq d^{t-1} \sum_{j=1}^d \left(\mathbb{E} (\bm y^{(j)})^{2t} + \frac{1}{\sqrt{n\delta}} \pmin^{-O(t)}\right).
\end{align*}
We have that \(\mathbb{E} (\bm y^{(j)})^{2t} \leq \pmin^{-O(t)}\). Using that \(n\delta \geq 1\), we get then
\[\frac{1}{n} \sum_{i=1}^n \|y_i\|^{2t} \leq d^t \cdot \pmin^{-O(t)} \leq (\pmin^{-1} d)^{O(t)}.\]

\end{proof}

\subsection{Isotropic position transformation lemmas}
The setting for the following two lemmas is that of \cref{sec:finitesample_col}.

\begin{lemma}[See Lemma 10 in \cite{MR3385380-Hsu13}]
\label{lemma:finite-sample-isotropic-matrix-facts}
We have
\begin{itemize}
\tightlist
\item
    \(\hat{W} \widehat{\operatorname{cov}(\bm y^0)} \hat{W}^\top = I_d\),
\item
    \(\hat{W} \operatorname{cov}(\bm y^0) \hat{W}^\top \succ 0\),
\item
    \(W \operatorname{cov}(\bm y^0) W^\top = I_d\).
\end{itemize}
\end{lemma}

\begin{proof}
The results are immediate by substitution.
\end{proof}

\begin{lemma}[See Lemma 10 in \cite{MR3385380-Hsu13}]
\label{lemma:finite-sample-isotropic-covariance-closeness}
Suppose that
\[\|I_d - \operatorname{cov}(\bm y^0)^{-1/2} \widehat{\operatorname{cov}}(\bm y^0) \operatorname{cov}(\bm y^0)^{-1/2}\| \leq \epsilon.\]
Then
\[\|I_d - (\hat{W} \operatorname{cov}(\bm y^0) \hat{W}^\top)^{1/2}\| \leq O(\epsilon) \cdot \|\operatorname{cov}(\bm y^0)\| \cdot \|\operatorname{cov}(\bm y^0)^{-1}\|.\]
\end{lemma}

\begin{proof}
The given assumption implies that all eigenvalues of \(\operatorname{cov}(\bm y^0)^{-1/2} \widehat{\operatorname{cov}}(\bm y^0) \operatorname{cov}(\bm y^0)^{-1/2}\) lie between \(1-\epsilon\) and \(1+\epsilon\).
Hence all eigenvalues of the inverse of this matrix lie between \(\frac{1}{1+\epsilon}=1+O(\epsilon)\) and \(\frac{1}{1-\epsilon}=1-O(\epsilon)\).
Then
\[\|I_d - \operatorname{cov}(\bm y^0)^{1/2} \widehat{\operatorname{cov}}(\bm y^0)^{-1} \operatorname{cov}(\bm y^0)^{1/2}\| \leq O(\epsilon),\]
\[(1-O(\epsilon)) \cdot \operatorname{cov}(\bm y^0)^{-1} \preceq \widehat{\operatorname{cov}}(\bm y^0)^{-1} \preceq (1+O(\epsilon)) \cdot \operatorname{cov}(\bm y^0)^{-1}.\]
Then
\begin{align*}
\|\hat{W}\|
&= \|(\hat{U}^\top\widehat{\operatorname{cov}}(\bm y^0)\hat{U})^{-1/2} \hat{U}^\top\| \leq \|(\hat{U}^\top\widehat{\operatorname{cov}}(\bm y^0)\hat{U})^{-1/2}\| = \|\widehat{\operatorname{cov}}(\bm y^0)^{-1/2}\|\\
&= \|\widehat{\operatorname{cov}}(\bm y^0)^{-1}\|^{1/2} \leq ((1+O(\epsilon)) \cdot \|\operatorname{cov}(\bm y^0)^{-1}\|)^{1/2}.
\end{align*}
The given assumption also implies that
\[-\epsilon \cdot \operatorname{cov}(\bm y^0) \preceq \widehat{\operatorname{cov}}(\bm y^0) - \operatorname{cov}(\bm y^0) \preceq \epsilon \cdot \operatorname{cov}(\bm y^0).\]
Hence
\begin{align*}
\|\widehat{\operatorname{cov}}(\bm y^0) - \operatorname{cov}(\bm y^0)\| \leq \epsilon \cdot \|\operatorname{cov}(\bm y^0)\|.
\end{align*}
Using these bounds on $\|\hat{W}\|$ and $\|\widehat{\operatorname{cov}}(\bm y^0) - \operatorname{cov}(\bm y^0)\|$, together with the fact that \(\hat{W} \widehat{\operatorname{cov}}(\bm y^0) \hat{W}^\top = I_d\), we get that
\begin{align*}
\|I_d - \hat{W} \operatorname{cov}(\bm y^0) \hat{W}^\top\|
&= \|\hat{W} (\widehat{\operatorname{cov}}(\bm y^0) - \operatorname{cov}(\bm y^0))\hat{W}^\top\|\\
&\leq \|\hat{W}\|^2 \cdot \|\widehat{\operatorname{cov}}(\bm y^0) - \operatorname{cov}(\bm y^0)\|\\
&\leq \epsilon \cdot (1+O(\epsilon)) \cdot \|\operatorname{cov}(\bm y^0)\| \cdot \|\operatorname{cov}(\bm y^0)^{-1}\|\\
&\leq O(\epsilon) \cdot \|\operatorname{cov}(\bm y^0)\| \cdot \|\operatorname{cov}(\bm y^0)^{-1}\|.
\end{align*}
Then all eigenvalues of \(\hat{W} \operatorname{cov}(\bm y^0) \hat{W}^\top\) lie between \(1-\delta\) and \(1+\delta\), for \(\delta = O(\epsilon) \cdot \|\operatorname{cov}(\bm y^0)\| \cdot \|\operatorname{cov}(\bm y^0)^{-1}\|\).
Hence all eigenvalues of the square root of this matrix lie between \(\sqrt{1-\delta} = 1-O(\delta)\) and \(\sqrt{1+\delta } = 1+O(\delta)\).
Then
\[\|I_d - (\hat{W} \operatorname{cov}(\bm y^0) \hat{W}^\top)^{1/2}\| \leq O(\epsilon) \cdot \|\operatorname{cov}(\bm y^0)\| \cdot \|\operatorname{cov}(\bm y^0)^{-1}\|.\]
\end{proof}

\subsection{Miscellaneous lemmas}

\begin{lemma}
\label{lemma:finite-sample-isotropic-matrix-orthogonal}
Let \(W \in \mathbb{R}^{d \times d}\) and \(\Sigma \in \mathbb{R}^{d \times d}\) with \(\Sigma \succ 0\) symmetric. Suppose that \(W \Sigma W^\top = I_d\). Then \(W = Q \Sigma^{-1/2}\) for some orthogonal matrix \(Q \in \mathbb{R}^{d \times d}\).
\end{lemma}

\begin{proof} We have
\begin{align*}
W\Sigma W^\top = I_d \Longleftrightarrow (W\Sigma^{1/2})(W\Sigma^{1/2})^\top = I_d \Longleftrightarrow W\Sigma^{1/2} = Q \Longleftrightarrow W = Q \Sigma^{-1/2}
\end{align*}
for some orthogonal matrix \(Q\).
\end{proof}

\begin{lemma}
\label{lemma:find-direction-binom}
For integers \(0 \leq s \leq t\),
\[\binom{2t}{2s} (2t-2s-1)!! \leq \binom{t}{s} (et)^{t-s},\]
\[\binom{2t}{2s} (2t-2s-1)!! \geq \binom{t}{s} (t/2)^{t-s}.\]
\end{lemma}

\begin{proof}
We use the known fact that \((2t-2s-1)!!=\frac{(2t-2s)!}{2^{t-s} (t-s)!}\). Then
\[\frac{\binom{2t}{2s} (2t-2s-1)!!}{\binom{t}{s} t^{t-s}} = \frac{\frac{(2t)!}{(2s)!(2t-2s)!} \frac{(2t-2s)!}{2^{t-s}(t-s)!}}{\frac{t!}{s!(t-s)!} t^{t-s}} = \frac{(t+1)(t+2)\cdots(2t)}{(s+1)(s+2)\cdots(2s) (2t)^{t-s}}.\]
For the upper bound, we have
\begin{align*}
\frac{(t+1)(t+2)\cdots(2t)}{(s+1)(s+2)\cdots(2s) (2t)^{t-s}}
&= \frac{(t+1)(t+2)\cdots(t+s)}{(s+1)(s+2)\cdots(2s)} \frac{(t+s+1)(t+s+2)\cdots(2t)}{(2t)^{t-s}}\\
&\leq \frac{(t+1)(t+2)\cdots(t+s)}{(s+1)(s+2)\cdots(2s)}\\
&\leq \left(\frac{t}{s}\right)^{s} \leq e^{t-s},
\end{align*}
where in the last inequality we used that \(\left(\frac{t}{s}\right)^{\frac{s}{t-s}} = \left(1 + \frac{t-s}{s}\right)^{\frac{s}{t-s}} \leq e\).

For the lower bound, we have
\begin{align*}
\frac{(t+1)(t+2)\cdots(2t)}{(s+1)(s+2)\cdots(2s) (2t)^{t-s}}
&= \frac{(t+1)(t+2)\cdots(t+s)}{(s+1)(s+2)\cdots(2s)} \frac{(t+s+1)(t+s+2)\cdots(2t)}{(2t)^{t-s}}\\
&\geq \frac{(t+s+1)(t+s+2)\cdots(2t)}{(2t)^{t-s}}\\
&\geq \frac{1}{2^{t-s}}.
\end{align*}

\end{proof}

\subsection{Proofs deferred from Section~\ref{sec:separatingpolynomial}}
\label{proofs-separatingpolynomial}

\begin{proof}[Proof of \cref{lemma:separating-polynomial-finite-sample-norm}.]
For the first proof, with probability $1-\epsilon/100$,
\begin{align*}
\sststile{2}{v} v^\top \operatorname{cov}(\bm z) v
&= v^\top \widehat{\operatorname{cov}}(\bm z)^{1/2} \widehat{\operatorname{cov}}(\bm z)^{-1/2} \operatorname{cov}(\bm z) \widehat{\operatorname{cov}}(\bm z)^{-1/2} \widehat{\operatorname{cov}}(\bm z)^{1/2} v\\
&\stackrel{(*)}{\leq} \|\widehat{\operatorname{cov}}(\bm z)^{-1/2} \operatorname{cov}(\bm z) \widehat{\operatorname{cov}}(\bm z)^{-1/2}\| \cdot \|\widehat{\operatorname{cov}}(\bm z)^{1/2} v\|^{2}\\
&\leq (1+\eta) C,
\end{align*}
where in (*) we used \cref{lemma:finite-sample-isotopic-samples-inverted}.

Similarly, for the second proof, with probability $1-\epsilon/100$
\begin{align*}
\sststile{2}{v} v^\top \widehat{\operatorname{cov}}(\bm z) v
&= v^\top \operatorname{cov}(\bm z)^{1/2} \operatorname{cov}(\bm z)^{-1/2} \widehat{\operatorname{cov}}(\bm z) \operatorname{cov}(\bm z)^{-1/2} \operatorname{cov}(\bm z)^{1/2} v\\
&\stackrel{(*)}{\leq} \|\operatorname{cov}(\bm z)^{-1/2} \widehat{\operatorname{cov}}(\bm z) \operatorname{cov}(\bm z)^{-1/2}\| \cdot \|\operatorname{cov}(\bm z)^{1/2} v\|^{2}\\
&\leq (1+\eta) C.
\end{align*}
where in (*) we used \cref{lemma:finite-sample-isotopic-samples}.
\end{proof}

\begin{proof}[Proof of \cref{lemma:separating-polynomial-finite-sample-closeness}.]
We have
\[\sststile{2t}{v} \hat{\mathbb{E}}\langle \bm z, v\rangle^{2t} = \mathbb{E}\langle \bm z, v\rangle^{2t} + \left(\hat{\mathbb{E}}\langle \bm z, v\rangle^{2t} - \mathbb{E}\langle \bm z, v\rangle^{2t}\right).\]

For the second term we have that
\begin{align*}
&\sststile{4t}{v} \left(\hat{\mathbb{E}}\langle \bm z, v\rangle^{2t} - \mathbb{E}\langle \bm z, v\rangle^{2t}\right)^2\\
&\quad = \left(\hat{\mathbb{E}}\langle \operatorname{cov}(\bm{z})^{-1/2} \bm z, \operatorname{cov}(\bm{z})^{1/2} v\rangle^{2t} - \mathbb{E}\langle \operatorname{cov}(\bm{z})^{-1/2} \bm z, \operatorname{cov}(\bm{z})^{1/2} v\rangle^{2t}\right)^2\\
&\quad= \left( \hat{\mathbb{E}}\langle  (\operatorname{cov}(\bm{z})^{-1/2}\bm z)^{\otimes 2t}, (\operatorname{cov}(\bm{z})^{1/2}v)^{\otimes 2t}\rangle - \mathbb{E}\langle  (\operatorname{cov}(\bm{z})^{-1/2}\bm z)^{\otimes 2t}, (\operatorname{cov}(\bm{z})^{1/2}v)^{\otimes 2t}\rangle \right)^2\\
&\quad= \langle \hat{\mathbb{E}} (\operatorname{cov}(\bm{z})^{-1/2}\bm z)^{\otimes 2t} - \mathbb{E}  (\operatorname{cov}(\bm{z})^{-1/2}\bm z)^{\otimes 2t}, (\operatorname{cov}(\bm{z})^{1/2}v)^{\otimes 2t}\rangle^2\\
&\quad\leq \|\hat{\mathbb{E}}  (\operatorname{cov}(\bm{z})^{-1/2}\bm z)^{\otimes 2t} - \mathbb{E} (\operatorname{cov}(\bm{z})^{-1/2}\bm z)^{\otimes 2t}\|^2 \cdot \|(\operatorname{cov}(\bm{z})^{1/2}v)^{\otimes 2t}\|\\
&\quad \leq \|\hat{\mathbb{E}}  (\operatorname{cov}(\bm{z})^{-1/2}\bm z)^{\otimes 2t} - \mathbb{E} (\operatorname{cov}(\bm{z})^{-1/2}\bm z)^{\otimes 2t}\|^2 \cdot C^t
\end{align*}

Note that $\mathbb{E} \bm{z} = 0$. Then $\operatorname{cov}(\bm{z})^{-1/2} \bm{z}$ is in isotropic position.
By \cref{lemma:finite-sample-isotropic-multidimensional}, with probability \(1-d^{2t}\delta\), we have that
\[\|\hat{\mathbb{E}}  (\operatorname{cov}(\bm{z})^{-1/2}\bm z)^{\otimes 2t} - \mathbb{E} (\operatorname{cov}(\bm{z})^{-1/2}\bm z)^{\otimes 2t}\|^2 \leq \frac{1}{n\delta} (\pmin^{-2} d)^{O(t)}.\]
Select $n = (C \pmin^{-1} d)^{O(t)} \eta^{-1} \epsilon^{-1}$ large enough the right-hand side is upper boundeed by \(\eta^2 C^{-t}\) with probability at least $1-\epsilon$.
Then it follows that
\[\sststile{4t}{v} \left(\hat{\mathbb{E}}\langle \bm z, v\rangle^{2t} - \mathbb{E}\langle \bm z, v\rangle^{2t}\right)^2 \leq \eta^2.\]
Then, by \cref{lemma:sos-square-root-bound},
we get that
\[\sststile{O(t)}{v} \hat{\mathbb{E}}\langle \bm z, v\rangle^{2t} - \mathbb{E}\langle \bm z, v\rangle^{2t} \leq \eta,\]
\[\sststile{O(t)}{v} -\hat{\mathbb{E}}\langle \bm z, v\rangle^{2t} + \mathbb{E}\langle \bm z, v\rangle^{2t} \leq \eta.\]
Rearranging leads to the desired results.

\end{proof}

\subsection{Proofs deferred from Section~\ref{sec:pp}}
\label{proofs-pp}

\begin{proof}[Proof of \cref{lemma:finite-moments-change-empirical}.]
We have
\[\sststile{2t}{v} \hat{\mathbb{E}}\langle \bm y, v\rangle^{2t} = \mathbb{E}\langle \bm y, v\rangle^{2t} + \left(\hat{\mathbb{E}}\langle \bm y, v\rangle^{2t} - \mathbb{E}\langle \bm y, v\rangle^{2t}\right).\]

For the second term we have that
\begin{align*}
\sststile{4t}{v} \left(\hat{\mathbb{E}}\langle \bm y, v\rangle^{2t} - \mathbb{E}\langle \bm y, v\rangle^{2t}\right)^2
&= \left( \hat{\mathbb{E}}\langle \bm y^{\otimes 2t}, v^{\otimes 2t}\rangle - \mathbb{E}\langle \bm y^{\otimes 2t}, v^{\otimes 2t}\rangle \right)^2\\
&= \langle \hat{\mathbb{E}} \bm y^{\otimes 2t} - \mathbb{E} \bm y^{\otimes 2t}, v^{\otimes 2t}\rangle^2\\
&\leq \|\hat{\mathbb{E}} \bm y^{\otimes 2t} - \mathbb{E} \bm y^{\otimes 2t}\|^2.
\end{align*}

By \cref{lemma:finite-sample-isotropic-multidimensional}, with probability \(1-d^{2t}\delta\), we have that
\[\|\hat{\mathbb{E}} \bm y^{\otimes 2t} - \mathbb{E} \bm y^{\otimes 2t}\|^2 \leq \frac{1}{n\delta} (\pmin^{-1} d)^{O(t)}.\]
For \(n \geq (\pmin^{-1} d)^{O(t)} \eta^{-2} \epsilon^{-1}\) the right-hand side is \(\eta^2\) with probability at least $1-\epsilon$.
Then it follows that
\[\sststile{4t}{v} \left(\hat{\mathbb{E}}\langle \bm y, v\rangle^{2t} - \mathbb{E}\langle \bm y, v\rangle^{2t}\right)^2 \leq \eta^2.\]
Then, by \cref{lemma:sos-square-root-bound}, we get that
\[\sststile{O(t)}{v} \hat{\mathbb{E}}\langle \bm y, v\rangle^{2t} - \mathbb{E}\langle \bm y, v\rangle^{2t} \leq \eta,\]
\[\sststile{O(t)}{v} -\hat{\mathbb{E}}\langle \bm y, v\rangle^{2t} + \mathbb{E}\langle \bm y, v\rangle^{2t} \leq \eta.\]
Rearranging leads to the desired results.
\end{proof}

\subsection{Proofs deferred from Section~\ref{sec:colinear}}
\label{proofs-colinear}

\begin{proof}[Proof of \cref{lemma:finite-moments-change-mean}.]
We have
\[\sststile{2t}{v} \hat{\mathbb{E}}\langle \hat{W} (\bm y^0 - \hat{\mathbb{E}} \bm y^0), v\rangle^{2t} = \hat{\mathbb{E}}\left(\langle \hat{W} (\bm y^0 - \mathbb{E} \bm y^0), v\rangle + \langle \hat{W} (\mathbb{E} \bm y^0 - \hat{\mathbb{E}} \bm y^0), v\rangle\right)^{2t}.\]

Let $\delta > 0$ to be specified later. For the upper bound:
\begin{align*}
&\sststile{2t}{v} \hat{\mathbb{E}}\left(\langle \hat{W} (\bm y^0 - \mathbb{E} \bm y^0), v\rangle + \langle \hat{W} (\mathbb{E} \bm y^0 - \hat{\mathbb{E}} \bm y^0), v\rangle\right)^{2t}\\
&\quad \leq \left(1+\delta\right)^{2t-1} \cdot \hat{\mathbb{E}} \langle \hat{W} (\bm y^0 - \mathbb{E} \bm y^0), v\rangle^{2t} + \left(1+\frac{1}{\delta}\right)^{2t-1} \cdot \langle \hat{W} (\mathbb{E} \bm y^0 - \hat{\mathbb{E}} \bm y^0), v\rangle^{2t},
\end{align*}
where in the inequality we used \cref{lemma:sos-triangle}.

For the lower bound:
\begin{align*}
&\sststile{2t}{v} \hat{\mathbb{E}}\left(\langle \hat{W} (\bm y^0 - \mathbb{E} \bm y^0), v\rangle + \langle \hat{W} (\mathbb{E} \bm y^0 - \hat{\mathbb{E}} \bm y^0), v\rangle\right)^{2t}\\
&\quad \stackrel{(1)}{\geq} \left(\frac{1}{1+\delta}\right)^{2t-1} \cdot \hat{\mathbb{E}} \langle \hat{W} (\bm y^0 - \mathbb{E} \bm y^0), v\rangle^{2t} - \left(\frac{1+\frac{1}{\delta}}{1+\delta}\right)^{2t-1} \cdot \langle \hat{W} (\mathbb{E} \bm y^0 - \hat{\mathbb{E}} \bm y^0), v\rangle^{2t}\\
&\quad \stackrel{(2)}{\geq} \left(\frac{1}{1+\delta}\right)^{2t-1} \cdot \hat{\mathbb{E}} \langle \hat{W} (\bm y^0 - \mathbb{E} \bm y^0), v\rangle^{2t} - O\left(1+\frac{1}{\delta}\right)^{2t-1} \cdot \langle \hat{W} (\mathbb{E} \bm y^0 - \hat{\mathbb{E}} \bm y^0), v\rangle^{2t},
\end{align*}
where in (1) we used that, by \cref{lemma:sos-triangle}, \(\sststile{2t}{A, B} A^{2t} \leq (1+\delta)^{2t-1} (A+B)^{2t} + (1+\frac{1}{\delta})^{2t-1} B^{2t}\), so \(\sststile{2t}{A,B} (A+B)^{2t} \geq (\frac{1}{1+\delta})^{2t-1} A^{2t} - (\frac{1+\frac{1}{\delta}}{1+\delta})^{2t-1} B^{2t}\). In (2) we assumed that $\delta = O(1)$, which will be the case for our choice.

Now take \(\delta = \frac{\eta}{100t}\). Then $(1+\delta)^{2t-1} \leq 1+\eta$ and $\left(\frac{1}{1+\delta}\right)^{2t-1} \geq 1-\eta$ for $\eta$ small.

For the second term in both bounds, we use that
\begin{align*}
\sststile{2t}{v} \langle \hat{W} (\mathbb{E} \bm y^0 - \hat{\mathbb{E}} \bm y^0), v\rangle^{2t}
&\leq \|\hat{W} (\mathbb{E} \bm y^0 - \hat{\mathbb{E}} \bm y^0)\|^{2t}\\
&= \|\hat{W} W^{-1} W (\mathbb{E} \bm y^0 - \hat{\mathbb{E}} \bm y^0)\|^{2t}\\
&\leq \|\hat{W} W^{-1}\|^{2t} \cdot \|W (\mathbb{E} \bm y^0 - \hat{\mathbb{E}} \bm y^0)\|^{2t}\\
&= \|(\hat{W} \operatorname{cov}(\bm y^0)^{-1} \hat{W}^\top)^{1/2}\|^{2t} \cdot \|W (\mathbb{E} \bm y^0 - \hat{\mathbb{E}} \bm y^0)\|^{2t}.
\end{align*}
By \cref{lemma:finite-sample-isotopic-samples} and \cref{lemma:finite-sample-isotropic-covariance-closeness}, with probability $1-\epsilon$,
\begin{align*}
\|W (\mathbb{E} \bm y^0 - \hat{\mathbb{E}} \bm y^0)\| &\leq \left(\frac{\eta}{t}\right)^{O(1)}
\end{align*}
and 
\[\|(\hat{W} \operatorname{cov}(\bm y^0)^{-1} \hat{W}^\top)^{1/2}\| \leq 1.\]

Then the second term in both bounds becomes
\begin{align*}
\sststile{2t}{v} O\left(1+\frac{1}{\delta}\right)^{2t-1} \cdot \langle \hat{W} (\mathbb{E} \bm y^0 - \hat{\mathbb{E}} \bm y^0), v\rangle^{2t}
\leq O\left(\frac{100t}{\eta}\right)^{2t-1} \cdot \left( \frac{\eta}{t} \right)^{O(t)}
\leq \eta.
\end{align*}

\end{proof}

\begin{proof}[Proof of \cref{lemma:finite-moments-change-covariance}.]
We have
\[\sststile{2t}{v} \hat{\mathbb{E}}\langle \hat{W} (\bm y^0 - \mathbb{E} \bm y^0), v\rangle^{2t} = \hat{\mathbb{E}}\left(\left\langle W (\bm y^0 - \mathbb{E} \bm y^0), v\right\rangle + \left\langle \left(\hat{W} - W \right) (\bm y^0 - \mathbb{E} \bm y^0), v\right\rangle\right)^{2t}.\]

Let $\delta > 0$ to be specified later. For the upper bound:
\begin{align*}
&\sststile{2t}{v} \hat{\mathbb{E}}\left(\left\langle W (\bm y^0 - \mathbb{E} \bm y^0), v\right\rangle + \left\langle \left(\hat{W} - W \right) (\bm y^0 - \mathbb{E} \bm y^0), v\right\rangle\right)^{2t}\\
&\quad \leq \left(1+\delta\right)^{2t-1} \cdot \hat{\mathbb{E}}\left\langle W(\bm y^0 - \mathbb{E} \bm y^0), v\right\rangle^{2t} + \left(1+\frac{1}{\delta}\right)^{2t-1} \cdot \hat{\mathbb{E}}\left\langle\left(\hat{W} - W \right) (\bm y^0 - \mathbb{E} \bm y^0), v\right\rangle^{2t},
\end{align*}
where in the inequality we used \cref{lemma:sos-triangle}.

For the lower bound:
\begin{align*}
&\sststile{2t}{v} \hat{\mathbb{E}}\left(\left\langle W (\bm y^0 - \mathbb{E} \bm y^0), v\right\rangle + \left\langle \left(\hat{W} - W \right) (\bm y^0 - \mathbb{E} \bm y^0), v\right\rangle\right)^{2t}\\
&\quad \stackrel{(1)}{\geq} \left(\frac{1}{1+\delta}\right)^{2t-1} \cdot \hat{\mathbb{E}}\left\langle W (\bm y^0 - \mathbb{E} \bm y^0), v\right\rangle^{2t} - \left(\frac{1+\frac{1}{\delta}}{1+\delta}\right)^{2t-1} \cdot \hat{\mathbb{E}}\left\langle\left(\hat{W} - W\right) (\bm y^0 - \mathbb{E} \bm y^0), v\right\rangle^{2t}\\
&\quad \stackrel{(2)}{\geq} \left(\frac{1}{1+\delta}\right)^{2t-1} \cdot \hat{\mathbb{E}}\left\langle W (\bm y^0 - \mathbb{E} \bm y^0), v\right\rangle^{2t} - O\left(1+\frac{1}{\delta}\right)^{2t-1} \cdot \hat{\mathbb{E}}\left\langle\left(\hat{W} - W\right) (\bm y^0 - \mathbb{E} \bm y^0), v\right\rangle^{2t},
\end{align*}
where in (1) we used that, by \cref{lemma:sos-triangle}, \(\sststile{2t}{A, B} A^{2t} \leq (1+\delta)^{2t-1} (A+B)^{2t} + (1+\frac{1}{\delta})^{2t-1} B^{2t}\), so \(\sststile{2t}{A,B} (A+B)^{2t} \geq (\frac{1}{1+\delta})^{2t-1} A^{2t} - (\frac{1+\frac{1}{\delta}}{1+\delta})^{2t-1} B^{2t}\). In (2) we assumed that $\delta = O(1)$, which will be the case for our choice.

Now take \(\delta = \frac{\eta}{100t}\). Then $(1+\delta)^{2t-1} \leq 1+\eta$ and $\left(\frac{1}{1+\delta}\right)^{2t-1} \geq 1-\eta$ for $\eta$ small.

For the second term in both bounds, we use that
\begin{align*}
\sststile{2t}{v} \hat{\mathbb{E}}\left\langle\left(\hat{W} - W \right) (\bm y^0 - \mathbb{E} \bm y^0), v\right\rangle^{2t}
&= \hat{\mathbb{E}}\left\langle\left(\hat{W}W^{-1} - I_d \right) W (\bm y^0 - \mathbb{E} \bm y^0), v\right\rangle^{2t}\\
&\leq \left\|I_d - \hat{W}W^{-1}\right\|^{2t} \cdot \hat{\mathbb{E}} \left\|W (\bm y^0 - \mathbb{E} \bm y^0)\right\|^{2t}\\
&= \left\|I_d - (\hat{W} \operatorname{cov}(\bm y^0)\hat{W}^\top)^{1/2}\right\|^{2t} \cdot \hat{\mathbb{E}} \left\|W (\bm y^0 - \mathbb{E} \bm y^0)\right\|^{2t}.
\end{align*}

By \cref{lemma:finite-sample-isotopic-samples} and \cref{lemma:finite-sample-isotropic-covariance-closeness}, with probability $1-\epsilon$,
\[\left\|I_d - (\hat{W} \operatorname{cov}(\bm y^0)\hat{W}^\top)^{1/2}\right\| \leq \left(\frac{\eta}{t \pmin^{-1} d}\right)^{O(1)}.\]
By \cref{lemma:finite-sample-isotorpic-moments-noncentered}, with probability $1-\epsilon$,
\[\hat{\mathbb{E}} \left\|W (\bm y^0 - \mathbb{E} \bm y^0)\right\|^{2t} \leq (\pmin^{-1} d)^{O(t)}.\]

Then the second term in both bounds becomes
\begin{align*}
&\sststile{2t}{v} O\left(1+\frac{1}{\delta}\right)^{2t-1} \cdot \hat{\mathbb{E}}\left\langle\left(\hat{W} - W\right) (\bm y^0 - \mathbb{E} \bm y^0), v\right\rangle^{2t}\\
&\qquad \leq O\left(\frac{100t}{\eta}\right)^{2t-1} \cdot \left(\frac{\eta}{t \pmin^{-1} d}\right)^{O(t)} \cdot (\pmin^{-1} d)^{O(t)}
\leq \eta.
\end{align*}

\end{proof}

\begin{proof}[Proof of \cref{lemma:finite-sample-direction-change}]
We have 
\begin{align*}
|\langle W(\mu_i^0 - \mu_j^0), v\rangle - \langle \hat{W}(\mu_i^0 - \mu_j^0), v\rangle|
&= |\langle (W - \hat{W})(\mu_i^0 - \mu_j^0), v\rangle|\\
&= |\langle (I_d - \hat{W}W^{-1})W(\mu_i^0 - \mu_j^0), v\rangle|\\
&\leq \|I_d - \hat{W}W^{-1}\| \cdot \|W(\mu_i^0 - \mu_j^0)\|\\
&= \|I_d - (\hat{W} \operatorname{cov}(\bm{y}^0) \hat{W}^\top)^{1/2}\| \cdot \|W(\mu_i^0 - \mu_j^0)\|.
\end{align*}

By \cref{lemma:finite-sample-isotopic-samples} and \cref{lemma:finite-sample-isotropic-covariance-closeness}, with probability $1-\epsilon$,
\[\left\|I_d - (\hat{W} \operatorname{cov}(\bm y^0)\hat{W}^\top)^{1/2}\right\| \leq \left(\frac{\eta}{\pmin^{-1}}\right)^{O(1)}.\]

Note that $W(\mu_i^0 - \mu_j^0) = \mu_i - \mu_j = \langle \mu_i - \mu_j, u\rangle u$, so $\|W(\mu_i^0 - \mu_j^0)\| = |\langle \mu_i - \mu_j, u\rangle|$. Using that $\sum_{i=1}^k p_i \langle \mu_i, u\rangle^2 \leq 1$, we have that $\sum_{i=1}^k \langle \mu_i, u\rangle^2 \leq \pmin^{-1}$, so $\langle \mu_i, u\rangle^2 \leq \pmin^{-1}$, so $|\langle \mu_i - \mu_j, u\rangle| \leq 2 \sqrt{\pmin^{-1}}$. Then $\|W(\mu_i^0 - \mu_j^0)\| \leq 2 \sqrt{\pmin^{-1}}$.

Therefore, 
\begin{align*}
|\langle W(\mu_i^0 - \mu_j^0), v\rangle - \langle \hat{W}(\mu_i^0 - \mu_j^0), v\rangle|
&\leq \left(\frac{\eta}{\pmin^{-1}}\right)^{O(1)} \cdot 2 \sqrt{\pmin^{-1}} \leq \eta.
\end{align*}
\end{proof}

\begin{proof}[Proof of \cref{lemma:finite-sample-variance-change}.]
We have 
\begin{align*}
\| v^\top W (\Sigma^0)^{1/2} - v^\top \hat{W} (\Sigma^0)^{1/2} \|
&= \| v^\top (W-\hat{W}) (\Sigma^0)^{1/2}\|\\
&= \|v^\top (I_d-\hat{W}W^{-1}) W (\Sigma^0)^{1/2}\|\\
&\leq \|I_d - \hat{W}W^{-1}\| \cdot \|W (\Sigma^0)^{1/2}\|\\
&= \|I_d - (\hat{W} \operatorname{cov}(\bm{y}^0) \hat{W}^\top)^{1/2}\| \cdot \|W (\Sigma^0)^{1/2}\|.
\end{align*}

By \cref{lemma:finite-sample-isotopic-samples} and \cref{lemma:finite-sample-isotropic-covariance-closeness}, with probability $1-\epsilon$,
\[\left\|I_d - (\hat{W} \operatorname{cov}(\bm y^0)\hat{W}^\top)^{1/2}\right\| \leq \eta.\]

Note that, by \cref{lemma:finite-sample-isotropic-matrix-orthogonal}, $W(\Sigma^0)^{1/2} = Q(\Sigma^0)^{-1/2}(\Sigma^0)^{1/2} = Q$ for an orthogonal matrix $Q$. We have $\|Q\|=1$, so $\|W(\Sigma^0)^{1/2}\| = 1$.

Therefore, 
\begin{align*}
\| v^\top W (\Sigma^0)^{1/2} - v^\top \hat{W} (\Sigma^0)^{1/2} \|
&\leq \eta.
\end{align*}
\end{proof}

\subsection{Proofs deferred from Section~\ref{sec:small-radius}}
\label{proofs-small-radius}

\begin{proof}[Proof of \cref{lemma:small-radius-finite-sample}.]
We have
\begin{align*}
\sststile{2t}{v} \hat{\mathbb{E}}\langle \bm y, v\rangle^{2t} - \mathbb{E}\langle \bm y, v\rangle^{2t}
&= \hat{\mathbb{E}}\langle \bm y^{\otimes 2t}, v^{\otimes 2t}\rangle - \mathbb{E}\langle \bm y^{\otimes 2t}, v^{\otimes 2t}\rangle\\
&= \langle \hat{\mathbb{E}} \bm y^{\otimes 2t} - \mathbb{E} \bm y^{\otimes 2t}, v^{\otimes 2t}\rangle\\
&= \langle \hat{\mathbb{E}} (\bm{y}\bm{y}^\top)^{\otimes t} - \mathbb{E} (\bm{y}\bm{y}^\top)^{\otimes t}, (vv^\top)^{\otimes t}\rangle.
\end{align*}

We now bound $\hat{\mathbb{E}} (\bm{y}\bm{y}^\top)^{\otimes t} - \mathbb{E} (\bm{y}\bm{y}^\top)^{\otimes t}$. Define 
\[E = \hat{\mathbb{E}} (\operatorname{cov}(\bm{y})^{-1/2}\bm{y}\bm{y}^\top\operatorname{cov}(\bm{y})^{-1/2})^{\otimes t} - \mathbb{E} (\operatorname{cov}(\bm{y})^{-1/2}\bm{y}\bm{y}^\top\operatorname{cov}(\bm{y})^{-1/2})^{\otimes t}.\] 
By \cref{lemma:finite-sample-isotropic-multidimensional}, with probability \(1-d^{2t}\delta\), we have that
\[\|E\|_F = \|\hat{\mathbb{E}} (\operatorname{cov}(\bm{y})^{-1/2}\bm{y})^{\otimes 2t} - \mathbb{E} (\operatorname{cov}(\bm{y})^{-1/2}\bm{y})^{\otimes 2t}\| \leq \frac{1}{\sqrt{n\delta}} (\pmin^{-1} d)^{O(t)}.\]
For \(n \geq (\pmin^{-1} d)^{O(t)} \eta^{-2} \epsilon^{-1}\) this term is \(\eta\) with probability at least $1-\epsilon$. In this case $\|E\| \leq \|E\|_F \leq \eta$, so 
\[-\eta \cdot \operatorname{cov}(\bm{y})^{\otimes t}\preceq (\operatorname{cov}(\bm{y})^{1/2})^{\otimes t} E (\operatorname{cov}(\bm{y})^{1/2})^{\otimes t} \preceq \eta \cdot \operatorname{cov}(\bm{y})^{\otimes t}.\]
We observe the connection between $\hat{\mathbb{E}} (\bm{y}\bm{y}^\top)^{\otimes t} - \mathbb{E} (\bm{y}\bm{y}^\top)^{\otimes t}$ and $E$:
\[\hat{\mathbb{E}} (\bm{y}\bm{y}^\top)^{\otimes t} - \mathbb{E} (\bm{y}\bm{y}^\top)^{\otimes t}
= (\operatorname{cov}(\bm{y})^{1/2})^{\otimes t} E (\operatorname{cov}(\bm{y})^{1/2})^{\otimes t}.\]
Using this and using that $\operatorname{cov}(\bm{y}) \preceq \mathbb{E} \bm{y}\bm{y}^\top$, we finally obtain that
\[-\eta \cdot \mathbb{E} (\bm{y}\bm{y}^\top)^{\otimes t} \preceq \hat{\mathbb{E}} (\bm{y}\bm{y}^\top)^{\otimes t} - \mathbb{E} (\bm{y}\bm{y}^\top)^{\otimes t} \preceq \eta \cdot \mathbb{E} (\bm{y}\bm{y}^\top)^{\otimes t}.\]

Then 
\begin{align*}
\sststile{2t}{v} \hat{\mathbb{E}}\langle \bm y, v\rangle^{2t} - \mathbb{E}\langle \bm y, v\rangle^{2t}
&= \langle \hat{\mathbb{E}} (\bm{y}\bm{y}^\top)^{\otimes t} - \mathbb{E} (\bm{y}\bm{y}^\top)^{\otimes t}, (vv^\top)^{\otimes t}\rangle\\
&\leq \eta \cdot \langle \mathbb{E} (\bm{y}\bm{y}^\top)^{\otimes t}, (vv^\top)^{\otimes t}\rangle\\
&= \eta \cdot \mathbb{E} \langle \bm{y}, v\rangle^{2t}
\end{align*}
and 
\begin{align*}
\sststile{2t}{v} \hat{\mathbb{E}}\langle \bm y, v\rangle^{2t} - \mathbb{E}\langle \bm y, v\rangle^{2t}
&= \langle \hat{\mathbb{E}} (\bm{y}\bm{y}^\top)^{\otimes t} - \mathbb{E} (\bm{y}\bm{y}^\top)^{\otimes t}, (vv^\top)^{\otimes t}\rangle\\
&\geq -\eta \cdot \langle \mathbb{E} (\bm{y}\bm{y}^\top)^{\otimes t}, (vv^\top)^{\otimes t}\rangle\\
&= -\eta \cdot \mathbb{E} \langle \bm{y}, v\rangle^{2t}.
\end{align*}
The conclusion follows.
\end{proof}

\end{document}